\documentclass[onecolumn,11pt]{article}
\usepackage[top=1in,bottom=1in,left=1in,right=1in]{geometry}
\usepackage{palatino}
\usepackage{microtype}
\usepackage{amsmath,amsfonts,amscd,amssymb}
\usepackage{graphicx}
\usepackage{import}
\newcommand{\smorebasedir}{}
\IfFileExists{arxiv/SMORE_arxiv.tex}{\renewcommand{\smorebasedir}{arxiv/}}{}
\usepackage{xcolor}
\usepackage{url}
\usepackage{setspace}
\usepackage[numbers,sort&compress]{natbib}
\usepackage[bottom,flushmargin,hang,multiple]{footmisc}
\usepackage{placeins}
\usepackage{multirow}
\usepackage{tabularx}
\usepackage{booktabs}
\usepackage{caption}
\usepackage{subcaption}
\usepackage{rotating}
\usepackage{bm}
\usepackage{makecell}
\usepackage{amsthm}
\usepackage{algorithm}
\usepackage{algpseudocodex}
\usepackage{mathtools}
\usepackage{float}
\usepackage{flafter}
\usepackage{enumitem}
\usepackage[hidelinks]{hyperref}
\usepackage{cleveref}

\makeatletter
\newenvironment{breakablealgorithm}
  {%
   \par\noindent
   \refstepcounter{algorithm}%
   \hrule height .8pt depth 0pt
   \kern 2pt
   \renewcommand{\caption}[2][\relax]{%
     {\raggedright\textbf{\fname@algorithm~\thealgorithm.} ##2\par}%
     \ifx\relax##1\relax
       \addcontentsline{loa}{algorithm}{\protect\numberline{\thealgorithm}##2}%
     \else
       \addcontentsline{loa}{algorithm}{\protect\numberline{\thealgorithm}##1}%
     \fi
     \kern 2pt\hrule\kern 2pt
   }%
  }
  {%
   \kern 2pt\hrule
   \par
  }
\makeatother

\newtheorem{theorem}{Theorem}
\newtheorem{lemma}[theorem]{Lemma}
\newtheorem*{remark}{Remark}
\crefname{theorem}{theorem}{theorems}
\Crefname{theorem}{Theorem}{Theorems}
\crefname{lemma}{lemma}{lemmas}
\Crefname{lemma}{Lemma}{Lemmas}

\newcommand*\iftodonotes{\@ifundefined{@todonotes@disabled}{\expandafter\@secondoftwo}{\if@todonotes@disabled\expandafter\@secondoftwo\else\expandafter\@firstoftwo\fi}}

\DeclareMathOperator*{\argmin}{\arg\!\min}

\algnewcommand\Inputs{\item[\textbf{Inputs:}]}
\algnewcommand\Initialize{\item[\textbf{Initialize:}]}

\crefname{figure}{Fig.}{Figs.}
\Crefname{figure}{Fig.}{Figs.}
\crefname{equation}{Eq.}{Eqs.}
\Crefname{equation}{Eq.}{Eqs.}

\title{\vspace{-.55in}{\fontsize{16}{18}\selectfont\textbf{SMORE: Stability-Promoting Mesh-Agnostic Model Reduction for Time-Dependent PDEs}}\vspace{-.15in}}
\author{\normalsize{Yangyuan Li, Weichao Li, Shaowu Pan\thanks{Corresponding author. Email: \href{mailto:pans2@rpi.edu}{pans2@rpi.edu}.}}\\
\footnotesize{Department of Mechanical, Aerospace and Nuclear Engineering,}\\
\footnotesize{Rensselaer Polytechnic Institute}\\
\footnotesize{110 8th Street, Troy, NY 12180, USA\vspace{-.2in}}}
\date{}

\begin{document}
\maketitle
\vspace{-.2in}
\begin{abstract}
High-fidelity simulations of time-dependent partial differential equations (PDE) are computationally expensive, motivating data-driven reduced-order surrogates for many-query tasks such as uncertainty quantification, design optimization, data assimilation, and optimal control. However, existing surrogate models often exhibit poor temporal stability, which can lead to unstable rollouts and exploding gradients during backpropagation, especially in multistep long-horizon forecasting. To address this, we propose SMORE, a mesh-agnostic framework for model order reduction of time-dependent PDEs. Its latent dynamics are trained with Lyapunov-guided stability regularization, which promotes stable long-horizon rollouts. We provide theoretical guarantees under the stated structural assumptions. Beyond forecasting PDE evolution, the learned latent dynamics, which are interpretable and linear or linear-quadratic, could bring benefits for downstream tasks such as data assimilation and optimal control. Moreover, our framework is capable of predicting continuous PDE solution fields from sparse measurements of the initial condition. We evaluate SMORE on a range of problems, including wave propagation, the Navier–Stokes equations, and the shallow water equations. Our results shows it improves long-horizon rollout generalization and empirical robustness, and achieves competitive accuracy at comparable parameter budgets, relative to competitive baselines including DINo, FNO, CNO, and Transolver.
\end{abstract}


\section{Introduction}
\label{sec:intro}                 
Many practical engineering problems involve high-fidelity numerical simulation of partial differential equations, so-called ``inner loop'' computations. Such simulations typically require memory-intensive spatial discretizations and small time steps to satisfy stability and accuracy requirements. Hence, it is desirable to develop data-driven surrogate models that exploit
existing simulation data to improve the computational efficiency of forward
simulations by orders of magnitude. Such surrogates are particularly valuable
for many-query tasks such as uncertainty quantification
(UQ)~\cite{abdar2021review, beran2017uncertainty}, design
optimization~\cite{du2021rapid}, data assimilation~\cite{zerfas2019continuous,
misaka2020image, casas2020reduced}, and optimal
control~\cite{troltzsch2024optimal}. A typical data-driven surrogate first reduces the dimensionality of the high-fidelity  state~\cite{asadi2020encoder} and then models the evolution of the resulting reduced states~\cite{li2022transformer, berman2024colora, chen2018neural, serrano2024operator, cho2024hypernetwork}. 
The choice of dimensionality reduction therefore determines which spatial discretizations the surrogate can handle. In practice, however, most existing approaches are tied to a specific discretization. Convolutional neural network (CNN)-based methods~\cite{khoo2021solving, qu2022learning}
require uniform Cartesian grids, while proper orthogonal decomposition (POD)~\cite{sidhu2018model}
requires the spatial discretization to remain fixed across snapshots. As a result, such
surrogates cannot generalize to a different spatial discretization of the domain.
To address this limitation, mesh-agnostic methods such as deep operator network (DeepONet) \cite{lu2019deeponet} and its extension~\cite{lu2022comprehensive} have been recently developed. Moreover, neural operators~\cite{li2020fourier, raonic2023convolutional, serrano2024operator} provide an alternative perspective by learning the forcing-to-solution operator via carefully designed neural networks. 
Fourier Neural Operator (FNO) \cite{li2020fourier} utilizes Green's function and convolution theorem to parameterize the integral kernel. For time-dependent PDEs, FNO can learn the mapping between an initial condition and the solution at a fixed later time. The Convolutional Neural Operator (CNO)~\cite{raonic2023convolutional} instead adopts a convolutional architecture that operates in physical space. It is designed to preserve the correspondence between continuous functions and their discrete representations, thereby mitigating aliasing errors. However, CNO and FNO rely on CNNs and/or fast Fourier transforms, which limits their applicability mainly to PDE data on uniformly structured meshes.
In practice, many engineering problems are discretized on non-uniform unstructured meshes. A straightforward way to extend data-driven surrogate models to such meshes is to use graph neural networks. For example, MeshGraphNet~\cite{pfaff2020learning} treats the PDE solution at each cell as node features on a graph and uses message passing within an encoder-decoder framework to learn the evolution of a PDE over time. 
Beyond graph-based frameworks~\cite{belbute2020combining}, several works have extended neural operators to unstructured meshes~\cite{li2023fourier, liu2023nuno}. Li et al.~\cite{li2023fourier} transformed the irregular physical space into uniformly structured computational space in order to perform FFT in FNO. Similarly, Liu et al.~\cite{liu2023nuno} developed a non-uniform neural operator to map nonuniform inputs to a uniform grid. However, these methods could face scalability issues for 3D problems due to the use of convolutional layers and challenges in handling moving meshes~\cite{serrano2024operator}. More recently, attention-based PDE solvers have been developed to overcome the limitations of convolutional or Fourier-based backbones. For
example, Transolver~\cite{wu2024Transolver} introduces a Physics-Attention
mechanism that groups discretized mesh points into learnable physical-state
slices and performs attention over the resulting physics-aware tokens, enabling
efficient modeling on general geometries. Nevertheless, such models still
operate as high-capacity neural PDE solvers in the physical discretization
space, rather than learning a reduced latent dynamical system with explicit stability-oriented structure for
long-horizon prediction.

To address the above challenges, the idea of implicit neural representations~\cite{lu2022comprehensive, serrano2024operator, chen2022crom} can be used to reduce dimensionality~\cite{pan2023neural}. Implicit neural representations were initially developed for image and video processing~\cite{sitzmann2020implicit,fathony2020multiplicative} and geometric shape representation~\cite{chen2019learning,mescheder2019occupancy}, where most of these problems only need spatial representation. 
To adapt them to time-dependent PDEs, one approach is to parameterize the weights of the neural network by another function of time $t$~\cite{pan2023neural} or model the latent state using a nonlinear autonomous system. To our knowledge, the closest related work is DINo, proposed by Yin et
al.~\cite{Yin2022Continuous}. It uses MFN and shift modulation to decode latent
representations at queried spatial coordinates. The latent codes are inferred
through autodecoding and then evolved by a standard NeuralODE~\cite{chen2018neural}. A follow-up work is the coordinate-based model for operator learning (CORAL) \cite{serrano2024operator} using SIREN and shift and scale modulation for autodecoding. Very recently, CoNFiLD \cite{du2024confild} combines SIREN and modulation, but employs the diffusion model \cite{yang2024survey} for the dynamics of latent space.  
Despite these successes, the frameworks above lack explicit stability guarantees for the reduced latent dynamics, which are crucial for long-term autoregressive prediction~\cite{Yin2022Continuous, li2020fourier, serrano2024operator, du2024confild}.  
To address this issue, we propose SMORE, a stability-promoting and mesh-agnostic model reduction framework for time-dependent PDEs. SMORE combines an Implicit neural representation (INR) autodecoder with either a Koopman-inspired linear latent-dynamics model, referred to as \emph{SMORE-Koopman}, or a Low-Rank Linear-Quadratic (LRLQ) latent-dynamics model, referred to as \emph{SMORE-LRLQ}. 
Koopman operator theory~\cite{mezic2005spectral} provides a formalism to search for intrinsic coordinates that can globally linearize the nonlinear dynamics~\cite{brunton2021modern}.  For systems that do not admit a reduced set of linear dynamics, we utilize linear-quadratic dynamics, since such models naturally arise in the reduced-order formulation of discretized fluid mechanical models~\cite{peherstorfer2016data,qian2020lift}, allowing for a more accurate representation. 
In contrast to mesh-agnostic learning frameworks that use
black-box dynamics, such as NeuralODE~\cite{chen2018neural,Yin2022Continuous}
or diffusion models~\cite{du2024confild}, we equip these structured latent dynamics with Lyapunov-guided stability
regularization. The theoretical derivation establishes sufficient
conditions for latent dynamics stability, transfers latent-state bounds
and errors to decoded fields, and characterizes prediction from
sparse initial measurements under additional observability
assumptions. We further introduce pointwise minibatching (PM), which partitions the spatial samples evaluated by the pointwise decoder into smaller chunks during training. Each chunk contributes to the reconstruction loss using a common normalization over the
full batch, and gradients are accumulated across chunks before parameter updates. This reduces peak training memory while preserving the reconstruction loss and its gradients up to floating-point round-off.
We conducted numerical experiments to assess prediction accuracy and training robustness across different dynamics models, mechanism configurations, and initial-observation densities. We further evaluated the proposed pointwise minibatching mechanism by examining the extent to which it increases the maximum feasible spatial resolution under a fixed GPU-memory budget and by quantifying the resulting trade-offs among memory, computation time, and accuracy at a fixed resolution.

The paper is organized as follows. \Cref{sec:problem_description} describes the setup of the problem. \Cref{sec:related_work} first discusses implicit neural representations for mesh-agnostic reduced-order modeling. It then turns to latent dynamics modeling, with particular emphasis on stability-oriented linear and linear-quadratic approaches. \Cref{sec:method} presents the proposed SMORE framework, including the INR-based latent representation and decoder, the SMORE-Koopman and SMORE-LRLQ latent-dynamics models, Lyapunov-guided stability regularization, theoretical boundedness and sparse-inference results, scheduled sampling, a pointwise decoder minibatching mechanism for memory-efficient training, and the overall training and evaluation workflow. \Cref{sec:numerical_experiment_setup} presents the numerical experiment setup, long-horizon rollout comparisons with baseline models, stability ablation study, sparse-observation inference, and spatial-resolution and memory-efficiency studies. Finally, conclusions are given in \Cref{sec:Conclusions}.


\section{Problem Setup}
\label{sec:problem_description}

This section formulates the learning problem for data-driven reduced-order modeling of time-dependent PDEs. We consider a time-dependent partial differential equation (PDE)
with initial and boundary conditions,
\begin{equation}
\label{eq:pde_problem}
\begin{array}{@{}l@{\;}c@{\;}l@{\qquad}l@{}}
    \mathcal{F}(u)(t,\mathbf{x}) & = & 0,
    & (t,\mathbf{x}) \in \mathcal{T}\times\Omega, \\[2pt]
    \mathcal{B}(u)(t,\mathbf{x}) & = & 0,
    & (t,\mathbf{x}) \in \mathcal{T}\times\partial\Omega, \\[2pt]
    u(0,\mathbf{x}) & = & u_0(\mathbf{x}), \mathbf{x} \in \Omega.
    &
\end{array}
\end{equation}
Here, $\Omega\subseteq\mathbb{R}^d$ is the spatial domain,
$\mathcal{T}\coloneqq[0,\infty)$ is the time domain, $u\colon\mathcal{T}\times\Omega\to\mathbb{R}^c$ is the solution field
with $c$ components, and $u_0\in L^2(\Omega;\mathbb{R}^c)$ is the
initial condition. The operators $\mathcal{F}$ and
$\mathcal{B}$ specify the governing PDE and boundary conditions,
respectively, and may be nonlinear.
For the initial conditions and time horizons considered, we assume
that the solution exists, is unique, and can be evaluated at the
sampling locations. We denote the solution map by
\begin{equation}
\label{eq:solution_map}
    S(t)u_0 = u(t,\cdot).
\end{equation}

To formulate the learning problem, let $\mu$ be a Borel probability
measure on $L^2(\Omega;\mathbb{R}^c)$, and draw $M$
initial conditions independently from $\mu$,
\[
    u_{0,i}\overset{\mathrm{i.i.d.}}{\sim}\mu,
    \qquad i=1,\ldots,M.
\]
The prediction benchmarks in this work use trajectories
generated by numerical PDE solvers. All trajectories are sampled on the time grid
\[
    0=t_1<t_2<\cdots<t_{N_{\mathrm{train}}+N_{\mathrm{extra}}},
\]
which we split into the training times and the extrapolation times,
\[
    \mathcal{T}_{\mathrm{train}}
    \coloneqq \{t_j\}_{j=1}^{N_{\mathrm{train}}},
    \qquad
    \mathcal{T}_{\mathrm{extra}}
    \coloneqq \{t_j\}_{j=N_{\mathrm{train}}+1}^{N_{\mathrm{train}}+N_{\mathrm{extra}}},
\]
and on the spatial sample points
\[
    \Omega_{N_x}
    \coloneqq \{\mathbf{x}_k\}_{k=1}^{N_x}
    \subset\Omega.
\]
Here, $N_{\mathrm{train}}$ and $N_{\mathrm{extra}}$ are the numbers of snapshots inside and beyond the training window $[0,t_{N_{\mathrm{train}}}]$, and $N_x$ is the number of spatial
sample points. For each training initial condition $u_{0,i}$, the solution is observed only at the training times, giving the training dataset
\begin{equation}
\label{eq:training_dataset}
    \mathcal{D}_{\mathrm{tr}}
    =
    \left\{
        \left\{
            \left(
                t_j,
                \mathbf{x}_k,
                [S(t_j)u_{0,i}](\mathbf{x}_k)
            \right)
        \right\}_{\begin{subarray}{l}
            1\leq j\leq N_{\mathrm{train}}\\
            1\leq k\leq N_x
        \end{subarray}}
    \right\}_{i=1}^{M}.
\end{equation}
For a fixed initial condition, these spatiotemporal samples
constitute a trajectory~\cite{yu2024learning}.

Our objective is to learn a reduced-order surrogate for the solution map from these trajectories. At test time, the surrogate receives the initial state only on an initial observation set $\Omega^{\mathrm{init}}_{N_x} \subseteq \Omega_{N_x}$ and predicts $u(t,\mathbf{x})$ at queried spatial coordinates over a prescribed time horizon. We refer to the case $\Omega^{\mathrm{init}}_{N_x} = \Omega_{N_x}$ as full initialization and to the case $\Omega^{\mathrm{init}}_{N_x} \subsetneq \Omega_{N_x}$ as sparse initialization. Since the initial latent state is inferred by autodecoding with the trained networks frozen, $\Omega^{\mathrm{init}}_{N_x}$ can be chosen at test time without retraining; in particular, a model trained on the fully observed data in \Cref{eq:training_dataset} can be initialized from sparse measurements. The surrogate is required to generalize to unseen initial conditions drawn from $\mu$, both within the training window (\emph{in-horizon prediction}) and at the extrapolation times $\mathcal{T}_{\mathrm{extra}}$ (\emph{extrapolation}).

\section{Reduced-Order Modeling with Implicit Neural Representations}\label{sec:related_work}

\subsection{Implicit Neural Representation}
INR~\cite{genova2019learning,sitzmann2020implicit,genova2020local} is a class of feedforward neural networks with the spatial coordinate $\mathbf{x} \in \Omega$ as input to describe an arbitrary function $D_\phi(\mathbf{x})$ implicitly via parameter $\phi$. 
Unlike traditional pixelized approaches that describe a function directly as a
vector at fixed sampling points over $\Omega$~\cite{murata2020nonlinear}, INR
allows a mesh-agnostic approximation that can be queried at any point in
$\Omega$. This in turn permits spatial queries to be partitioned into smaller
batches to reduce decoder memory use~\cite{pan2023neural}. This is sometimes described as learning a \emph{continuous} vector field, as opposed to \emph{already-discretized} vector fields~\cite{chen2022crom}, on which traditional methods such as POD are built on, e.g., POD.
Since the solution of a PDE is most often a continuous function in space, INR has been recently introduced to the community of machine learning for PDE~\cite{pan2023neural,duvall2025discretization,chen2023model,chen2023implicit}. 
Pan et al.~\cite{pan2023neural} demonstrated the effectiveness of a hypernetwork~\cite{ha2016hypernetworks} built on top of SIREN~\cite{sitzmann2020implicit} for reduced-order models of complex multi-scale PDE data in a mesh-agnostic fashion. 
Duvall and Duraisamy~\cite{duvall2025discretization} employed coordinate-based networks with hypernetwork modulation to build discretization-independent surrogates of physical fields around variable geometries and non-parametric meshes, albeit for steady-state fields without temporal evolution.
Chen et al.~\cite{chen2023implicit,chen2023model} solved the \emph{known} PDE by
integrating the projected governing equation of the \emph{reduced} latent state
from INR, using either a Lagrangian or an Eulerian method. In contrast, several
other concurrent works~\cite{du2021evolutional,bruna2024neural,anderson2022evolution}
took the \emph{whole} neural network parameters as the latent space. 
Yin et al.~\cite{Yin2022Continuous} utilized an INR autodecoder of multiplicative filter network (MFN) together with NeuralODE~\cite{chen2018neural} to learn the latent representation of the PDE system.  
Berman and Peherstorfer~\cite{berman2024colora} combined LoRA~\cite{hu2021lora} with INR for both intrusive and non-intrusive ROMs.


\subsection{Modeling Temporal Evolution of Latent Variable}

Propagating latent variables forward in time plays an important role in reducing computational and memory requirements as compared to full-order models~\cite{lee2021deep, wiewel2019latent, li2022transformer}. One straightforward approach is to use tools from the time-series and sequence-modeling community, such as ARIMA~\cite{hamilton2020time}, LSTM~\cite{hochreiter1997long,mohan2019compressed}, TCN~\cite{hewage2020temporal,xu2020multi}, echo-state networks~\cite{chattopadhyay2020data}, and Transformers~\cite{vaswani2017attention,geneva2022transformers}. 
%
%
Alternatively, one could also choose a system-theoretic approach via Koopman operator~\cite{otto2019linearly,pan2020physics}, polynomial nonlinear dynamics~\cite{brunton2016discovering,peherstorfer2016data}, neural ODE~\cite{chen2018neural}.  
The key advantage of using system-theoretic approaches is the possibility of ensuring stability and direct integration with well-established techniques from system identification,  control, and state estimation communities. 
To ensure the stability of the learned model, one would typically need to work with certain simplified system dynamics. 
For linear models, Pan et al.~\cite{pan2020physics} built a stable Koopman operator by constructing a special tri-block diagonal matrix. 
For linear-quadratic models, Kaptanoglu et al.~\cite{kaptanoglu2021promoting} proposed a stability criterion that can be utilized to promote the global stability of learned models. Kramer~\cite{kramer2021stability} presented an optimization-based method to estimate the stability domain of quadratic-bilinear ROMs, offering a more scalable and less conservative approach than traditional analytical methods.  Goyal et al.~\cite{goyal2025guaranteed} developed inference formulations designed to construct quadratic models with built-in stability. Ouala et al.~\cite{ouala2023bounded} adopted the linear–quadratic dynamics and imposed a global boundedness constraint, enhancing the generalization of the learned low-dimensional ODE dynamics to arbitrary initial conditions.
Yin et al.~\cite{Yin2022Continuous} proposed DINo where spatial measurements are encoded through an INR autodecoder with Feature-wise Linear Modulation (FiLM) modulation~\cite{perez2018film} and evolved through a learned neural ODE~\cite{chen2018neural}. 
However, temporal stability, which is crucial for long-horizon rollout, is not guaranteed by a generic neural ODE parameterization. In fact, ensuring the temporal stability of data-driven models for dynamical systems is a well-known long-standing issue~\cite{yu2024learning}.
In contrast to existing INR-based frameworks~\cite{Yin2022Continuous,chen2023implicit,chen2023model,berman2024colora}, we combine coordinate-based autodecoding with structured latent dynamics and stability-promoting regularization. The analysis provides explicit latent-space stability certificates under stated spectral and conservation conditions, while retaining the mesh-agnostic and memory-efficient advantages of INR representations~\cite{pan2023neural}. 

\section{Proposed Framework}
\label{sec:method}

\begin{figure}[htbp]
    \centering

    \begin{subfigure}[t]{\textwidth}
        \centering
        \includegraphics[width=14cm]{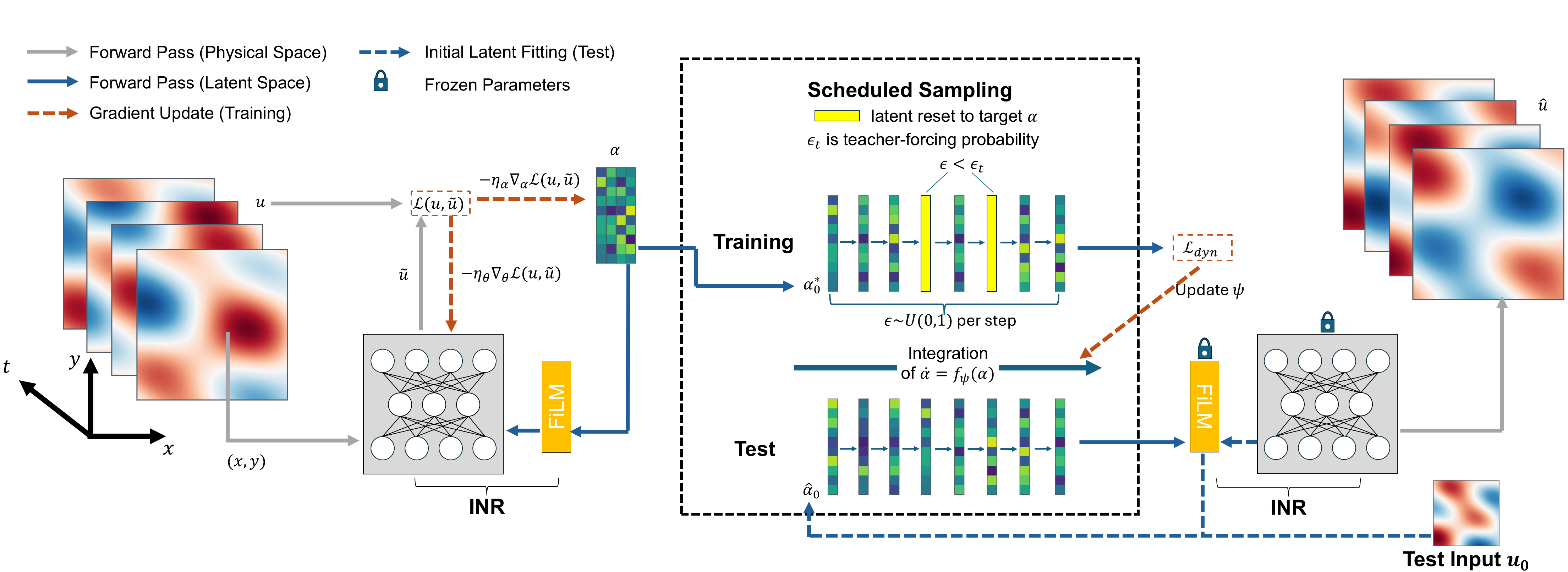}
        \caption{}
        \label{fig:subfig1_1}
    \end{subfigure}

    \begin{subfigure}[t]{\textwidth}
        \centering
        \includegraphics[width=15cm]{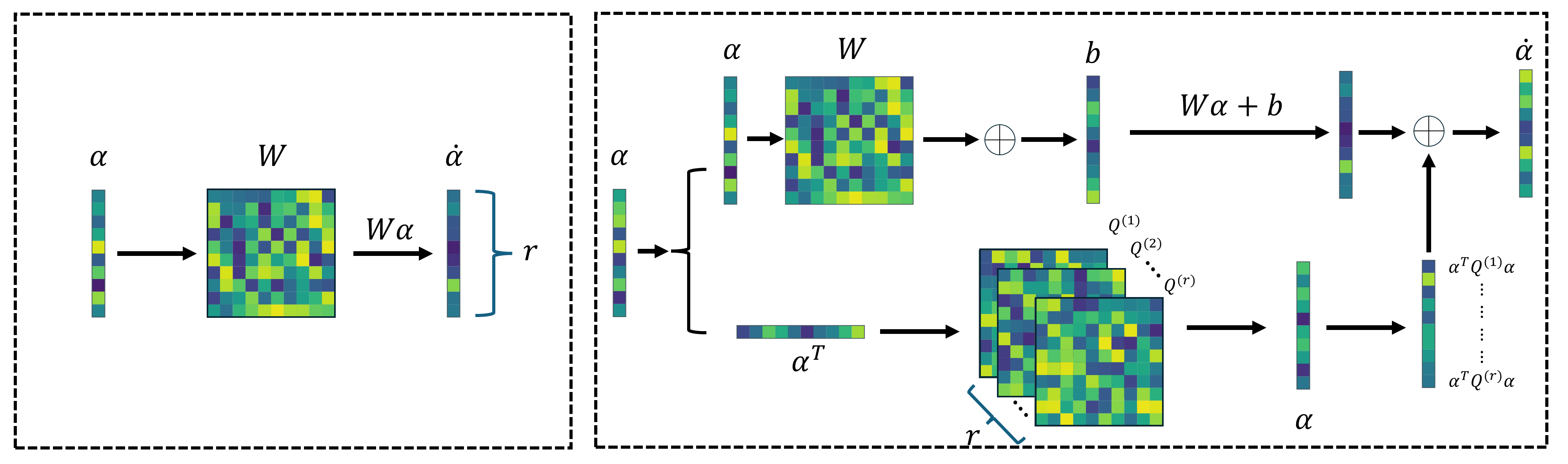}
        \caption{}
        \label{fig:subfig1_2}
    \end{subfigure}

    \begin{subfigure}[t]{\textwidth}
        \centering
        \includegraphics[width=15cm]{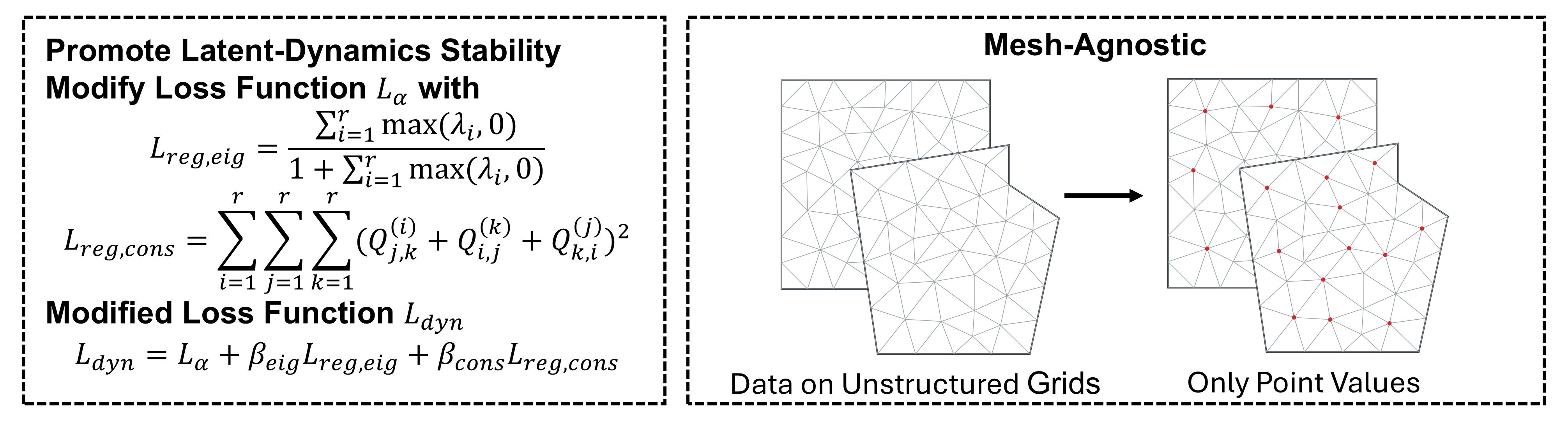}
        \caption{}
        \label{fig:subfig1_3}
    \end{subfigure}

    \caption{Overview of SMORE.
    (a) Overall SMORE workflow. An INR autodecoder fits compact latent states to observed field
    values. The latent dynamics advances these states in time, and the
    coordinate decoder reconstructs the corresponding solution fields.
    Training uses scheduled sampling, whereas testing rolls out from
    the inferred initial state.
    (b) SMORE-Koopman latent dynamics (left) and SMORE-LRLQ latent dynamics (right).
    (c) Lyapunov-guided stability regularization and coordinate decoding.
Stability is promoted by augmenting the dynamics loss $L_\alpha$ with
Lyapunov-guided regularizers (left). The eigenvalue term $L_{\text{reg,eig}}$ penalizes positive eigenvalues of the symmetric part $\mathbf{W}_s$ of the linear operator and is used by both SMORE-Koopman and SMORE-LRLQ, while the conservation term $L_{\text{reg,cons}}$, which encourages the
energy-preserving structure of the quadratic tensor $\mathbf{Q}$, is used only by
SMORE-LRLQ. Their combination yields the modified loss $L_{\text{dyn}}$.
Coordinate-based decoding then evaluates the reconstructed field at
arbitrary spatial query locations (right).}
    \label{fig:overview}
\end{figure}


We now present our framework SMORE in detail. As summarized in \cref{fig:overview}, it consists of an INR-based decoder $D_{\phi}\colon\Omega\to\mathbb{R}$ and a dynamics model for the latent state $\alpha \in \mathbb{R}^r$. The dependence of the decoded field on $\alpha_t$ is realized through FiLM~\cite{perez2018film}, while the temporal evolution of $\alpha_t$ is governed by a structured latent dynamics model with Lyapunov-guided stability regularization. Its two variants differ only in the latent dynamics.
SMORE-Koopman uses linear dynamics (\cref{eq:linear}), whereas SMORE-LRLQ
uses linear-quadratic dynamics with a low-rank factorization of the
quadratic tensor (\cref{eq:linear-quadratic,eq:low_rank}). This section is organized as follows. We first
introduce the INR-based spatial representation and the structured latent
dynamics models with their stability regularization, and show how
bounded latent states yield bounded reconstructed solution fields. We
then describe inference from sparse initial observations and the overall
training and evaluation workflow, including scheduled sampling and
pointwise minibatching.

\subsection{Latent Representation and Modulated Implicit Neural Decoding}
\label{sec:inr-autoencoder}
Instead of learning an explicit encoder, we adopt an autodecoding approach, in which the latent state of each snapshot is obtained by optimization (\cref{eq:inr-loss-1}). Although the components of the PDE state are generally dependent (in fluid systems, for example, different directional velocity components and other physical variables are often coupled through the governing equations), each component is decoded from its own latent code by a shared decoder, and the coupling is captured by the latent dynamics, which evolve the full latent state $\alpha_t$. Let $u^{(q)}(t,\cdot)$ denote the $q$-th component of the solution field, where $q=1,\dots,c$. We associate each component $u^{(q)}(t,\cdot)$ with a latent code $z_t^{(q)} \in \mathbb{R}^p$. Collecting these component-wise latent codes yields the latent representation in \cref{eq:all_latent}, where the total latent dimension satisfies $r = cp$:
\begin{equation}
\label{eq:all_latent}
\alpha_t =
\begin{bmatrix}
z_t^{(1)} \\
\vdots \\
z_t^{(c)}
\end{bmatrix}
\in \mathbb{R}^{r}, \qquad r = cp.
\end{equation}
Each component $u^{(q)}(t,\cdot)$ is approximated by a coordinate-based feedforward neural network $D_\phi\colon\Omega\to\mathbb{R}$ (the INR decoder) shared across components. Here, $\phi$ denotes the parameters of the base INR decoder.
The complete trainable parameter set $\theta$ includes $\phi$
and the FiLM modulation matrices. Given a component latent code $z\in\mathbb{R}^p$,
the map $\Phi_\theta(z)$ produces the effective decoder
parameters by replacing each modulated bias $b^{(l)}$ with
$b^{(l)}+A^{(l)}z$, while leaving the other parameters
unchanged. We denote the resulting decoder by
$D_{\Phi_\theta(z)}$. The decoded field is obtained by applying the shared decoder to each component code,
\begin{equation}
\label{eq:fullfield}
\hat{u}_\theta(\alpha_t)(\mathbf{x})
\coloneqq
\big[D_{\Phi_\theta(z_t^{(1)})}(\mathbf{x}),\dots,D_{\Phi_\theta(z_t^{(c)})}(\mathbf{x})\big]^\top .
\end{equation} 
For brevity, we also write $D_{\Phi_\theta(\alpha)}\coloneqq\hat{u}_\theta(\alpha)$ for the component-wise decoding of the full latent state.
The latent state of each snapshot is obtained by autodecoding, 
\begin{equation}
\label{eq:inr-loss-1}
    \alpha_t \in \argmin_{\alpha\in \mathbb{R}^{r}} L_{\mathrm{rec}}\big(\hat{u}_\theta(\alpha), u(t, \cdot)\big).
\end{equation}
Here, $L_{\mathrm{rec}}(f,g)\coloneqq\sum_{q=1}^{c}\|f^{(q)}-g^{(q)}\|_{L^2(\Omega)}^2$ measures reconstruction
error across field components. In practice, this minimization is solved approximately by gradient-based optimization. The FiLM construction is detailed
in Appendix~\ref{apdx:film}. Although one could choose arbitrary INR frameworks, for simplicity, we consider the multiplicative filter network (MFN)~\cite{fathony2020multiplicative} as $D_\phi$. Details of MFN are presented in Appendix~\ref{apdx:mfn}. In practice, $L_{\mathrm{rec}}$ is evaluated empirically on $\Omega_{N_x}$ during training, and on $\Omega^{\mathrm{init}}_{N_x}$ when the initial latent state is inferred at test time.

\subsection{Stability-Promoting Latent Dynamics}
\label{sec:dynamic-model}

Given the latent states obtained by autodecoding, we model their
temporal evolution through $\dot{\alpha}=f_\psi(\alpha)$.
We consider linear and linear-quadratic parameterizations of
$f_\psi$, both of which are continuously differentiable, hence locally
Lipschitz, so the latent ODE admits a unique local solution for
every initial state.
We approximate the latent trajectories using a fourth-order
Runge--Kutta scheme and train the dynamics parameters using
the continuous adjoint method~\cite{chen2018neural}.
To address rapidly growing rollouts and ill-conditioned gradients,
we incorporate Lyapunov-guided stability regularization.

\subsubsection{Linear Latent Dynamics} 
Koopman theory describes the evolution of observables of a
nonlinear dynamical system through a linear operator on a
generally infinite-dimensional function space~\cite{mezic2023operator}.
When a finite set of observables spans an invariant subspace,
their evolution admits a closed finite-dimensional linear
representation. In continuous time, the corresponding matrix
represents the Koopman generator restricted to that subspace.
Motivated by this construction and prior work on learning linear latent 
dynamics~\cite{pan2020physics,constante2024data}, we jointly learn
latent coordinates and a matrix $\mathbf{W}$ to approximate
their evolution by  
\begin{equation}
    \label{eq:linear}
    \dot{\alpha}_t = \mathbf{W}\alpha_t, 
\end{equation}
where $\mathbf{W}\in\mathbb{R}^{r\times r}$ is trainable.
The learned coordinates provide an approximate linear
representation; exact Koopman invariance is not assumed.

\subsubsection{Linear-Quadratic Latent Dynamics} 
For highly nonlinear PDE systems, the finite-dimensional linear dynamics above may be restrictive. We therefore use a linear-quadratic system~\cite{goyal2025guaranteed, ouala2023bounded}, 
\begin{equation}
    \label{eq:linear-quadratic}
    \dot{\alpha} =  \mathbf{W}\alpha +  \mathbf{b} + \begin{bmatrix} \alpha^\top \mathbf{Q}^{(1)} \alpha & \cdots & \alpha^\top \mathbf{Q}^{(r)} \alpha \end{bmatrix}^\top, 
\end{equation} 
with trainable parameters $\mathbf{b} \in\mathbb{R}^{r}$, $\mathbf{W}\in\mathbb{R}^{r \times r}$, and $\forall i \in \{1,\ldots,r\}, \mathbf{Q}^{(i)} \in \mathbb{R}^{r\times r}$. Here, $\alpha \in \mathbb{R}^r$ denotes the concatenated latent state defined in \cref{eq:all_latent}, where $r=cp$. The linear model \cref{eq:linear} is the special case $\mathbf{b}=\mathbf{0}$, $\mathbf{Q}^{(i)}=\mathbf{0}$.

\subsubsection{Lyapunov-Guided Stability Regularization}
To prevent numerical overflow in both the rollout and back-propagation caused by an \emph{unbounded} forward pass through \cref{eq:linear,eq:linear-quadratic}, we introduce Lyapunov-guided stability regularization for the latent ODE. The design is guided by the following dissipativity criterion for linear-quadratic systems~\cite{schlegel2015long, kaptanoglu2021promoting, ouala2023bounded, goyal2023guaranteed}.  The criterion separates a strictly dissipative case, which yields an explicit absorbing radius, from the marginal case $\gamma=0$ used as the practical training target, which excludes finite-time blow-up.
\begin{theorem}
\label{thm:1}
Consider \cref{eq:linear-quadratic}. Since only the symmetric part of each quadratic slice contributes to $\alpha^\top \mathbf{Q}^{(i)}\alpha$, assume without loss of generality that $\mathbf{Q}^{(i)}=(\mathbf{Q}^{(i)})^\top$. For a shift $\mathbf{m}\in\mathbb{R}^r$, define $\mathbf{y}=\alpha-\mathbf{m}$,
\begin{equation}
\label{eq:shifted_d}
    \mathbf{d}
    =
    \mathbf{b}+\mathbf{W}\mathbf{m}
    +
    \begin{bmatrix}
    \mathbf{m}^\top\mathbf{Q}^{(1)}\mathbf{m} & \cdots &
    \mathbf{m}^\top\mathbf{Q}^{(r)}\mathbf{m}
    \end{bmatrix}^{\top},
\end{equation}
and
\begin{equation}
\label{eq:shifted_A}
    \mathbf{A}
    =
    \mathbf{W} + \begin{bmatrix}
    2\mathbf{m}^\top \mathbf{Q}^{(1)}\\
    \vdots \\
    2\mathbf{m}^\top \mathbf{Q}^{(r)}
    \end{bmatrix}.
\end{equation}
These are the constant and linear coefficients of the shifted system (\cref{eq:app_shifted_lq}). If the following two conditions hold:
\begin{enumerate}
    \item there exists $\gamma\geq0$ such that the symmetric part $\mathbf{A}_s=(\mathbf{A}+\mathbf{A}^\top)/2$ satisfies $\mathbf{A}_s\preceq -\gamma \mathbf{I}$,
    \item $\forall i,j,k \in \{1,\ldots,r\}$, $\mathbf{Q}^{(i)}_{j,k} + \mathbf{Q}^{(k)}_{i,j} + \mathbf{Q}^{(j)}_{k,i} = 0$,  
\end{enumerate}
then the shifted energy satisfies
\begin{equation*}
    \dot V
    \leq
    \Vert \mathbf{d}\Vert\Vert\mathbf{y}\Vert
    -
    \gamma \Vert \mathbf{y}\Vert^2,
    \qquad
    V(\mathbf{y})=\frac{1}{2}\Vert\mathbf{y}\Vert^2.
\end{equation*}
Consequently, solutions exist for all $t\geq0$. 
If $\gamma>0$, the shifted latent state is ultimately bounded and satisfies
\begin{equation}
\label{eq:hard_absorbing_radius}
    \limsup_{t\to\infty}\Vert \alpha_t-\mathbf{m}\Vert
    \leq
    \frac{\Vert \mathbf{d}\Vert}{\gamma}.
\end{equation}
Moreover, if $\mathbf{d}=\mathbf{0}$ and $\gamma\geq0$, then
\begin{equation}
\label{eq:hard_exponential_decay}
    \Vert \alpha_t-\mathbf{m}\Vert
    \leq
    e^{-\gamma t}\Vert \alpha(0)-\mathbf{m}\Vert .
\end{equation}
In particular, the boundary case $\gamma=0$ gives a non-expansive energy estimate $\Vert \alpha_t-\mathbf{m}\Vert\leq\Vert \alpha(0)-\mathbf{m}\Vert$ when $\mathbf{d}=\mathbf{0}$.
If $\gamma=0$ and $\mathbf{d}\neq\mathbf{0}$, the growth is at most linear,
\begin{equation}
\label{eq:hard_linear_growth}
    \Vert \alpha_t-\mathbf{m}\Vert
    \leq
    \Vert \alpha(0)-\mathbf{m}\Vert + \Vert\mathbf{d}\Vert\, t.
\end{equation}
\end{theorem}
\noindent \emph{The proof is provided in \Cref{apdx:boundedness}.}

\begin{remark}
The implementation uses $\mathbf{m}=\mathbf{0}$, so
$\mathbf{A}_s=\mathbf{W}_s=(\mathbf{W}+\mathbf{W}^{\top})/2$ and
$\mathbf{d}=\mathbf{b}$.  This is the most conservative choice. Because $\mathbf{A}$ depends affinely on $\mathbf{m}$, a trained model that violates $\mathbf{W}_s\preceq0$ may still admit some $\mathbf{m}\neq\mathbf{0}$ with $\mathbf{A}_s\preceq-\gamma\mathbf{I}$. 
Define the certificate margin as
$\gamma_{\mathrm{cert}}=-\lambda_{\max}(\mathbf{W}_s)$. For SMORE-Koopman, $\mathbf{b}=\mathbf{0}$ and
$\mathbf{Q}^{(i)}=\mathbf{0}$, so the cyclic condition holds trivially
and $\gamma_{\mathrm{cert}}\geq0$ alone yields the non-expansive
estimate in \cref{eq:hard_exponential_decay}.
For SMORE-LRLQ, $\mathbf{d}=\mathbf{b}$ is generally nonzero.
A positive margin, together with the cyclic cancellation condition in \Cref{thm:1}, yields the absorbing radius $\Vert\mathbf{b}\Vert/\gamma_{\mathrm{cert}}$ 
in \cref{eq:hard_absorbing_radius}, whereas a zero margin yields the
linear-growth bound in \cref{eq:hard_linear_growth}.
Training targets $\gamma_{\mathrm{cert}}\geq0$, i.e., the zero-margin target  $\gamma_{\mathrm{tgt}}=0$, through finite-weight penalties on positive eigenvalues of $\mathbf{W}_s$ and
violations of cyclic energy cancellation.
These penalties encourage the structural conditions but do not
guarantee that $\lambda_{\max}(\mathbf{W}_s)\leq0$ or that the cyclic identity holds exactly. The theoretical guarantees apply to a trained
model only when the stated conditions hold, which can be checked
a posteriori by evaluating $\lambda_{\max}(\mathbf{W}_s)$ and the
cyclic defect in \cref{eq:cyclic_defect_tensor}.
\end{remark}

The preceding Lyapunov analysis motivates two soft penalties. One
penalizes positive eigenvalues of the symmetric linear part, and the
other penalizes violations of cyclic energy cancellation in the
quadratic term.
To promote the marginal spectral condition in \Cref{thm:1}, we define
\begin{equation}
\label{eq:reg_1}
L_{\text{reg,eig}} = \frac{ \sum_{i=1}^{r} \max(\lambda_i, 0)}{ 1+ \sum_{i=1}^{r} \max(\lambda_i, 0)}. 
\end{equation}
Here, $\lambda_1,\ldots,\lambda_r$ are the eigenvalues of
$\mathbf{W}_s$. The loss penalizes their positive part. Because
$\lambda_{\max}(\mathbf{W}_s)$ bounds the instantaneous growth rate of
$V$ due to the linear part, and
$\lambda_{\max}(\mathbf{W}_s)\geq\max_i\operatorname{Re}\lambda_i(\mathbf{W})$,
this penalty also suppresses transient growth that a condition on the
spectrum of $\mathbf{W}$ alone would permit. The normalization keeps
$L_{\text{reg,eig}}\in[0,1)$.

To encourage the cyclic energy cancellation in \Cref{thm:1}, we define
\begin{equation}
\label{eq:reg_2}
    L_{\text{reg,cons}} = \sum_{i=1}^r \sum_{j=1}^r \sum_{k=1}^r 
\left(\mathbf{Q}^{(i)}_{j,k} 
+ \mathbf{Q}^{(k)}_{i,j}
+ \mathbf{Q}^{(j)}_{k,i}
\right)^2.
\end{equation}
The two regularizers are added to the latent-dynamics loss $L_\alpha$ with fixed coefficients 
$\beta_{\text{eig}}$ and $\beta_{\text{cons}}$, yielding
\begin{equation}
    \label{eq:total_loss}
    L_{\text{dyn}}
    = L_{\alpha}
    + \beta_{\text{eig}}\, L_{\text{reg,eig}}
    + \beta_{\text{cons}}\, L_{\text{reg,cons}},
\end{equation}
where $L_{\alpha}$ is the latent-dynamics loss
\begin{equation}
\label{eq:latent_dyn_loss}
    L_{\alpha}
    =
    \frac{1}{|\mathcal{I}|\,N_{\text{train}}\,r}
    \sum_{i\in\mathcal{I}}\sum_{l=1}^{N_{\text{train}}}
    \big\Vert \hat{\alpha}_{i,l}-\alpha^{*}_{i,l}\big\Vert^2 .
\end{equation}
Here, $\mathcal{I}$ is the minibatch of training trajectories,
$\alpha^{*}_{i,l}$ is the latent state of snapshot $l$ of trajectory $i$
obtained by autodecoding (\cref{eq:inr-loss-1}), and
$\hat{\alpha}_{i,l}$ is the prediction obtained by integrating
$f_\psi$ from $\alpha^{*}_{i,1}$ with scheduled sampling
(\Cref{sec:scheduled_sampling}). The targets $\alpha^{*}_{i,l}$ are held
fixed when differentiating $L_{\alpha}$, so this loss updates only the
dynamics parameters $\psi$ and not the decoder or the latent states
(\Cref{sec:overall_workflow}). For SMORE-Koopman, the quadratic term vanishes, so $\beta_{\text{cons}}=0$ in \cref{eq:total_loss}. Because the regularizers depend only on the dynamics parameters, the implementation scales them by $1/N_b$ per minibatch, where $N_b$ is the number of minibatches per epoch, so that they enter once per epoch (\Cref{alg:smore-tr}).

\begin{remark}
The conservation regularizer $L_{\text{reg,cons}}$ has a direct certificate interpretation. Let
\begin{equation}
\label{eq:cyclic_defect_tensor}
    T_{ijk}
    =
    \mathbf{Q}^{(i)}_{j,k}
    + \mathbf{Q}^{(k)}_{i,j}
    + \mathbf{Q}^{(j)}_{k,i},
    \qquad
    \eta_Q=\frac{1}{3}\Vert T\Vert_F
    =
    \frac{1}{3}L_{\text{reg,cons}}^{1/2}.
\end{equation}
Then the quadratic energy defect obeys
\begin{equation}
\label{eq:soft_defect_bound}
    \left|
    \sum_{i=1}^{r} y_i\,\mathbf{y}^{\top}\mathbf{Q}^{(i)}\mathbf{y}
    \right|
    \leq
    \eta_Q \Vert \mathbf{y}\Vert^3.
\end{equation}
With a small defect in place of the exact cyclic identity, the Lyapunov estimate becomes
\begin{equation}
\label{eq:soft_lyapunov}
    \dot V
    \leq
    \Vert \mathbf{d}\Vert\Vert\mathbf{y}\Vert
    -
    \gamma\Vert\mathbf{y}\Vert^2
    +
    \eta_Q\Vert\mathbf{y}\Vert^3.
\end{equation}
When a positive certificate margin $\gamma>0$ is available, $\eta_Q>0$, and $4\Vert\mathbf{d}\Vert\eta_Q<\gamma^2$, the interval
\begin{equation}
\label{eq:soft_radius_interval}
    \left(\frac{2\Vert\mathbf{d}\Vert}{\gamma},
    \frac{\gamma}{2\eta_Q}\right]
\end{equation}
is nonempty, and for every $R_{\mathrm{cert}}$ in it the ball $\{ \|\alpha - \mathbf{m}\| \le R_{\mathrm{cert}} \}$ is forward invariant. Reducing $L_{\text{reg,cons}}$ enlarges the upper endpoint and hence the certified region. This soft certificate requires a strictly positive margin and is therefore not implied by the zero-margin training target alone. If $\eta_Q=0$, the upper endpoint is not needed and the exact cyclic case recovers \Cref{thm:1}.
\end{remark}
 
\subsubsection{Low-Rank Parameterization of the Quadratic Tensor}
For SMORE-LRLQ, storing the full quadratic tensor $\{\mathbf{Q}^{(i)}\}_{i=1}^{r}$ becomes prohibitively expensive as the latent dimension $r$ increases, requiring $O(r^3)$ parameters. To reduce this cost, SMORE-LRLQ adopts a low-rank parameterization of the quadratic term as a standard architectural component.
Specifically, we use a rank-$\kappa$ factorization for each output
index $i\in\{1,\ldots,r\}$:
\begin{equation}
\label{eq:low_rank}
    \mathbf{Q}_{\mathrm{low}}^{(i)} = \mathbf{A}_i \mathbf{B}_i,
    \quad \mathbf{A}_i \in \mathbb{R}^{r \times \kappa},\;
    \mathbf{B}_i \in \mathbb{R}^{\kappa \times r},
\end{equation}
which reduces the quadratic parameter count from $r^3$ to $2r^2\kappa$. Since a quadratic form depends only on the symmetric part of its matrix, the effective slice used in the theory is
\begin{equation}
\label{eq:low_rank_effective}
    \mathbf{Q}_{\mathrm{eff}}^{(i)}
    =
    \frac{\mathbf{A}_i\mathbf{B}_i+\mathbf{B}_i^\top\mathbf{A}_i^\top}{2}.
\end{equation}
The latent dynamics in \cref{eq:linear-quadratic} is unchanged because
$\alpha^\top\mathbf{A}_i\mathbf{B}_i\alpha=\alpha^\top\mathbf{Q}_{\mathrm{eff}}^{(i)}\alpha$.
Because $\mathbf{A}_i\mathbf{B}_i$ is generally non-symmetric, its
symmetrized effective slice has rank at most $2\kappa$ rather than
$\kappa$, so the factorization spans a broader class than symmetric
rank-$\kappa$ slices. It should therefore be viewed as a scalable
low-rank parameterization of the quadratic term, not as a structural
constraint that enforces a symmetric, cyclic form. Accordingly, $L_{\text{reg,cons}}$ in \cref{eq:reg_2} and the defect
$\eta_Q$ in \cref{eq:cyclic_defect_tensor} are evaluated on the
effective symmetric slices $\{\mathbf{Q}_{\mathrm{eff}}^{(i)}\}_{i=1}^r$,
since the low-rank factorization improves scalability but does not by
itself enforce the cyclic identity.

In summary, SMORE-LRLQ combines three mechanisms that can contribute
to robust latent evolution. The Lyapunov-guided regularizers in
\cref{eq:reg_1,eq:reg_2} promote the certificate conditions of
\Cref{thm:1}. The low-rank quadratic parameterization in
\cref{eq:low_rank} carries no certificate by itself, but it restricts
the capacity of the quadratic term and may act as an implicit
regularizer. Gradient clipping is applied to the decoder and dynamics
parameters (\Cref{alg:smore-tr}) to bound individual updates when gradients
become ill conditioned. Only the first mechanism is tied to the theory
above; the contribution of each is assessed empirically in
\Cref{sec:results_stability_robustness}.

\subsection{Boundedness and Error Bounds for Reconstructed Fields}
\label{sec:decoder_transfer}

The stability criterion above concerns the latent dynamics.
To relate this analysis to the predicted solution, we show that
bounded latent states produce bounded reconstructed fields and
that perturbations in the latent state lead to bounded changes
in the decoder output. The result relies on two properties of the SMORE decoder. Its Fourier filters are bounded, and the latent code enters only through the shift modulation in \cref{eq:film}. Since the MFN applies no nonlinearity to its hidden state other than multiplication by the $\mathbf{x}$-dependent filters, the decoder is therefore affine in the latent code.

\begin{theorem}[Decoder growth and Lipschitz transfer]
\label{thm:decoder_transfer}

Assume $\Omega$ is compact. Since the Fourier filters in \cref{eq:fourier-embedding} are bounded by one, the FiLM-modulated MFN decoder has bounded spatial filters. Then there exist constants $C_D,D_D,L_D\ge0$, depending only on the trained decoder weights and on $\Omega$, such that for all latent codes $\alpha,\beta\in\mathbb R^r$,
\begin{equation}
\label{eq:decoder_growth_method}
    \sup_{\mathbf{x}\in\Omega}
    \left\Vert D_{\Phi_\theta(\alpha)}(\mathbf{x})\right\Vert
    \le
    C_D+D_D\Vert\alpha\Vert,
\end{equation}
and
\begin{equation}
\label{eq:decoder_lipschitz_transfer_method}
    \left\Vert
    D_{\Phi_\theta(\alpha)}
    -
    D_{\Phi_\theta(\beta)}
    \right\Vert_{L^2(\Omega)}
    \le
    L_D\Vert\alpha-\beta\Vert.
\end{equation}
If a latent trajectory satisfies $\Vert\alpha_t-\mathbf m\Vert\le R$ for all $t\in[0,T]$, then for all such $t$ the decoded field is bounded by
\begin{equation}
\label{eq:decoded_field_bound_method}
    \left\Vert D_{\Phi_\theta(\alpha_t)}\right\Vert_{L^2(\Omega)}
    \le
    |\Omega|^{1/2}
    \left[
    C_D+D_D(\Vert\mathbf m\Vert+R)
    \right].
\end{equation}
In particular, under the conditions of \Cref{thm:1}, \cref{eq:app_comparison} supplies $R$ over a rollout horizon $[0,T]$:
$R=\Vert\alpha(0)-\mathbf m\Vert$ if $\mathbf d=\mathbf 0$;
$R=\max\{\Vert\alpha(0)-\mathbf m\Vert,\ \Vert\mathbf d\Vert/\gamma\}$ if $\gamma>0$; and
$R=\Vert\alpha(0)-\mathbf m\Vert+\Vert\mathbf d\Vert T$ if $\gamma=0$.
With the implementation choice $\mathbf m=\mathbf 0$, the first case applies to SMORE-Koopman and the last to SMORE-LRLQ at the target $\gamma_{\mathrm{tgt}}=0$, so marginal latent stability still yields a finite bound on the decoded field over any finite rollout horizon.

\end{theorem}
\begin{remark}[Explicit constants]
Since the decoder is affine in the latent code, for each component
$D_{\Phi_\theta(z)}(\mathbf{x})=c_0(\mathbf{x})+J_0(\mathbf{x})z$ with
$J_0(\mathbf{x})\in\mathbb{R}^{1\times p}$ independent of $z$, and hence
$D_{\Phi_\theta(\alpha)}(\mathbf{x})=c_0(\mathbf{x})\mathbf{1}_c+(\mathbf{I}_c\otimes J_0(\mathbf{x}))\alpha$.
The constants in \Cref{thm:decoder_transfer} can therefore be taken as
$C_D=\sqrt{c}\,\sup_{\mathbf{x}\in\Omega}|c_0(\mathbf{x})|$,
$D_D=\sup_{\mathbf{x}\in\Omega}\Vert J_0(\mathbf{x})\Vert$, and
$L_D=\lambda_{\max}(G)^{1/2}$ with
$G=\int_\Omega J_0(\mathbf{x})^\top J_0(\mathbf{x})\,d\mathbf{x}$.
The Lipschitz bound \cref{eq:decoder_lipschitz_transfer_method} is then
attained, and $L_D\le|\Omega|^{1/2}D_D$.
\end{remark}

\noindent \emph{The proof is provided in Appendix~\ref{apdx:decoder_transfer}.}

\subsection{Sparse-to-Rollout Inference from Partial Initial Measurements}
\label{sec:sparse_to_rollout_theory}

Sparse initialization estimates a low-dimensional latent state
rather than an independent value at every grid point.
A small fraction of the initial field can therefore provide
enough information for recovery, provided that the observed
locations distinguish the relevant latent states.
Once the initial state is inferred, the learned dynamics
predicts its evolution and the decoder reconstructs the future
solution fields. We first characterize when the observations
identify the initial latent state, then bound the resulting
prediction error.

Let $\mathcal S=\{\mathbf{x}_1,\ldots,\mathbf{x}_s\}\subset\Omega$ denote the observed spatial locations in the first snapshot, i.e., $\mathcal{S} = \Omega_{N_x}^{\mathrm{init}}$. Define the sparse measurement operator and the latent-to-sensor map as
\begin{equation}
\label{eq:sparse_sensor_map_method}
    P_{\mathcal S}u=
    \begin{bmatrix}
    u(\mathbf{x}_1)\\ \vdots\\ u(\mathbf{x}_s)
    \end{bmatrix},
    \qquad
    G_{\mathcal S}(\alpha)=P_{\mathcal S}D_{\Phi_\theta(\alpha)} .
\end{equation}

At test time, given only $y_{\mathrm{obs}}=P_{\mathcal S}u_0+\xi$, the initial latent state is inferred by fitting the decoder to the sparse measurements,
\begin{equation}
\label{eq:sparse_autodecoding_method}
    \widehat\alpha_0
    \approx
    \argmin_{\alpha\in\mathcal B}
    \left\Vert G_{\mathcal S}(\alpha)-y_{\mathrm{obs}}\right\Vert^2,
\end{equation}
where $\mathcal B\subset\mathbb R^r$ is the latent region explored during inference. The optimization slack $\delta_{\mathrm{opt}}\ge0$ bounds how much the residual norm of the computed fit exceeds the smallest achievable residual norm on $\mathcal B$,
\[
\Vert G_{\mathcal S}(\widehat\alpha_0)-y_{\mathrm{obs}}\Vert
\le
\inf_{\alpha\in\mathcal B}
\Vert G_{\mathcal S}(\alpha)-y_{\mathrm{obs}}\Vert
+\delta_{\mathrm{opt}},
\]
and accounts for incomplete optimization, separately from measurement noise and decoder representation error. 

With the decoder fixed, changing the observation mask changes
the fitting objective but requires no retraining.
The following lemma quantifies how sensitively the observed
values respond to changes in the latent state.

\begin{lemma}[Sensor Jacobian and restricted observability]
\label{lem:sensor_jacobian}

Let $c$ denote the number of field components. Define the sensor Jacobian $J_{\mathcal S}(\alpha)=P_{\mathcal S}\,\partial_\alpha D_{\Phi_\theta(\alpha)}\in\mathbb R^{(s c)\times r}$. Here we use the lower singular value
\[
    \sigma_{\min}(J)
    :=
    \inf_{\Vert v\Vert=1}\Vert Jv\Vert
    =
    \lambda_{\min}(J^\top J)^{1/2}.
\]
Then:
\begin{enumerate}
\item if $s c<r$, then $\sigma_{\min}(J_{\mathcal S}(\alpha))=0$ and the latent code is not identifiable to first order from $\mathcal S$;
\item enlarging a nested observed set $\mathcal S$ appends rows to $J_{\mathcal S}$, so $\sigma_{\min}(J_{\mathcal S}(\alpha))$ is nondecreasing in $\mathcal S$;
\item since the decoder is affine in the latent code (\Cref{sec:decoder_transfer}), $J_{\mathcal S}$ does not depend on $\alpha$ and $G_{\mathcal S}(\alpha)-G_{\mathcal S}(\beta)=J_{\mathcal S}(\alpha-\beta)$. Hence
\begin{equation}
\label{eq:local_observability}
    \left\Vert G_{\mathcal S}(\alpha)-G_{\mathcal S}(\beta)\right\Vert
    \ge
    \sigma_{\min}(J_{\mathcal S})\Vert\alpha-\beta\Vert
\end{equation}
for all $\alpha,\beta\in\mathbb R^r$, so the restricted-observability constant in Appendix~\ref{apdx:sparse_recovery} can be taken as $c_{\mathcal S}=\sigma_{\min}(J_{\mathcal S})$ with $\mathcal B=\mathbb R^r$, and no larger constant is possible.
\end{enumerate}

\end{lemma}
\begin{proof}
For (1), $J_{\mathcal S}$ has $s c<r$ rows and hence a nontrivial right null space, along which $G_{\mathcal S}$ is first-order insensitive. For (2), appending a row $v^\top$ gives $J_{\mathcal S}'^\top J_{\mathcal S}'=J_{\mathcal S}^\top J_{\mathcal S}+v v^\top\succeq J_{\mathcal S}^\top J_{\mathcal S}$, so no eigenvalue of $J_{\mathcal S}^\top J_{\mathcal S}$, and in particular the lower singular value, can decrease. For (3), the identity follows from the affine form of the decoder, and the bound follows from the definition of $\sigma_{\min}$, with equality when $\alpha-\beta$ is a corresponding right singular vector.
\end{proof}
\begin{remark}
At $10\%$ visibility in the reported experiments, the number
of observed scalar values $sc$ exceeds the latent dimension
$r$, so the fitting problem is overdetermined by equation
count. Recovery still depends on the information carried by those observations,
since redundant measurements or weak sensitivity to some latent
directions can make the problem ill conditioned. The smallest singular
value of the sensor Jacobian provides a diagnostic of this
sensitivity.
\end{remark}
The following certificate, adapted from the sparse-measurement and rollout estimates in Appendix~\ref{apdx:sparse_recovery}, states conditions under which sparse initialization leads to accurate future full-field prediction.

\begin{theorem}[Sparse-to-rollout full-field certificate]
\label{thm:sparse_to_rollout}
Let $\sigma$ bound the measurement noise, $\|\xi\|\le\sigma$, let $\varepsilon_{\mathrm{sens}}$ bound the decoder representation error of the initial state at the sensors, and let $c_{\mathcal S}=\sigma_{\min}(J_{\mathcal S})$ be the observability constant from \Cref{lem:sensor_jacobian}. Assume the sparse-recovery hypotheses stated in Appendix~\ref{apdx:sparse_recovery} hold on $\mathcal B$, and define
\begin{equation}
\label{eq:e_sparse_method}
    e_{\mathrm{sparse}}
    =
    \frac{2(\varepsilon_{\mathrm{sens}}+\sigma)+\delta_{\mathrm{opt}}}{c_{\mathcal S}}.
\end{equation}
Let $\widehat\alpha_n$ be the latent rollout initialized from $\widehat\alpha_0$, and let $\alpha_n^\star$ be an ideal latent trajectory representing the target field $u_n$, with $\alpha_0^\star=\alpha_0$ from the sparse-recovery statement. Suppose that the latent rollout error $E_n=\Vert\widehat\alpha_n-\alpha_n^\star\Vert$ satisfies, for $n=0,\ldots,N-1$ over a finite horizon under consideration $0\le n\le N$,
\begin{equation}
\label{eq:sparse_rollout_recursion_method}
    E_{n+1}\le aE_n+b_{\mathrm{roll}},
    \qquad a\ge0,\quad b_{\mathrm{roll}}\ge0.
\end{equation}
Let $L_D$ be the Lipschitz constant of the decoder from \Cref{thm:decoder_transfer}, and assume that the target fields have decoder representation error
\begin{equation}
\label{eq:decoder_rep_error_method}
    \left\Vert D_{\Phi_\theta(\alpha_n^\star)}-u_n\right\Vert_{L^2(\Omega)}
    \le
    \varepsilon_{\mathrm{dec}}.
\end{equation}
Then the decoded full-field rollout satisfies
\begin{equation}
\label{eq:sparse_to_rollout_bound_method}
    \left\Vert
    D_{\Phi_\theta(\widehat\alpha_n)}-u_n
    \right\Vert_{L^2(\Omega)}
    \le
    \varepsilon_{\mathrm{dec}}
    +
    L_D
    \left(
    a^n e_{\mathrm{sparse}}
    +
    b_{\mathrm{roll}}\sum_{\ell=0}^{n-1}a^\ell
    \right),
\end{equation}
for every $0\le n\le N$. If the recurrence and decoder bounds hold uniformly for all $n\ge0$ and $0\le a<1$, then
\begin{equation}
\label{eq:sparse_to_rollout_limsup_method}
    \limsup_{n\to\infty}
    \left\Vert
    D_{\Phi_\theta(\widehat\alpha_n)}-u_n
    \right\Vert_{L^2(\Omega)}
    \le
    \varepsilon_{\mathrm{dec}}
    +
    \frac{L_D b_{\mathrm{roll}}}{1-a}.
\end{equation}

\end{theorem}

\noindent The proof is provided in Appendix~\ref{apdx:sparse_to_rollout_proof}.

\begin{remark}

\Cref{thm:sparse_to_rollout} gives the theoretical counterpart of the sparse-mask experiment. The visible fraction of the first snapshot affects the
restricted-observability constant $c_{\mathcal S}$, as fewer sensors can
make the inverse problem less identifiable and increase
$e_{\mathrm{sparse}}$. The latent dynamics then determines how this initial error propagates through $a$ and $b_{\mathrm{roll}}$, while the decoder Lipschitz constant $L_D$ transfers latent error into full-field error. The recurrence in \cref{eq:sparse_rollout_recursion_method} is a local or finite-horizon rollout assumption rather than an automatic consequence of \Cref{thm:1}; it can be justified by a locally contractive learned flow or checked empirically on the trained model. Under these explicit assumptions, a model that sees only a sparse initial snapshot can still reconstruct the full future field accurately.

\end{remark}

\subsection{Scheduled Sampling}
\label{sec:scheduled_sampling}
Following Yin et al.~\cite{Yin2022Continuous}, we adopt \emph{scheduled sampling}
during latent dynamics training to stabilize the early stage of optimization.
Before the dynamics pass, the latent states $\bm{\alpha}_t$ are updated by
minimizing the reconstruction loss through the modulated INR decoder $D_{\Phi_\theta(\alpha)}$. The
post-update latent states $\bm{\alpha}^{*}_{t}$ are then detached and used as
the target latent trajectory for latent-dynamics training.
Starting from the initial target latent state $\bm{\alpha}^{*}_{t_1}$, the
dynamics network rolls out the latent trajectory over time. Each intermediate
time point $t_i$ is independently selected as a \emph{teacher-forcing anchor}
with probability $\varepsilon_t$. Whenever the rollout reaches such an anchor,
the current predicted latent state is reset to the corresponding target latent
state $\bm{\alpha}^{*}_{t_i}$; otherwise, the dynamics network continues
autoregressively from its own previous prediction. After the rollout is
completed, the predicted latent trajectory is compared with the target latent
trajectory for latent-dynamics training.

The teacher-forcing probability $\varepsilon_t$ decays exponentially during
training. Early in training, frequent resets divide the trajectory
into many short autoregressive segments and limit the accumulation of rollout
errors. Later in training, anchors become increasingly sparse, and the dynamics
network is forced to predict longer latent trajectories without intermediate
target-latent resets. In the limit, the dynamics model produces a fully
autoregressive latent rollout initialized only at
$\bm{\alpha}^{*}_{t_1}$. This approaches the test-time latent rollout setting,
where the initial latent state is inferred from the initial snapshot and no
intermediate target latent states are available. This schedule also changes the interaction between the decoder pass and
the dynamics model. In the early phase, the dynamics model frequently
consumes target latent states obtained from the current reconstruction
pass, whereas later training relies increasingly on latent states
generated by the dynamics model itself. Additional implementation
details are provided in Appendix~\ref{apdx:scheduled_sampling}.

\subsection{Pointwise Decoder Minibatching}
\label{sec:pointwise_minibatching}

We introduce a pointwise minibatching mechanism to address potential GPU memory shortages. The coordinate-based structure of the INR decoder allows spatial query
locations to be processed independently for a fixed latent state.
This property can be exploited to reduce the peak memory required by the
decoder and autodecoding updates at high spatial resolutions.
Consider a training minibatch indexed by $\mathcal{I}$, and let
$\mathcal{M}$ denote the observed spatial-sample mask used for autodecoding, i.e.,
$\Omega_{N_x}$ during training and $\Omega_{N_x}^{\mathrm{init}}$ during test-time initial latent fitting. The full decoder reconstruction loss is
\begin{equation}
\label{eq:full_decoder_loss}
L_{\mathrm{dec}}
=
\mathrm{MSE}\!\left(
D_{\Phi_{\theta}(\alpha_{i,l})}
(\mathbf{x}_{i,\mathcal{M}}),
\mathbf{u}_{i,l}[\mathcal{M}]
\right),
\qquad
i\in\mathcal{I},\quad
l=1,\ldots,N_{\mathrm{train}},
\end{equation}
where the MSE is taken over all trajectories, training snapshots,
observed spatial locations, and field components in the minibatch.
Evaluating \cref{eq:full_decoder_loss} over the entire spatial mask at
once requires retaining decoder activations for all queried locations,
which can become the dominant memory cost as the spatial resolution
increases.
To reduce this cost, we partition the observed spatial mask into
disjoint pointwise minibatches,
\begin{equation}
\label{eq:pointwise_partition}
\mathcal{M}
=
\bigcup_{k=1}^{N_c}\mathcal{C}_k,
\qquad
\mathcal{C}_k \cap \mathcal{C}_{k'}=\varnothing
\quad \text{for } k\neq k',
\qquad
|\mathcal{C}_k|\le B_{\mathrm{PM}},
\end{equation}
where $B_{\mathrm{PM}}$ denotes the maximum number of spatial coordinates processed
in one pointwise minibatch.
For each spatial chunk, we define
\begin{equation}
\label{eq:pointwise_decoder_loss}
L_{\mathrm{dec}}^{(k)}
=
\frac{|\mathcal{C}_k|}{|\mathcal{M}|}
\mathrm{MSE}\!\left(
D_{\Phi_{\theta}(\alpha_{i,l})}
(\mathbf{x}_{i,\mathcal{C}_k}),
\mathbf{u}_{i,l}[\mathcal{C}_k]
\right),
\end{equation}
where the MSE is taken over the same trajectory, time, and field-component
indices as in \cref{eq:full_decoder_loss}. Since the spatial chunks form
a disjoint partition of $\mathcal{M}$,
\begin{equation}
\label{eq:pointwise_loss_equivalence}
L_{\mathrm{dec}}
=
\sum_{k=1}^{N_c} L_{\mathrm{dec}}^{(k)}.
\end{equation}
Consequently, the corresponding gradients also satisfy
\begin{equation}
\label{eq:pointwise_gradient_equivalence}
\nabla L_{\mathrm{dec}}
=
\sum_{k=1}^{N_{\mathrm{c}}}
\nabla L_{\mathrm{dec}}^{(k)}.
\end{equation}
Therefore, the spatial chunks can be evaluated and backpropagated
sequentially while accumulating their gradients. The computation graph
associated with each chunk can be released before the next chunk is
processed, so the peak decoder activation memory is governed primarily
by the pointwise minibatch size $B_{\mathrm{PM}}$ rather than by the total number of
observed spatial locations $|\mathcal{M}|$. Pointwise
minibatching does not subsample the observed field or modify the
reconstruction objective; all spatial samples are still used in each
decoder update. The resulting memory reduction comes at the cost of
additional computation, and the result agrees with full-grid evaluation up to floating-point round-off. The same partition applies to test-time initial latent fitting, where only the initial latent state is updated, and to field reconstruction for evaluation, which requires no gradients. Its integration into the complete SMORE training procedure is summarized in
\Cref{alg:smore-pbsw}.

\subsection{Overall Workflow}
\label{sec:overall_workflow}

The SMORE learning procedure consists of two alternating stages, a
training stage using the training set $\mathcal{D}_{tr}$ and an
evaluation stage using the test set $\mathcal{D}_{te}$. The training
stage updates the latent representations, decoder, and latent dynamics,
whereas the evaluation stage keeps the learned networks fixed and is used
to assess the current model and select the checkpoint retained for final
evaluation.
Let $N_{\mathrm{train}}$ denote the number of snapshots from each trajectory
used during the training stage, including the initial snapshot, and let
$N_{\mathrm{roll}}$ denote the total number of snapshots in a complete
trajectory.
Thus, each training trajectory contains $N_{\mathrm{roll}}$ snapshots in
total, but only its first $N_{\mathrm{train}}$ snapshots are used for model
training.
As summarized in \cref{fig:overview}, SMORE consists of an INR-based decoder and a structured latent-dynamics model. \cref{fig:subfig1_1} illustrates the overall training and inference workflow, \cref{fig:subfig1_2} summarizes the two latent-dynamics formulations, and \cref{fig:subfig1_3} highlights the Lyapunov-guided stability regularization and coordinate-based decoding. The overall learning procedure is summarized as follows:

\begin{enumerate}

\item
Before training begins, each trajectory in $\mathcal{D}_{tr}$ is assigned a
sequence of trainable latent states corresponding to the
$N_{\mathrm{train}}$ snapshots used in the training window.
These latent states are stored in the latent-state table
$\mathcal{Z}$.
Snapshots outside this training window are not assigned trainable latent
states and do not participate in the training objective.
Consequently, although each complete trajectory contains
$N_{\mathrm{roll}}$ snapshots, only the first $N_{\mathrm{train}}$
snapshots contribute to the latent-state table and the subsequent parameter
updates.

\item
During training, the latent-state table ($\mathcal{Z}$), the decoder parameters ($\theta$), and the latent-dynamics parameters ($\psi$) are all updated.
For each minibatch of training trajectories, the corresponding latent states
are retrieved from $\mathcal{Z}$.
The latent states and the decoder are jointly optimized through the
reconstruction loss so that
$D_{\Phi_{\theta}(\alpha)}$
matches the observed physical fields at the sampled spatial locations.

After the latent-state update, the resulting latent states are retrieved again
from $\mathcal{Z}$ and detached from the reconstruction computation graph.
They are then treated as the target latent trajectory for learning the
dynamics.
For each training trajectory, the latent dynamics starts from the latent state
associated with the initial snapshot and rolls forward through the remaining
$N_{\mathrm{train}}-1$ training steps.
Scheduled sampling is used during the early stage of training, allowing
selected target latent states to serve as intermediate rollout anchors.
As training progresses, the teacher-forcing probability decreases and the
dynamics increasingly relies on its own predicted latent states.

The predicted latent trajectory is compared with the detached target latent
trajectory through the latent-dynamics loss, which provides the learning
signal for updating $f_{\psi}$.
The dynamics objective is additionally augmented by the stability
regularization terms defined in \cref{eq:reg_1,eq:reg_2}, with the applicable
terms determined by the selected latent-dynamics model.
The decoder and latent-dynamics gradients are accumulated over the training
minibatches and their parameters are updated according to the optimization
schedule summarized in \Cref{alg:smore-tr}.

\item
The evaluation stage does not update the latent-state table, decoder, or
latent-dynamics parameters.
Instead, the current decoder and dynamics networks are fixed and the initial
latent state of each trajectory in $\mathcal{D}_{te}$ is inferred from its
initial snapshot through autodecoding.
Only this initial latent state is optimized during the inference procedure;
the learned network parameters remain unchanged.

Starting from the inferred initial latent state, the latent dynamics is rolled
out for $N_{\mathrm{roll}}-1$ prediction steps.
The decoder then reconstructs the corresponding physical fields over the
complete rollout horizon.
The prediction error is computed over the predicted snapshots, excluding the
initial snapshot.
This evaluation provides the criterion used for checkpoint selection.
Whenever the finite overall test MSE improves over the previously recorded
value, the current decoder and latent-dynamics parameters are stored as the
checkpoint with the lowest test error.

\item
The training and evaluation stages are repeated throughout the prescribed
learning process.
After training is completed, the checkpoint with the lowest test error is loaded and both
networks are frozen for the final evaluation.
For each trajectory in $\mathcal{D}_{te}$, the initial latent state is
re-inferred from the initial-snapshot observations using the prescribed final
fitting budget.
The resulting latent state is then propagated for
$N_{\mathrm{roll}}-1$ prediction steps, and the decoder reconstructs the
corresponding physical fields.
The final reported MSEs are computed over these predicted snapshots, with the
initial snapshot excluded from the prediction metric.

\end{enumerate}

The complete training and evaluation procedure is summarized in
\Cref{alg:smore-tr}. 

\begin{breakablealgorithm}
\caption{SMORE training workflow with final test evaluation}
\label{alg:smore-tr}
\begin{algorithmic}

\Inputs{
- $\mathcal{D}_{tr}$:
training trajectories restricted to their first $N_{\mathrm{train}}$ snapshots
\Statex - $\mathcal{D}_{te}$:
complete test trajectories used for checkpoint selection and final evaluation
\Statex - $N_{\mathrm{train}}$:
number of snapshots from each training trajectory used in the training stage,
including the initial snapshot
\Statex - $N_{\mathrm{roll}}$:
total number of snapshots in each complete trajectory,
including the initial snapshot
\Statex - $\mathcal{T}_{N_{\mathrm{train}}}
=\{t_1,\ldots,t_{N_{\mathrm{train}}}\}$:
time stamps corresponding to the training window
\Statex - $\mathcal{M}$:
observed spatial-sample mask used for autodecoding
\Statex - $D_{\Phi_{\theta}}$:
INR decoder with FiLM modulation
\Statex - $f_{\psi}$:
latent dynamics model

\Statex - $N_e$:
number of training epochs
\Statex - $N_b$:
number of minibatches per epoch
\Statex - $N_{\mathrm{traj}}$: number of training trajectories
\Statex - $\mathcal{Z}\in
\mathbb{R}^{N_{\mathrm{traj}}\times N_{\mathrm{train}}\times r}$:
trainable latent-state table indexed by training trajectory and time
\Statex - $\beta_{\mathrm{eig}},\beta_{\mathrm{cons}}$:
stability-regularization weights
\Statex - $\varepsilon_0,\rho_{\mathrm{TF}}$:
initial teacher-forcing probability and decay factor
\Statex - $K_{\mathrm{eval}}$:
initial latent fitting budget for training-time evaluation,
specified in \Cref{sec:initial_latent_evaluation_protocol}
\Statex - $K_{\mathrm{final}}$:
initial latent fitting budget for final test evaluation
}

\Initialize{
- Initialize decoder parameters $\theta$, dynamics parameters $\psi$,
and latent-state table $\mathcal{Z}$
\Statex - Initialize three optimizers
$\mathcal{O}_{\mathcal{Z}}$, $\mathcal{O}_{\theta}$, and
$\mathcal{O}_{\psi}$ for $\mathcal{Z}$, $\theta$, and $\psi$, respectively
\Statex - $\varepsilon_t \leftarrow \varepsilon_0$;
best test MSE $\leftarrow +\infty$
\Statex - $\mathcal{O}_{\mathcal{Z}}$ steps after each minibatch;
$\mathcal{O}_{\theta}$ and $\mathcal{O}_{\psi}$ step once per epoch
after gradient accumulation
}

\For{$\text{epoch} \leftarrow 1$ to $N_e$}

\State
$\mathcal{O}_{\theta}.\texttt{zero\_grad}()$;
$\mathcal{O}_{\psi}.\texttt{zero\_grad}()$

\For{batch $\mathcal{D}_{tr,\mathcal{I}}$ in $\mathcal{D}_{tr}$}

\State
$\{\mathbf{u}_{i,l},\mathbf{x}_i,
\mathcal{T}_{N_{\mathrm{train}}}\}_{i\in\mathcal{I},\,l=1}^{N_{\mathrm{train}}}
\leftarrow \mathcal{D}_{tr,\mathcal{I}}$

\State
$\{\alpha_{i,l}\}_{i\in\mathcal{I},\,l=1}^{N_{\mathrm{train}}}
\leftarrow \texttt{gather}(\mathcal{Z},\mathcal{I})$
\Comment{retrieve latent states for this batch}

\State
$L_{\mathrm{dec}} \leftarrow
\mathrm{MSE}\!\left(
D_{\Phi_{\theta}(\alpha_{i,l})}(\mathbf{x}_{i,\mathcal{M}}),
\mathbf{u}_{i,l}[\mathcal{M}]
\right)$
over $i\in\mathcal{I}$ and $l=1,\ldots,N_{\mathrm{train}}$

\State
$\mathcal{O}_{\mathcal{Z}}.\texttt{zero\_grad}()$;
backpropagate $L_{\mathrm{dec}}$

\State
$\mathcal{O}_{\mathcal{Z}}.\texttt{step}()$
\Comment{update selected latent states; retain accumulated decoder gradients}

\State
$\{\alpha_{i,l}\}_{i\in\mathcal{I},\,l=1}^{N_{\mathrm{train}}}
\leftarrow \texttt{gather}(\mathcal{Z},\mathcal{I})$
\Comment{re-read post-update latent states}

\State
$\alpha^{*}_{i,l}\leftarrow\texttt{detach}(\alpha_{i,l})$
for $i\in\mathcal{I}$ and $l=1,\ldots,N_{\mathrm{train}}$
\Comment{construct target latent trajectories}

\State
$\{\hat{\alpha}_{i,l}\}_{i\in\mathcal{I},\,l=1}^{N_{\mathrm{train}}}
\leftarrow
\texttt{scheduled\_sampling\_ode\_solve}
\left(
f_\psi,
\{\alpha^{*}_{i,l}\}_{l=1}^{N_{\mathrm{train}}},
\mathcal{T}_{N_{\mathrm{train}}},
\varepsilon_t
\right)$
\Comment{roll out latent dynamics with scheduled sampling}

\State
$L_{\alpha} \leftarrow
\mathrm{MSE}\!\left(
\{\hat{\alpha}_{i,l}\}_{i\in\mathcal{I},\,l=1}^{N_{\mathrm{train}}},
\{\alpha^*_{i,l}\}_{i\in\mathcal{I},\,l=1}^{N_{\mathrm{train}}}
\right)$

\If{the latent dynamics is linear-quadratic}

\State
$L_{\mathrm{dyn}} \leftarrow
L_{\alpha}+\big(\beta_{\mathrm{eig}} L_{\mathrm{reg,eig}}
+ \beta_{\mathrm{cons}} L_{\mathrm{reg,cons}}\big)/N_b$
\Comment{
$L_{\mathrm{reg,eig}}$ and $L_{\mathrm{reg,cons}}$
are defined in \cref{eq:reg_1,eq:reg_2}
}

\ElsIf{the latent dynamics is linear}

\State
$L_{\mathrm{dyn}} \leftarrow
L_{\alpha}+\beta_{\mathrm{eig}} L_{\mathrm{reg,eig}}/N_b$
\Comment{
$L_{\mathrm{reg,eig}}$ is defined in \cref{eq:reg_1}
}

\Else

\State
$L_{\mathrm{dyn}} \leftarrow L_{\alpha}$

\EndIf

\State
Backpropagate $L_{\mathrm{dyn}}$
\Comment{
accumulate dynamics gradients; regularizers that depend only on the dynamics parameters are divided by $N_b$ so that they enter once per epoch
}

\EndFor

\State
Apply gradient clipping to decoder and dynamics parameters

\State
$\mathcal{O}_{\theta}.\texttt{step}()$;
$\mathcal{O}_{\psi}.\texttt{step}()$

\State
$\varepsilon_t \leftarrow \rho_{\mathrm{TF}}\,\varepsilon_t$

\If{test evaluation is scheduled}

\State
Fit the initial latent state of each test trajectory using
$K_{\mathrm{eval}}$ steps with fixed networks

\State
Roll out for $N_{\mathrm{roll}}-1$ prediction steps and compute the
overall test MSE over the predicted snapshots

\If{the test MSE is finite and lower than the best recorded value}

\State
Save the decoder and dynamics parameters as the checkpoint with the lowest test error;
update the best test MSE

\EndIf
\EndIf

\EndFor

\State
\textbf{Final test stage:}
load the checkpoint with the lowest test error and freeze both networks

\For{each test trajectory in $\mathcal{D}_{te}$}

\State
Initialize its latent state using the training-time evaluation initialization rule

\State
Optimize only the initial latent state for
$K_{\mathrm{final}}$ steps using initial-snapshot observations

\State
Roll out the latent dynamics for
$N_{\mathrm{roll}}-1$ prediction steps and decode the predicted fields

\EndFor

\State
Compute and report final test MSEs over the
$N_{\mathrm{roll}}-1$ predicted snapshots,
excluding the initial snapshot

\end{algorithmic}
\end{breakablealgorithm}

\FloatBarrier

\medskip
\noindent
The pointwise decoder minibatching procedure introduced in
\Cref{sec:pointwise_minibatching} is summarized in
\Cref{alg:smore-pbsw}. 
\par\medskip
\begin{breakablealgorithm}
\caption{Memory-saving pointwise decoder minibatch update}
\label{alg:smore-pbsw}
\begin{algorithmic}

\Inputs{
- Latent states $\{\alpha_{i,l}\}$, targets $\{\mathbf{u}_{i,l}\}$, and observed spatial-sample mask $\mathcal{M}$
\Statex - Optimizer $\mathcal{O}$ for the latent states being fitted
\Statex - $B_{\mathrm{PM}}$: maximum number of spatial samples in each pointwise minibatch
}

\State Partition $\mathcal{M}$ into disjoint spatial chunks
$\{\mathcal{C}_k\}_{k=1}^{N_{\mathrm{c}}}$ such that
$\bigcup_{k=1}^{N_{\mathrm{c}}}\mathcal{C}_k=\mathcal{M}$ and
$|\mathcal{C}_k|\le B_{\mathrm{PM}}$

\State $\mathcal{O}.\texttt{zero\_grad}()$

\For{$k=1$ to $N_{\mathrm{c}}$}

\State
$L_{\mathrm{dec}}^{(k)}
\leftarrow
\frac{|\mathcal{C}_k|}{|\mathcal{M}|}
\mathrm{MSE}\!\left(
D_{\Phi_{\theta}(\alpha_{i,l})}
(\mathbf{x}_{i,\mathcal{C}_k}),
\mathbf{u}_{i,l}[\mathcal{C}_k]
\right)$

\State Backpropagate $L_{\mathrm{dec}}^{(k)}$
\Comment{accumulate gradients; release the graph of chunk $k$}

\EndFor

\State $\mathcal{O}.\texttt{step}()$

\end{algorithmic}
\end{breakablealgorithm}

\medskip
\noindent
\Cref{alg:smore-pbsw} is used in two places. During training, it replaces the
decoder update in \Cref{alg:smore-tr} with $\mathcal{O}=\mathcal{O}_{\mathcal{Z}}$;
the accumulated decoder gradients are retained for the per-epoch update of
$\theta$. During test-time initial latent fitting, it is applied at each of the
$K_{\mathrm{eval}}$ or $K_{\mathrm{final}}$ fitting steps with the networks frozen,
so that only the initial latent state is updated. Field reconstruction for
evaluation uses the same spatial chunks in a forward pass without gradients. The experiments in \Cref{sec:results_long_horizon_extrapolation,sec:results_stability_robustness,sec:results_sparse_mask_ablation}
use full-grid decoder updates, since their spatial resolutions fit within the
available GPU memory and pointwise minibatching would only add computation time.
Pointwise minibatching becomes useful when the spatial resolution is large
enough that full-grid decoder evaluation becomes limited by GPU memory, as examined in
\Cref{sec:results_memory_efficiency}.

\section{Numerical Experiments}
\label{sec:numerical_experiment_setup}
In this section, we conduct several numerical experiments to evaluate the performance of SMORE. We evaluate rollout accuracy, training robustness, and prediction
from sparse initial observations on the Wave, Navier--Stokes,
and Spherical Shallow Water benchmarks. We also compare the
maximum spatial resolutions supported by different model
architectures on a single GPU and assess how
pointwise decoder minibatching affects memory use, training
time, and prediction accuracy at a fixed resolution.

\subsection{Experiment Setup}
\label{sec:experiment_setup}
This section describes the benchmark datasets, compared models, rollout
protocols, training and evaluation windows, training objectives, evaluation
metrics, checkpoint-selection and final-evaluation procedures, and experimental environment.

\subsubsection{Benchmark Datasets}\label{sec:dataset}

Following the setup of \cite{Yin2022Continuous}, we consider three time-dependent PDE benchmarks, including 2D Wave, 2D Navier--Stokes, and 3D Spherical Shallow Water. The data are generated by numerically solving the corresponding governing
equations under the dataset-specific sampling protocols described in
Appendix~\ref{apdx:dataset_details}. Each training or test sequence contains $N_{\mathrm{roll}}=20$
temporal states,
$\{\mathbf{u}_0,\mathbf{u}_1,\ldots,\mathbf{u}_{N_{\mathrm{roll}}-1}\}$.
The spatial resolutions, temporal step sizes, prediction targets, and
training/test splits are summarized in \Cref{tab:dataset_summary}.

\begin{table}[!htbp]
\centering
\caption{Summary of the PDE benchmark datasets. Grid denotes the spatial
sampling resolution, and sequence length denotes the number of snapshots in each
trajectory.}
\label{tab:dataset_summary}
\setlength{\tabcolsep}{4pt}
\resizebox{\textwidth}{!}{%
\begin{tabular}{lcccccc}
\toprule
Dataset & Prediction target & Grid & $\Delta t$ & Seq. length & Train seq. & Test seq. \\
\midrule
2D Wave & two-component wave field & $64^2$ & $0.25$ s & 20 & 512 & 32 \\
2D Navier--Stokes & vorticity & $64^2$ & $1$ s & 20 & 512 & 32 \\
3D Spherical Shallow Water & height, vorticity & $64\times128$ & $1$ h & 20 & 64 & 16 \\
\bottomrule
\end{tabular}%
}
\end{table}

\subsubsection{Compared Models}\label{sec:training_task_models}

We compare the proposed SMORE family with
FNO~\cite{li2020fourier,bonev2023spherical},
CNO~\cite{raonic2023convolutional},
Transolver~\cite{wu2024Transolver}, and
DINo~\cite{Yin2022Continuous}. DINo uses the same modulated-MFN INR decoder
structure as SMORE, but adopts a multilayer perceptron (MLP) as its latent
dynamics model.

\subsubsection{Rollout Protocol}
All model comparisons are formulated as initial-condition rollout prediction
tasks. For FNO, CNO, and Transolver, we use a physical-space autoregressive
rollout. Starting from $\mathbf{u}_0$, the model predicts
$\hat{\mathbf{u}}_1$, which is then fed back as input to predict
$\hat{\mathbf{u}}_2$; the procedure is repeated until the desired rollout
horizon is reached. For DINo and the two SMORE models, the initial physical state
$\mathbf{u}_0$ is first represented by an inferred latent code $\alpha_0$
through INR-based autodecoding.
The latent dynamics model then rolls out
$\hat{\alpha}_1,\hat{\alpha}_2,\ldots$ from $\alpha_0$, and the INR decoder
maps these predicted latent states to the corresponding physical fields
$\hat{\mathbf{u}}_1,\hat{\mathbf{u}}_2,\ldots$.

\subsubsection{Training and Evaluation Protocol}
As described in \Cref{sec:method}, the learning procedure consists of two
stages. The first is model training on trajectories restricted to their first
$N_{\mathrm{train}}$ snapshots, followed by evaluation on complete test
trajectories containing $N_{\mathrm{roll}}$ snapshots.
In our experiments, $N_{\mathrm{train}}=10$ and $N_{\mathrm{roll}}=20$; the
model is supervised over the first $N_{\mathrm{train}}-1=9$ prediction steps
and evaluated over the full $N_{\mathrm{roll}}-1=19$-step rollout.
Throughout, prediction step $n$ predicts $\mathbf{u}_n$, i.e., the
$(n{+}1)$-th snapshot of the trajectory.

Accordingly, we define the in-horizon, extrapolation, and full prediction
windows as
\[
\mathcal{T}_{\mathrm{in}}
=
\{1,\ldots,N_{\mathrm{train}}-1\},
\]
\[
\mathcal{T}_{\mathrm{extra}}
=
\{N_{\mathrm{train}},\ldots,N_{\mathrm{roll}}-1\},
\]
and
\[
\mathcal{T}_{\mathrm{full}}
=
\{1,\ldots,N_{\mathrm{roll}}-1\}.
\]

The in-horizon MSE is computed over snapshots 2--10,
corresponding to $N_{\mathrm{train}}-1=9$ predictions.
The extrapolation MSE is computed over snapshots 11--20,
corresponding to
$N_{\mathrm{roll}}-N_{\mathrm{train}}=10$ predictions.
The overall MSE is computed over all
$N_{\mathrm{roll}}-1=19$ predicted snapshots.
The initial snapshot is excluded from all prediction MSEs.

For a prediction window $\mathcal{T}$ and dataset $\mathcal{D}$, we define
the rollout MSE as
\begin{equation}
    \mathcal{E}_{\mathcal{T}}(\mathcal{D})
    =
    \frac{1}{|\mathcal{T}|\,|\mathcal{D}|}
    \sum_{i=1}^{|\mathcal{D}|}
    \sum_{n\in\mathcal{T}}
    \frac{1}{N_x c}
    \left\|
    \hat{\mathbf{u}}_{i,n}
    -
    \mathbf{u}_{i,n}
    \right\|_2^2 ,
    \label{eq:rollout_mse}
\end{equation}
where $N_x$ is the number of spatial sample points and $c$ is the number of
output channels.
This quantity is the mean squared prediction error per field entry.
We report it over
$\mathcal{T}_{\mathrm{in}}$,
$\mathcal{T}_{\mathrm{extra}}$,
and $\mathcal{T}_{\mathrm{full}}$. Regarding training epochs, unless otherwise specified, the budget for all models is fixed at 5,000 epochs. All models are trained on the fully observed data in~\Cref{eq:training_dataset}.

\subsubsection{Training Objectives}
FNO, CNO, and Transolver follow the physical-space autoregressive rollout
described above.
During training, a prediction loss is computed at each of the
$N_{\mathrm{train}}-1$ rollout steps, and the losses over all prediction
steps are summed and backpropagated to update the model parameters.

DINo and SMORE instead separate representation learning from latent-dynamics
learning.
The latent states and INR decoder are jointly optimized through the
reconstruction loss defined in \cref{eq:inr-loss-1}, which measures the
discrepancy between the decoded fields and the corresponding ground-truth
fields.
After this update, the resulting latent states are detached and used as the
target latent trajectory.
The latent-dynamics model is then trained separately by minimizing the
discrepancy between its rolled-out latent trajectory and these target latent
states.

SMORE further augments the latent-dynamics objective with the
Lyapunov-guided stability regularization introduced in
\cref{eq:reg_1,eq:reg_2}.
For SMORE-LRLQ, both the spectral and conservation regularization terms are applied,
whereas the conservation term is omitted for the SMORE-Koopman. Gradient clipping is applied to the decoder and dynamics parameters before each
parameter update, with the global $\ell_2$ norm of the gradient clipped to $1.0$.

\subsubsection{Checkpoint Selection and Final Evaluation}
\label{sec:initial_latent_evaluation_protocol}

For DINo and both SMORE models, evaluation requires fitting the initial latent
state of each test trajectory with the networks frozen.
During training, we periodically fit this initial latent with a short
budget of $K_{\mathrm{eval}}=300$ on Navier--Stokes and $500$ on Wave and
Spherical Shallow Water, and select the checkpoint with the lowest finite
overall test MSE, following the workflow in \Cref{sec:overall_workflow}.
At the end of training, we load the selected checkpoint, freeze the decoder and dynamics
networks, and refit each test trajectory's initial latent with
$K_{\mathrm{final}}=1000$ steps to produce the final reported errors.
Only the observed initial snapshot enters this fitting objective; the
initialization rules and fitting learning rates are unchanged, and the
checkpoint is not reselected using the $K_{\mathrm{final}}$ results.
For further details of the experimental setup, see Appendix~\ref{app:Hyper_parameters} and \ref{app:Ex_device}.

\subsection{Long-Horizon Rollout Extrapolation}
\label{sec:results_long_horizon_extrapolation}

In this section, we evaluate SMORE's ability to accurately predict the
evolution of physical fields beyond the supervised training horizon.
Data-driven models for time-dependent PDEs are typically trained over a
finite temporal window. A key question is therefore whether the learned
dynamics can continue to predict the evolution of the physical field beyond
this window while maintaining low prediction error.
This question is particularly relevant to SMORE, where the generic black-box
MLP latent dynamics used in DINo are replaced by explicit linear or
linear--quadratic dynamics equipped with Lyapunov-guided stability
regularization.
To assess this capability, we evaluate SMORE over both the supervised
in-horizon window and the subsequent extrapolation window and compare its
performance against several representative baseline models. DINo provides the
closest latent-model baseline to SMORE, since both use the same general
INR-based autodecoding framework, differing only in the latent dynamics, where
DINo uses an MLP while the two SMORE variants use structured linear
(SMORE-Koopman) or linear--quadratic (SMORE-LRLQ) dynamics.

Following the common experimental protocol described in \Cref{sec:experiment_setup}, we evaluate FNO, CNO, Transolver, DINo, SMORE-Koopman, and SMORE-LRLQ on the 2D Wave, 2D Navier--Stokes, and 3D Spherical Shallow Water benchmarks.
All models in this comparison are trained and evaluated using three random seeds, selected once and then held fixed.
To ensure a fair comparison across models, the model sizes are selected
such that the overall parameter budgets are approximately matched across methods
for each dataset; the corresponding hyperparameter configurations and budgets are reported in Appendix \Cref{tab:app_hparam_candidates} and  \Cref{tab:app_rollout_decomposition} respectively.
 We report the rollout MSE separately over the supervised in-horizon window and
the subsequent extrapolation window as shown in \Cref{tab:combined_loss_comparison}.
\begin{table}[htbp]
  \centering
  \caption{In-horizon and extrapolation MSE for long-horizon rollout
prediction under approximately matched parameter budgets.
  Values are averaged over three random seeds, with the corresponding
standard deviations reported in Appendix \Cref{tab:app_rollout_decomposition}.
  For each run, the checkpoint with the lowest finite overall test MSE over the 19 predicted snapshots is selected; the initial snapshot is excluded.
  Best values in each metric column are bolded, and second-best values are underlined.}
  \resizebox{\textwidth}{!}{%
  \begin{tabular}{lcccccc}
    \toprule
    Model
    & \makecell{Wave\\In-horizon}
    & \makecell{Wave\\Extrapolation}
    & \makecell{Navier--Stokes\\In-horizon}
    & \makecell{Navier--Stokes\\Extrapolation}
    & \makecell{Shallow Water\\In-horizon}
    & \makecell{Shallow Water\\Extrapolation} \\
    \midrule
    FNO~\cite{li2020fourier,bonev2023spherical}
      & \underline{2.89e-04} & 7.37e-02 & 1.61e-03 & 2.43e-03 & 1.79e-03 & 5.22e-03 \\
    CNO~\cite{raonic2023convolutional}
      & 1.99e-02 & 1.53e-01 & 1.52e-03 & 3.87e-03 & 8.03e-04 & 2.86e-03 \\
    Transolver~\cite{wu2024Transolver}
      & 7.82e-04 & 1.67e-02 & 4.13e-01 & 6.22e-01 & 1.71e-03 & 6.05e-03 \\
    DINo~\cite{Yin2022Continuous}
      & 1.13e-03 & 1.38e-02 & \underline{1.34e-03} & \underline{2.32e-03} & \underline{4.90e-04} & \underline{6.49e-04} \\
    \textbf{SMORE-Koopman}
      & \textbf{6.74e-05} & \textbf{3.22e-04} & 5.37e-02 & 9.78e-02 & \textbf{2.22e-04} & \textbf{5.34e-04} \\
    \textbf{SMORE-LRLQ}
      & 1.25e-03 & \underline{8.65e-04} & \textbf{1.01e-03} & \textbf{2.00e-03} & 5.86e-04 & 7.67e-04 \\
    \bottomrule
  \end{tabular}%
  }
  \label{tab:combined_loss_comparison}
\end{table}

Across all three benchmarks, one of the two SMORE variants attains the
lowest MSE in every one of the six evaluation columns. SMORE-Koopman is
best on Wave and Spherical Shallow Water in both windows, and SMORE-LRLQ
is best on Navier--Stokes in both windows. The latent-model baseline DINo is at most
second-best. Measured against DINo in the extrapolation window, the best
SMORE variant reduces the error by $97.7\%$ on Wave (from
$1.38\times10^{-2}$ to $3.22\times10^{-4}$), by $13.8\%$ on
Navier--Stokes (from $2.32\times10^{-3}$ to $2.00\times10^{-3}$), and by
$17.7\%$ on Spherical Shallow Water (from $6.49\times10^{-4}$ to
$5.34\times10^{-4}$).

The advantage of SMORE is most pronounced in the extrapolation window,
where error accumulates without supervision. Several baselines that are
competitive in-horizon degrade sharply once they leave the supervised
window. On Wave, FNO attains the second-lowest in-horizon error
($2.89\times10^{-4}$), yet its extrapolation
error rises to $7.37\times10^{-2}$, a degradation of roughly $255\times$;
over the same windows SMORE-Koopman rises only from $6.74\times10^{-5}$
to $3.22\times10^{-4}$, a factor of about $4.8$, and its extrapolation
error is about $229\times$ lower than that of FNO. On Spherical Shallow
Water, CNO degrades from $8.03\times10^{-4}$ in-horizon to
$2.86\times10^{-3}$ in extrapolation (about $3.6\times$), whereas
SMORE-Koopman degrades only from $2.22\times10^{-4}$ to
$5.34\times10^{-4}$ (about $2.4\times$) and remains about $5.4\times$ more
accurate in the extrapolation window. Strong in-horizon accuracy therefore does not necessarily imply reliable extrapolation, whereas the structured, stability-regularized latent dynamics of SMORE show slower error accumulation over the unsupervised horizon than the baseline models.

The two SMORE variants exhibit a clear trade-off between peak performance
and cross-system robustness. SMORE-Koopman has the simplest structure, a
single linear operator, which makes it easy to learn and yields the lowest
errors when the dominant dynamics admit a linear latent evolution. This is the case for Wave, whose governing equation is itself linear. It
also holds for Spherical Shallow Water. Although this system is physically
nonlinear, with both advective and height--velocity coupling terms, it
becomes a more regular regime once the transient spin-up phase is removed
(\Cref{app:3dswdataset}).
The same simplicity, however, limits SMORE-Koopman when the dynamics are
strongly nonlinear. On Navier--Stokes, its errors
($5.37\times10^{-2}$ in-horizon and $9.78\times10^{-2}$ in extrapolation) are
about $50\times$ those of SMORE-LRLQ. SMORE-LRLQ adds quadratic latent terms
that are structurally compatible with the convective nonlinearity of
Navier--Stokes, on which it is the best method in both windows
($1.01\times10^{-3}$ and $2.00\times10^{-3}$). The additional quadratic
coupling makes it harder to learn and yields no benefit on the nearly linear
systems, where it is never the single best method, but it does not exhibit
extreme error degradation in any of the six evaluation settings. The
effective choice of latent dynamics is thus governed by the structure of the
physics rather than by expressive capacity alone. Adding quadratic terms
helps when the dynamics are convective and can hurt when an essentially
linear evolution is already well represented.

One result in \Cref{tab:combined_loss_comparison} deserves some additional explanation. For Wave, the extrapolation error of SMORE-LRLQ ($8.65\times10^{-4}$) is lower than its in-horizon error ($1.25\times10^{-3}$), which may appear inconsistent with the usual tendency of prediction error to increase with horizon. This behavior results from window averaging in an oscillatory system rather than from an actual improvement at longer horizons. As shown by the per-snapshot error profile in \Cref{fig:wave_per_frame}, the Wave error is non-monotonic and varies with the oscillatory dynamics. It reaches a local maximum within the in-horizon window, around snapshots 7--9, and then decreases through much of the extrapolation window. Consequently, averaging over the two windows places the error peak in the in-horizon interval and much of the subsequent decline in the extrapolation interval, causing the extrapolation-window mean to fall below the in-horizon mean. This effect is specific to Wave; for Navier--Stokes and Shallow Water, the extrapolation error remains higher than the in-horizon error for all methods.

\Cref{fig:extrapolation_field_comparison} shows example predicted fields in the
extrapolation window. The examples show that the low numerical errors are accompanied by accurate recovery of the underlying spatial structures. SMORE preserves these structures as they evolve throughout the rollout, suggesting that the model captures the field dynamics rather than producing overly smoothed spatial predictions.

\begin{figure}[!htbp]
  \centering
  \includegraphics[height=0.8\textheight,keepaspectratio]{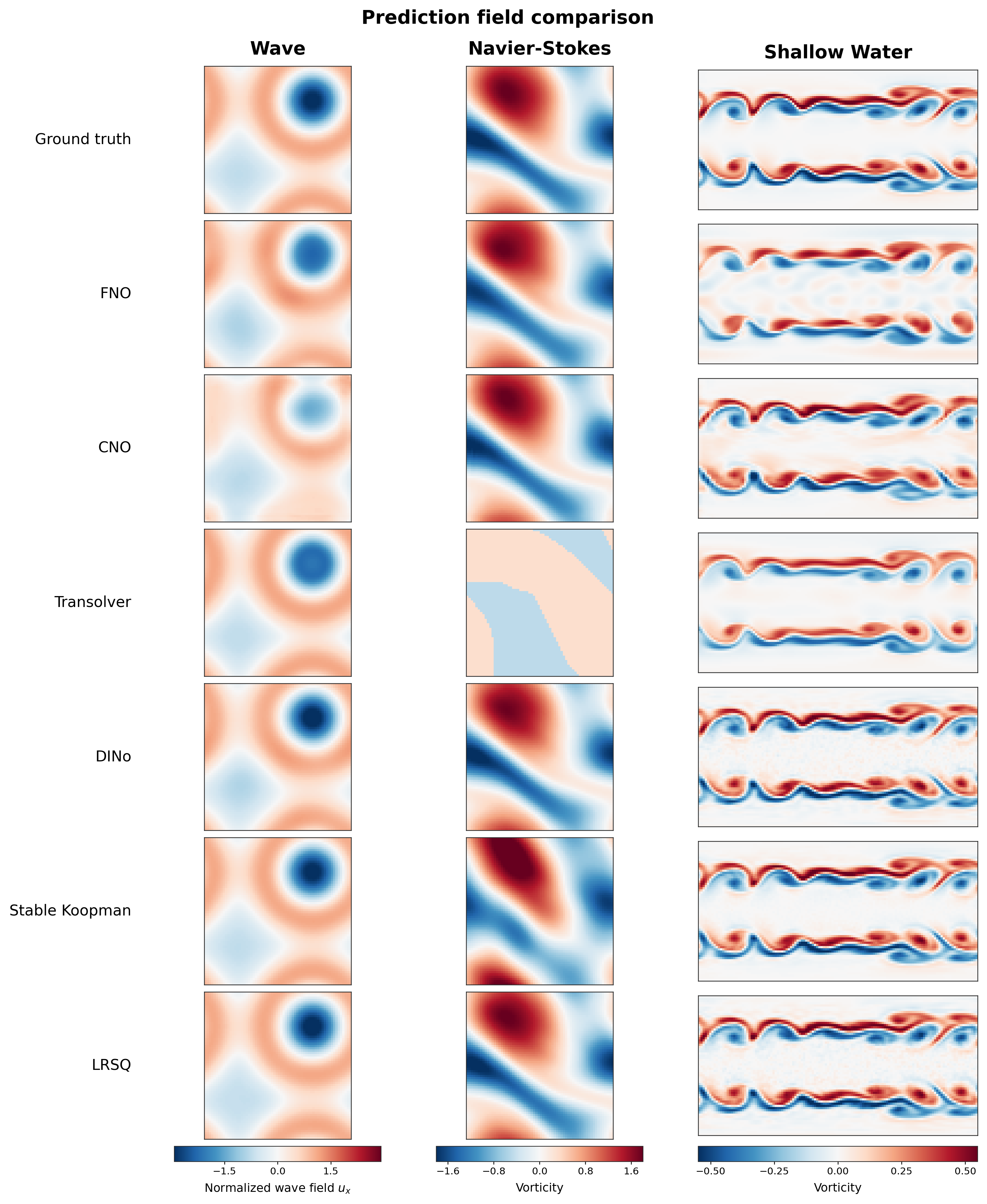}
    \caption{Field-level extrapolation comparison at snapshot 15
(prediction step $n=14$).
Within each benchmark, all models are compared using the same random seed,
the same test trajectory, and the same prediction step $n$.
Rows show the ground truth followed by model predictions; columns show the
normalized Wave field $u_x$, Navier--Stokes vorticity, and Spherical Shallow
Water vorticity.
These snapshots confirm that the low MSE in \Cref{tab:combined_loss_comparison} reflects actual structure recovery. The predicted fields track the real spatial patterns and their evolution rather than only a low average error.}
  \label{fig:extrapolation_field_comparison}
\end{figure}

SMORE achieves the lowest error across all six evaluation columns, with the advantage being most pronounced in the extrapolation window. This suggests that its structured, stability-regularized latent dynamics
accumulate error more slowly over the rollout than the baselines. The two variants, however, show complementary strengths. SMORE-Koopman performs particularly well when a linear latent evolution is sufficient to capture the dominant dynamics, as in Wave and Shallow Water, but degrades substantially on Navier--Stokes. In contrast, SMORE-LRLQ provides more consistent performance across systems and achieves the lowest errors on the convection-dominated Navier--Stokes case. The two variants therefore show complementary strengths across the three systems.

\subsection{Stability and Robustness Across Latent Dimensions}
\label{sec:results_stability_robustness}

In this section, we investigate how the stability-related mechanisms in SMORE
affect the robustness of its predictions.
As discussed at the end of \Cref{sec:dynamic-model}, SMORE-LRLQ incorporates
three mechanisms that can contribute to robust latent evolution, which are 
Lyapunov-guided stability regularization, low-rank quadratic parameterization,
and gradient clipping. In our experiments, we observe that enabling or disabling these mechanisms
greatly affects both the prediction error and the training stability. 
We therefore ask what each of these three mechanisms contributes to the training stability and predictive accuracy of SMORE-LRLQ, how these contributions change with the per-component latent dimension 
$p$, and whether they are consistent across the three PDE systems.
To answer these questions, we evaluate SMORE-LRLQ across different per-component
latent dimensions $p$ while selectively enabling or removing the three
mechanisms.
In the discussion that follows, the FULL configuration includes all three mechanisms.
For compactness in the ablation labels, $L_{\mathrm{reg}}$ denotes the pair of
Lyapunov-guided stability regularizers,
$L_{\mathrm{reg,cons}}$ and $L_{\mathrm{reg,eig}}$.
The reported prediction error follows evaluation protocol
described in \Cref{sec:experiment_setup}, with final test errors computed using $K_{\mathrm{final}}=1000$ initial latent fitting steps.
The number of epochs is set to 10{,}000 to facilitate the observation of
long-term trends. For configurations that develop NaN, we report the
first logged NaN epoch in \cref{fig:stability_robustness}(a). The first NaN epochs for all tested
configurations on Shallow Water are listed in Appendix~\Cref{tab:app_sw_nan_epochs}, and the complete Navier--Stokes
mechanism sweep in Appendix~\Cref{tab:app_ns_mechanism_sweep}.
\begin{figure}[!htbp]
  \centering
  \includegraphics[width=\textwidth]{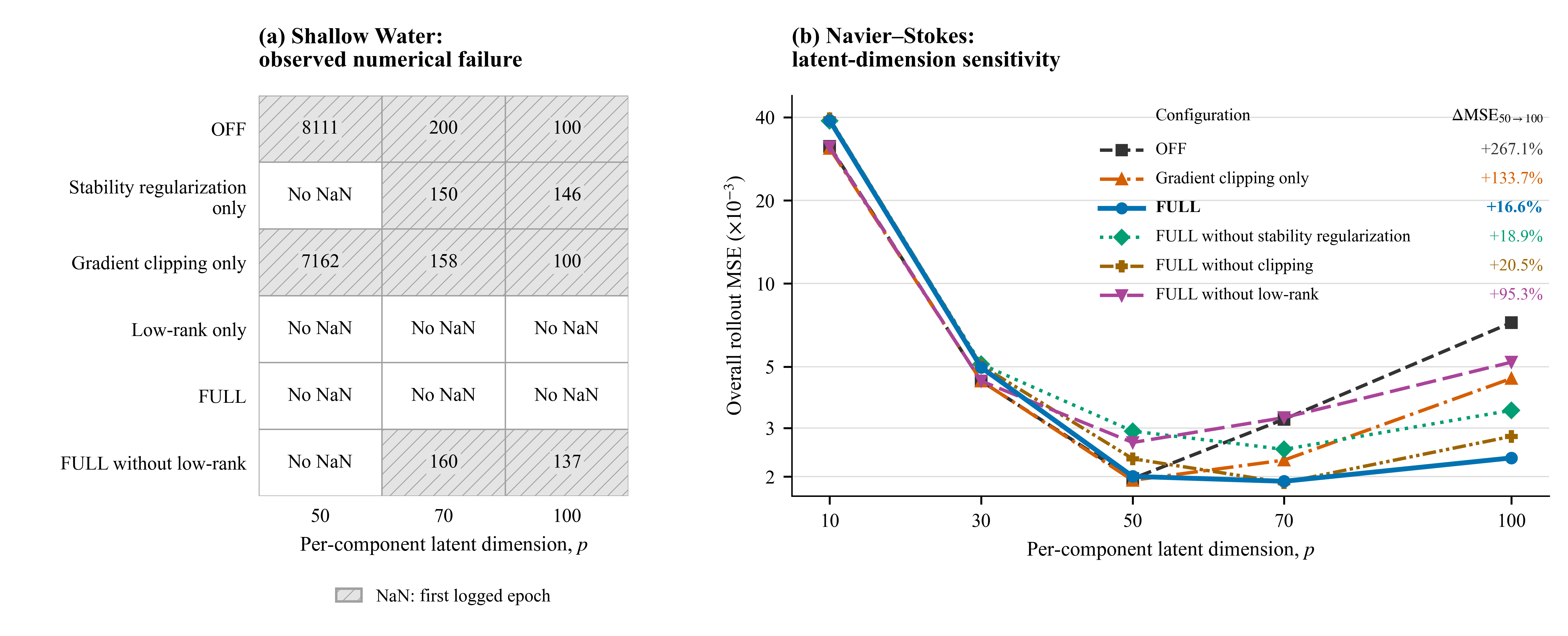}
  \caption{Training robustness of SMORE-LRLQ across per-component latent dimension
  $p$.
  (a) Spherical Shallow Water: rows denote stability configurations and columns denote per-component latent dimensions. Hatched cells report the first logged epoch with NaN; ``Not observed'' indicates that none was observed during the available training window.
  (b) Navier--Stokes: best finite overall test MSE over the 19 predicted
  snapshots for different stability configurations. Legend annotations give the
  relative MSE change from $p=50$ to $p=100$ for each curve.}
  \label{fig:stability_robustness}
\end{figure}
\cref{fig:stability_robustness} reveals two different robustness issues.
According to \cref{fig:stability_robustness}(a), when all three stabilization
mechanisms are disabled (\textsc{OFF}), Spherical Shallow Water develops
non-finite training behavior at larger latent dimensions.
In \cref{fig:stability_robustness}(b), under the same \textsc{OFF}
configuration, Navier--Stokes remains numerically finite, but its prediction
error becomes strongly dependent on the selected latent dimension and tends to
increase rapidly as the latent dimension grows.

On Spherical Shallow Water, SMORE-LRLQ exhibits a severe numerical failure mode. Of the
six configurations shown in \cref{fig:stability_robustness}(a), four
eventually produce NaN at one or more latent dimensions, and these
failures occur earlier as the latent dimension increases. For the
OFF configuration the first observed NaN occurs at epoch 8111 for $p=50$,
but already at epochs 200 and 100 for $p=70$ and $p=100$, respectively.
Neither the stability regularizers alone nor gradient clipping alone
prevents failure at the two larger latent dimensions, and removing only
the low-rank parameterization from FULL still fails at $p=70$ and $p=100$
(epochs 160 and 137). In contrast, the low-rank quadratic parameterization alone is sufficient to
prevent non-finite training across all three tested latent dimensions,
indicating that it is the main mechanism mitigating the increased instability
at larger $p$. At $p=50$, however, the stability regularizers also have a
clear effect on long-term training stability: the OFF configuration develops
NaN at epoch 8111, whereas the configuration with stability regularization
alone remains finite throughout the 10,000-epoch training window. At $p=100$ this benefit
also appears in accuracy, where removing the low-rank parameterization
increases the best finite rollout MSE from $0.6299\times10^{-3}$ to
$5.585\times10^{-3}$ before the run diverges, a $+786.8\%$ degradation
(\Cref{tab:stability_ablation_main}). However, a low error at the best checkpoint does not guarantee that training
remains stable. At $p=50$, the OFF and clip only configurations reach best finite MSEs
comparable to FULL, yet both subsequently become non-finite, as shown in
Appendix Table~\ref{tab:app_sw_mechanism_sweep}.

For Navier--Stokes, every configuration follows a U-shaped trend in $p$,
dropping sharply to a minimum near $p=50$--$70$ and then rising on the
high-capacity side. The configurations differ mainly in how steeply the
error rises on this side, summarized by the relative change
$\Delta\mathrm{MSE}_{50\to100}$ from $p=50$ to $p=100$ reported in the
legend. FULL has the shallowest right arm ($+16.6\%$) and is therefore the least
sensitive of all configurations to the exact choice of $p$ once the latent
space dimension is large. Removing a single mechanism
from FULL increases this sensitivity only slightly for the stability
regularizers ($+18.9\%$) and gradient clipping ($+20.5\%$), but sharply
for the low-rank parameterization ($+95.3\%$), while the two configurations with neither the stability regularizers nor the
low-rank parameterization are the most sensitive of all (gradient clipping only
$+133.7\%$, OFF $+267.1\%$). For example, the clip only configuration is marginally better than FULL at $p=50$
($1.939\times10^{-3}$ versus $2.003\times10^{-3}$) but rises to
$4.532\times10^{-3}$ at $p=100$, whereas FULL changes only from
$2.003\times10^{-3}$ to $2.337\times10^{-3}$ over the same range and in
fact reaches its own minimum of $1.923\times10^{-3}$ at $p=70$. A low error at a single latent dimension therefore does not imply robustness
across dimensions. The ordering of $\Delta\mathrm{MSE}_{50\to100}$ identifies
the low-rank quadratic parameterization as the largest contributor to
robustness at high latent dimensions. The stability regularizers and gradient
clipping provide smaller additional protection. Together, they make FULL far
less dependent on a well-chosen $p$. This matters in practice, since the
intrinsic dimension of a physical system is rarely known in advance and the
latent dimension is typically chosen with a margin. FULL remains accurate even when more resources are provided than necessary, which makes it applicable to a
broader range of systems without careful tuning of $p$.

\Cref{tab:stability_ablation_main} isolates the contribution of each
mechanism at a fixed latent dimension $p=100$ across the datasets,
and shows that these contributions are strongly regime-dependent. On
Navier--Stokes, removing any mechanism degrades accuracy, most strongly
the low-rank parameterization ($+122.2\%$) and the stability regularizers
($+48.7\%$), so all three contribute at this capacity. On Shallow Water
the differences are negligible ($+1.7\%$ and $+0.7\%$ for the
stability regularizers and gradient clipping), except for the low-rank column.
Its removal causes divergence and triggers a non-finite run ($\dagger$) after the
large accuracy loss noted above, so here the low-rank parameterization is
essential for stable training rather than for accuracy. 

\begin{table}[!htbp]
  \centering
  \caption{Leave-one-out ablation of the FULL SMORE-LRLQ configuration (all
mechanisms enabled) at per-component latent dimension $p=100$.
Each entry is the total MSE over the 19 predicted snapshots (snapshots
2--20) under the 1000-step initial latent fitting during evaluation; the value in
parentheses is the relative change from FULL, where positive means
higher error.
The columns remove the
Lyapunov-guided stability regularization, gradient clipping, and the low-rank
quadratic parameterization respectively.
A dagger ($\dagger$) marks a run that diverged to non-finite values during training; its best pre-divergence checkpoint is used for evaluation.}
  \small
  \setlength{\tabcolsep}{6pt}
  \begin{tabular}{lcccc}
    \toprule
    Dataset
    & FULL
    & FULL without $L_{\mathrm{reg}}$
    & FULL without $\mathrm{clip}$
    & FULL without $\mathrm{low\mbox{-}rank}$ \\
    \midrule
    Navier--Stokes
      & 2.337e-3
      & 3.475e-3 \; (+48.7\%)
      & 2.796e-3 \; (+19.7\%)
      & 5.191e-3 \; (+122.2\%) \\
    Shallow Water
      & 6.299e-4
      & 6.403e-4 \; (+1.7\%)
      & 6.342e-4 \; (+0.7\%)
      & 5.585e-3$^\dagger$ \; (+786.8\%) \\
    \bottomrule
  \end{tabular}
  \label{tab:stability_ablation_main}
\end{table}

The mechanism study shows that the three stability-related design
choices play distinct and system-dependent roles.
For Spherical Shallow Water, the low-rank quadratic parameterization is the
primary mechanism for mitigating the deterioration in training stability as
the latent dimension increases, while the Lyapunov-guided stability
regularization prevents the long-term divergence observed at $p=50$ when
the regularization is disabled.
For Navier--Stokes, all three mechanisms help reduce the sensitivity of
prediction accuracy to the per-component latent dimension.
Together, these results show that the stability-related design in SMORE
primarily provides flexibility in choosing the latent dimension and
robustness to numerical training failure, rather than uniformly reducing prediction error in every
setting.

\subsection{Prediction from Sparse Initial Observations}
\label{sec:results_sparse_mask_ablation}
In this section, we investigate whether SMORE can retain full-field predictive
capability when only sparse spatial observations of the initial state are
available.
While the previous experiments assume access to the complete initial field,
many practical settings provide only a limited number of measurements at the
initial time.
We therefore test whether SMORE trained entirely on fully observed fields can
infer a useful initial latent state from sparse test-time measurements and
subsequently recover the complete spatiotemporal trajectory without retraining.
We use the three benchmark datasets introduced in
\Cref{sec:experiment_setup} together with the corresponding SMORE models
trained on fully observed spatial fields in
\Cref{sec:results_long_horizon_extrapolation}.
At test time, only a fraction $v_{\mathrm{obs}}$ of the spatial locations in
the initial snapshot is retained, with the observed locations sampled uniformly
without replacement.
Only these observed values are used to fit the initial latent state through
the autodecoding procedure, after which prediction proceeds with the trained
decoder and latent-dynamics model.
Details of the spatial masks and the corresponding observation counts for each
dataset are provided in Appendix~\ref{app:mask_results_full}.

\begin{table}[htbp]
  \centering
  \caption{
  Sparse-initialization performance under selected initial-snapshot observation
  ratios.
  Each entry reports overall MSE over the 19 predicted snapshots,
averaged across three random seeds.
The fitted initial snapshot is excluded from the metric.
  Only the retained samples of the first snapshot are used to infer the initial
  latent state; the subsequent rollout reconstructs complete spatial fields
  without retraining.
  The complete visibility sweep with standard deviations is reported in
  Appendix \Cref{tab:app_mask_overall_test_full}.
  }
  \small
  \setlength{\tabcolsep}{3.5pt}
  \begin{tabular}{lcccccc}
    \toprule
    Dataset (model)
      & $100\%$
      & $50\%$
      & $10\%$
      & $4\%$
      & $2\%$
      & $1\%$ \\
    \midrule
    \makecell[l]{Wave\\(SMORE-Koopman)}
      & 2.01e-04
      & 2.05e-04
      & 1.98e-04
      & 2.11e-04
      & 6.96e-04
      & 4.06e-02 \\
    \makecell[l]{Navier--Stokes\\(SMORE-LRLQ)}
      & 1.53e-03
      & 1.53e-03
      & 1.60e-03
      & 1.78e-03
      & 3.77e-03
      & 2.48e-02 \\
    \makecell[l]{Shallow Water\\(SMORE-Koopman)}
      & 3.86e-04
      & 3.94e-04
      & 4.55e-04
      & 6.71e-04
      & 1.32e-03
      & 1.56e-03 \\
    \bottomrule
  \end{tabular}
  \label{tab:mask_overall_test}
\end{table}

\Cref{tab:mask_overall_test} reports a representative subset of visibility
levels; the full sweep with across-seed standard deviations is given in
Appendix \Cref{tab:app_mask_overall_test_full}.
Retaining $10\%$ of the initial snapshot ($v_{\mathrm{obs}}=10\%$) changes the
overall MSE by only $-1.5\%$ for Wave, $+4.6\%$ for Navier--Stokes, and
$+17.9\%$ for Shallow Water relative to full observation, a modest change in
all three cases.
More generally, every system tolerates substantial sparsity, with the error
staying within roughly a factor of two of its full-observation value down to
about $4\%$ visibility.
The systems diverge only under extreme sparsity.
For Wave and Navier--Stokes, the error increases moderately at \(2\%\) visibility, reaching about \(3.5\times\) and \(2.5\times\) their respective full-observation baselines, before rising sharply at \(1\%\) to \(4.06\times10^{-2}\) and \(2.48\times10^{-2}\). These values are approximately \(200\times\) and \(16\times\) the corresponding full-observation errors. In contrast, the error for Shallow Water increases much more gradually as visibility is reduced, from \(3.86\) to \(3.94\), \(4.55\), \(6.71\), \(13.2\), and \(15.6\), in units of \(10^{-4}\), and reaches only about \(4\times\) its baseline at \(1\%\).

This difference at extreme sparsity is consistent with the loss of identifiability characterized by \Cref{lem:sensor_jacobian}. Because the decoder is affine in the latent code, the observability constant \(c_{\mathcal S}=\sigma_{\min}(J_{\mathcal S})\) can be computed exactly for each mask (\Cref{tab:app_mask_observability}). It decreases as visibility is reduced and becomes zero once the number of observed locations falls below \(p\), which occurs at \(2\%\) and \(1\%\) for Wave and at \(1\%\) for Navier--Stokes and Shallow Water. The vanishing of $c_{\mathcal{S}}$ marks the loss of observability.
For Wave and Navier--Stokes, the strongest error degradation occurs at
or beyond this point, whereas Shallow Water degrades more gradually.
\Cref{fig:sparse_initialization_examples} shows example predictions at $10\%$ visibility.
From a single sparse initial snapshot, the frozen decoder and latent dynamics
reconstruct the complete field sequence, and the predictions closely track the
ground truth across the rollout.

\begin{figure}[!htbp]
  \centering

  \begin{subfigure}{0.85\textwidth}
    \centering
    \includegraphics[width=\textwidth]{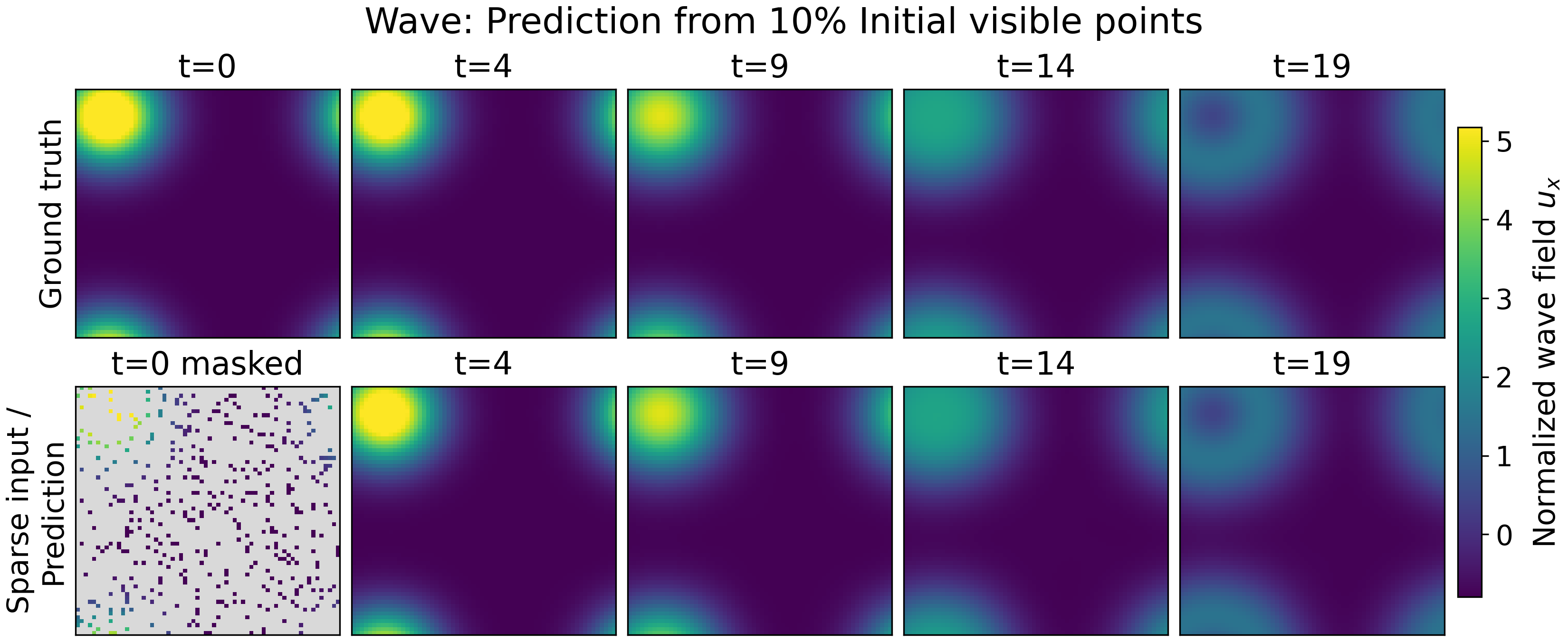}
    \caption{Wave using SMORE-Koopman.}
    \label{fig:sparse_wave}
  \end{subfigure}

  \vspace{0.6em}

  \begin{subfigure}{0.85\textwidth}
    \centering
    \includegraphics[width=\textwidth]{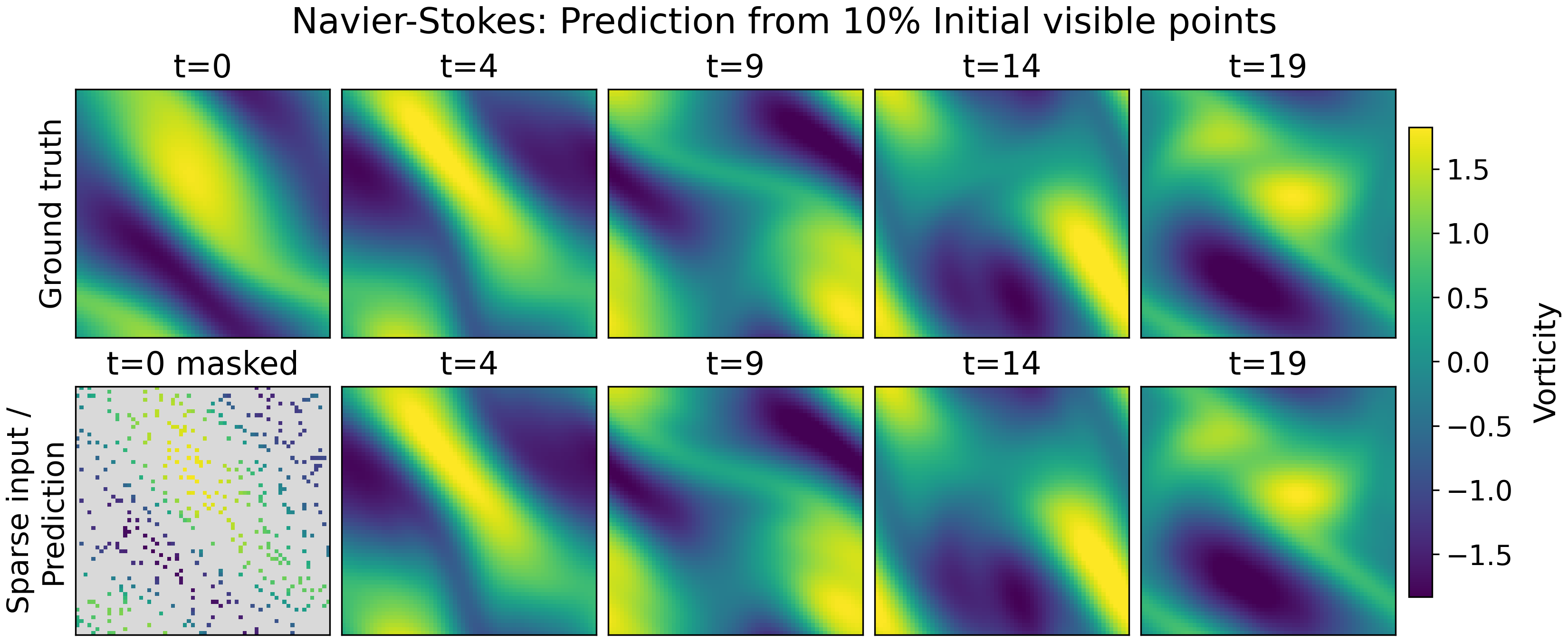}
    \caption{Navier--Stokes using SMORE-LRLQ.}
    \label{fig:sparse_ns}
  \end{subfigure}

  \vspace{0.6em}

  \begin{subfigure}{0.85\textwidth}
    \centering
    \includegraphics[width=\textwidth]{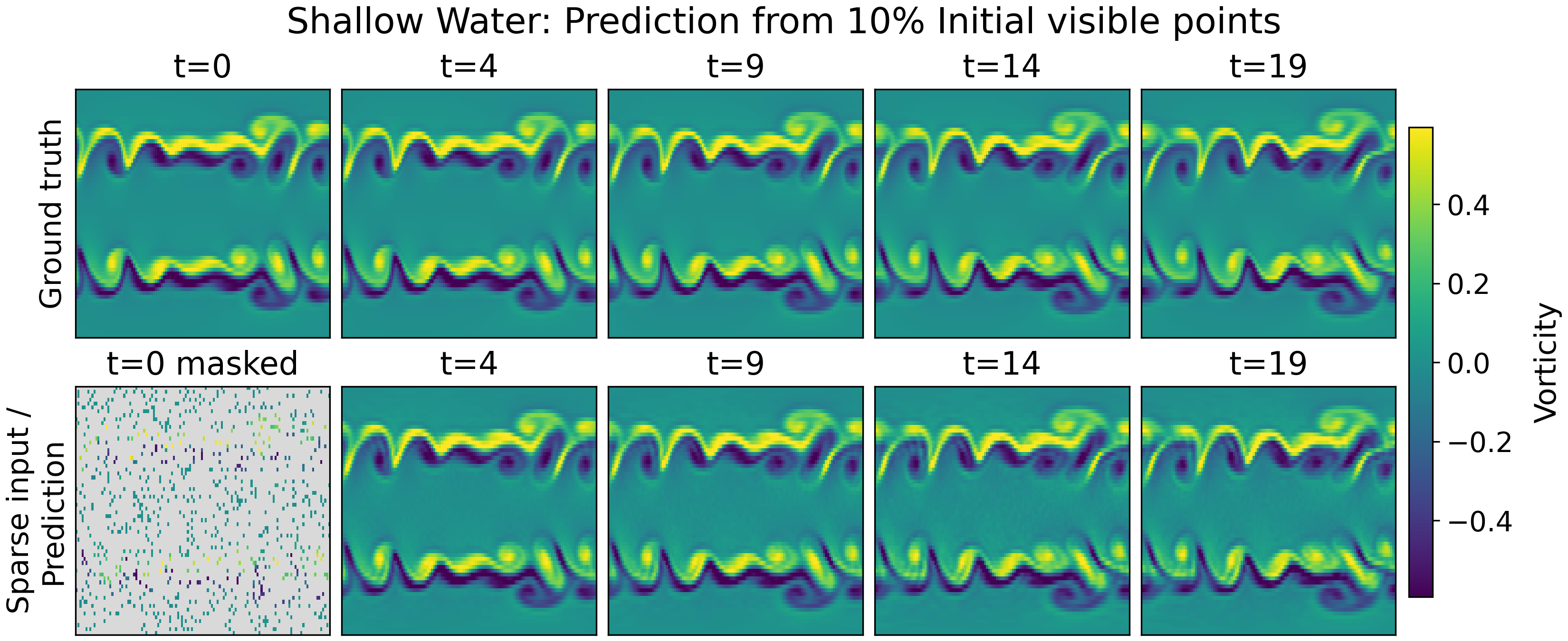}
    \caption{Spherical Shallow Water using SMORE-Koopman.}
    \label{fig:sparse_sw}
  \end{subfigure}

  \caption{
  Sparse initialization at $10\%$ visibility for
  (a) Wave, (b) Navier--Stokes, and (c) Spherical Shallow Water.
  In each panel, the top row shows the ground-truth field at snapshot
  indices $n=0,4,9,14,19$ (labeled $t$ in the panels).
  The bottom-left image shows the sparse initial observation,
  which is the only test-time input, while the remaining bottom images
  show full-field predictions at the corresponding later snapshots.
  Only the observed samples from the initial snapshot are used to infer
  $\widehat{\alpha}_0$; the decoder and latent dynamics remain frozen.
  Wave and Spherical Shallow Water use SMORE-Koopman, whereas
  Navier--Stokes uses SMORE-LRLQ.
  }
  \label{fig:sparse_initialization_examples}
\end{figure}

The sparse-initialization study demonstrates a practical capability of the
coordinate-based latent representation. A model trained only on fully observed
fields can be reused under substantially sparser initial measurements, without
retraining, and still produce complete future fields. This is not spatial interpolation of the observed locations.
The observed values are used to identify an initial state in the learned
latent representation. The latent dynamics then evolves this state in time,
and the frozen decoder maps it to full spatiotemporal fields. The capability
has a clear limit. Under extremely sparse observations, the error increases
sharply, since successful sparse-to-rollout inference requires the observed
values to carry enough information to identify the initial latent state. This
limit is set by the number of observed locations relative to the
per-component latent dimension $p$ rather than by the visible ratio itself,
since the same visible ratio yields different numbers of observed locations on
different grids; at $1\%$, for example, it retains $40$ locations on the
$64\times64$ grids but $81$ on the $64\times128$ grid.

\subsection{Memory Efficiency and Spatial Capacity}
\label{sec:results_memory_efficiency}

In this section, we investigate the practical benefits of the pointwise
minibatching strategy in SMORE. For neural models used to solve
time-dependent PDEs, increasing spatial resolution can make training
increasingly memory intensive. This issue remains relevant even for latent-space models such as SMORE.
Although temporal evolution is performed in a compact latent space, the
decoder must still evaluate and backpropagate through the physical field over
the spatial grid during training. As the number of spatial points increases, this decoder
evaluation can become a major memory bottleneck. The coordinate-based
decoder in SMORE provides a natural way to alleviate it. Pointwise
minibatching divides the spatial coordinates into pointwise minibatches
without changing either the decoder architecture or the latent-dynamics
model, which can greatly reduce GPU-memory usage and, under a fixed
memory budget, make much larger spatial grids computationally feasible.

We examine pointwise minibatching from two complementary perspectives.
First, we test how much it increases the maximum spatial resolution that
can be processed under a fixed GPU-memory budget, and compare this
capacity with models that use their standard spatial execution
strategies. Second, at a fixed spatial resolution, we quantify the
corresponding changes in memory usage, computation time, and prediction
accuracy. Throughout these experiments, we cap the pointwise-minibatch
chunk size at 64 spatial coordinates and refer to this setting as PM64.
All spatial coordinates are still processed, so PM64 changes only the
batching strategy rather than subsampling the observed field. If the
total number of spatial points is not divisible by 64, the final chunk
contains the remaining points; for example, 150 spatial points are
processed as chunks of $64+64+22$. A smaller chunk size could further
reduce the peak memory requirement and potentially extend the feasible
spatial resolution, at the cost of additional computation and execution
overhead, a trade-off we do not explore here.

\subsubsection{Maximum Spatial Resolution on One GPU}

We first examine whether the memory reduction provided by pointwise
minibatching translates into greater spatial capacity. We
evaluate all models introduced in \Cref{sec:experiment_setup}. For each
model, we determine the largest square grid on which it can complete a
fixed training-and-evaluation protocol on a single A100 40\,GB GPU. For
this capacity experiment, all models use approximately matched parameter
counts and the same hyperparameter configurations as in the
Navier--Stokes comparison in
\Cref{sec:results_long_horizon_extrapolation}, with a common batch size
of 32. Instead of using a physical simulation dataset, we generate
square-grid tensors at different spatial resolutions, and train each
model for two epochs at each resolution. The grid side length is then
progressively increased until training produces a CUDA out-of-memory
(OOM) error, thereby determining the maximum spatial resolution supported
by each model under the fixed GPU-memory budget. The two SMORE variants
use PM64, while DINo, FNO, CNO, and Transolver are evaluated with their
standard full-grid or native implementations. The two-epoch protocol is
designed only to identify the memory limit and does not establish
predictive performance or long-training feasibility at the limiting
resolutions.

\begin{table}[!htbp]
  \centering
  \caption{Maximum successful square grids on one A100 40\,GB GPU under
the two-epoch capacity protocol, together with model parameter counts on
Navier--Stokes. Each maximum successful grid is one step below a CUDA OOM
failure (i.e., the next larger square grid runs out of memory). PM64
denotes pointwise minibatches of 64 coordinates in all decoder
evaluations, including training autodecoding, test-time initial latent
fitting, and evaluation field reconstruction. Grid points is the number
of spatial points at the maximum successful grid, and the last column
reports this count relative to Transolver.}
  \label{tab:ns_spatial_capacity}
  \small
  \setlength{\tabcolsep}{2pt}
  \begin{tabular}{llcccc}
    \toprule
    Model & Spatial execution & Parameters & Maximum successful grid
      & Grid points & Relative points \\
    \midrule
    Transolver~\cite{wu2024Transolver}
      & Full grid & 358{,}007 & $64\times64$   & 4{,}096   & $1.00\times$ \\
    DINo~\cite{Yin2022Continuous}
      & Full grid & 356{,}444 & $163\times163$ & 26{,}569  & $6.49\times$ \\
    FNO~\cite{li2020fourier,bonev2023spherical}
      & Full grid & 336{,}653 & $271\times271$ & 73{,}441  & $17.93\times$ \\
    CNO~\cite{raonic2023convolutional}
      & Full grid & 326{,}184 & $307\times307$ & 94{,}249  & $23.01\times$ \\
    SMORE-Koopman
      & PM64      & 355{,}165 & $1714\times1714$ & 2{,}937{,}796 & $717.24\times$ \\
    SMORE-LRLQ
      & PM64      & 353{,}961 & $1714\times1714$ & 2{,}937{,}796 & $717.24\times$ \\
    \bottomrule
  \end{tabular}
\end{table}

The results in \Cref{tab:ns_spatial_capacity} show that pointwise
minibatching increases the spatial capacity of SMORE under the tested memory budget. All models use approximately
matched parameter counts, between $326{,}184$ and $358{,}007$ (within
about $10\%$), so the comparison reflects the spatial execution strategy
rather than model size. Measured by side length, SMORE-LRLQ and
SMORE-Koopman both reach a maximum successful grid of $1714\times1714$,
the largest of all models and well beyond the best physical-space
baseline (CNO at $307\times307$). The gap is larger in terms of the
actual number of spatial points processed, which is the quantity that
drives decoder memory. Both SMORE variants handle $2{,}937{,}796$ points,
corresponding to $717.24\times$ the Transolver baseline and roughly
$31\times$ the points of CNO. Pointwise minibatching therefore enables
SMORE to operate at higher spatial resolutions than any
baseline evaluated here.

The two SMORE variants reach the same maximum grid size because their peak memory usage is dominated by data storage rather than by the model architecture at these resolutions. Pointwise minibatching limits the memory required by decoder activations, since it depends on the chunk size instead of the full grid size. As a result, most of the memory is used to store the input fields, predicted fields, and temporary tensors used to compute the error. These tensors have the same size for both variants because they use the same batch size, grid size, and number of snapshots. For example, at a resolution of \(1715\times1715\), storing 32 trajectories with 20 snapshots of a single-channel float32 field alone requires approximately 7 GiB. When the next pointwise chunk is processed during training, the additional memory required for its activations and temporary tensors exceeds the remaining available GPU memory, resulting in an out-of-memory error.

\subsubsection{Resource and Accuracy Changes at a Fixed Resolution}

Having established the spatial-resolution benefit of PM64, we next evaluate
its effects on GPU memory usage, computation time, and prediction
accuracy, using the SMORE-LRLQ model on the Navier--Stokes dataset at a
fixed spatial resolution of $64\times64$ as a representative case. We use
the same Navier--Stokes SMORE-LRLQ configuration as in
\Cref{sec:results_long_horizon_extrapolation}, except that it
additionally uses pointwise minibatching. These runs use four A100
80\,GB GPUs and are separate from the single-GPU capacity experiment
above; the reported memory is the peak allocated memory logged on rank 0.
\Cref{tab:ns_pointwise_tradeoff} reports training memory, elapsed time,
and prediction error.

\begin{table}[!htbp]
  \centering
  \caption{Observed resource and accuracy changes for $64\times64$
  Navier--Stokes SMORE-LRLQ. Memory is the logged training-epoch peak
  allocated memory on rank 0. Elapsed times are measured over complete
  5,000-epoch runs and refer to one representative run per execution
  setting; these timings exclude the final 1,000-step test stage. MSEs
  are mean $\pm$ sample standard deviation over three random seeds.}
  \label{tab:ns_pointwise_tradeoff}
  \small
  \begin{tabular}{lcc}
    \toprule
    Metric & Full-grid decoding & PM64 \\
    \midrule
    Recorded peak allocated memory (MB) & 18{,}613.9 & 314.6 \\
    Time for 5,000 epochs & 1\,h 8\,min 21\,s & 4\,h 4\,min 32\,s \\
    In-horizon MSE ($\times10^{-3}$) & $1.008\pm0.025$ & $1.048\pm0.014$ \\
    Extrapolation MSE ($\times10^{-3}$) & $2.000\pm0.199$ & $2.111\pm0.262$ \\
    Overall MSE ($\times10^{-3}$) & $1.530\pm0.101$ & $1.607\pm0.132$ \\
    \bottomrule
  \end{tabular}
\end{table}

At the fixed $64\times64$ resolution, PM64 reduces the recorded peak
allocated training memory from $18{,}613.9$ to $314.6$\,MB, a
$59.2\times$ reduction (about $98.31\%$), while keeping the model
architecture, spatial resolution, and per-rank batch size unchanged. The
main effect of PM64 is therefore to trade reduced memory consumption for
increased computation time. The representative full-grid run completes
5,000 epochs in 1\,h 8\,min 21\,s, whereas PM64 requires 4\,h 4\,min
32\,s, an increase of about $3.58\times$. The prediction error changes only slightly. The overall MSE increases by
about $5.0\%$, from $1.530\times10^{-3}$ to $1.607\times10^{-3}$, which is
within the across-seed standard deviation. The increase is similar in the
two windows, $4.0\%$ in-horizon (from $1.008$ to $1.048\times10^{-3}$) and
$5.6\%$ in extrapolation (from $2.000$ to $2.111\times10^{-3}$).

\FloatBarrier

\section{Conclusions}
\label{sec:Conclusions}

We introduced SMORE, a stability-promoting mesh-agnostic
reduced-order modeling framework for time-dependent PDEs.
SMORE separates spatial representation from temporal evolution. An INR autodecoder represents the physical field using a decoder modulated by a compact latent state that can be queried at arbitrary coordinates, while structured latent dynamics evolve the latent state in time.
We considered two variants, SMORE-Koopman and SMORE-LRLQ, corresponding to linear and low-rank linear--quadratic latent
dynamics, respectively.

On the theoretical side, we derived Lyapunov-based sufficient conditions for
the stability of latent dynamics based on spectral
dissipativity and cyclic energy cancellation.
These conditions motivate the stability regularizers used during training.
We further showed that the decoder is affine in the latent code. This gives
explicit constants that transfer latent-state bounds and errors to the
reconstructed fields. We also bounded how errors in sparse initial-state
recovery propagate into future full-field predictions. The recovery error is
governed by the smallest singular value of the sensor Jacobian, and its
propagation is controlled under an explicit rollout assumption.

The numerical experiments examine four complementary aspects of the proposed
framework.
First, across the 2D Wave, 2D Navier--Stokes, and 3D Spherical Shallow Water
benchmarks, SMORE achieves the lowest extrapolation error in each
case at comparable parameter budgets.
SMORE-Koopman performs best on Wave and Spherical Shallow Water, whereas
SMORE-LRLQ provides the strongest performance on the more nonlinear
Navier--Stokes system, demonstrating the complementary expressive capabilities
of the linear and linear--quadratic latent dynamics.

Second, the mechanism ablations show that the stability-related design choices
play system-dependent roles.
For Navier--Stokes, the low-rank quadratic parameterization is the dominant contributor to reducing sensitivity to the selected latent dimension, with the Lyapunov-guided regularization and gradient clipping providing additional robustness.
For Spherical Shallow Water, the low-rank quadratic parameterization is
particularly important for preventing non-finite training behavior.
The Wave results further show that these mechanisms should not be interpreted
as universal improvements in prediction accuracy; their main practical value
is flexibility in choosing the latent dimension and robustness to numerical
training failure.

Third, SMORE retains full-field predictive capability when only sparse
observations of the initial state are available.
Models trained entirely on fully observed fields can infer an initial latent
state from sparse test-time measurements and subsequently predict complete
future fields without retraining.
Prediction remains accurate under substantial spatial sparsity, and its sharp
degradation under extreme sparsity coincides with the loss of identifiability,
$c_{\mathcal S}=0$, predicted by the analysis.

Finally, pointwise minibatching provides a practical trade-off of memory and computation. 
PM increases the maximum spatial resolution on a single GPU
to $1714\times1714$, compared with $307\times307$ for the best baseline CNO.
At a fixed $64\times64$ resolution, we compare SMORE-LRLQ trained with PM64
against the same model trained with full-grid decoding. PM64 reduces the recorded peak allocated training memory by a factor of $59.2$, while increasing training time by approximately
$3.58\times$. The overall MSE changes by about $5.0\%$, within the across-seed standard deviation. 
These gains are obtained without changing either the decoder architecture or
the latent-dynamics formulation.

Overall, the results show that SMORE combines structured latent evolution with
coordinate-based decoding to support long-horizon extrapolation, improved
empirical robustness, sparse-initialization inference, and memory-efficient
spatial execution within a single reduced-order modeling framework.
Future work will extend SMORE to parametric PDEs by conditioning the decoder and latent dynamics on physical parameters, so that a single trained model can represent families of PDE systems.
We will also investigate adaptive weighting of the Lyapunov-guided stability
regularization to better balance prediction accuracy and progress toward the
desired stability conditions during training.
Additional directions include more complex geometries and unstructured meshes.

\section*{Acknowledgement}
This work is supported by the National Science Foundation (NSF) under Award No. 2616260. Shaowu Pan acknowledges support from the Google Research Scholar Program. This research used resources of the National Energy Research Scientific Computing Center (NERSC), a U.S. Department of Energy Office of Science User Facility supported by the Office of Science under Contract No. DE-AC02-05CH11231, through NERSC Award DDR-ERCAP0035696. Computational resources were also provided by the Alpha HPC cluster operated by the Empire AI Consortium, Inc., with support from Empire State Development (State of New York), the Simons Foundation, and the Secunda Family Foundation. This work further used Purdue Anvil, Texas A\&M University FASTER, and Texas A\&M University ACES through allocations PHY240112 and MCH260003 from the Advanced Cyberinfrastructure Coordination Ecosystem: Services \& Support (ACCESS) program, which is supported by U.S. National Science Foundation Grants Nos. 2138259, 2138286, 2138307, 2137603, and 2138296. Additional computing hardware support was provided by the NVIDIA Academic Grant Program. This research used resources of the Oak Ridge Leadership Computing Facility at the Oak Ridge National Laboratory, which is supported by the Office of Science of the U.S. Department of Energy under Contract No. DE-AC05-00OR22725. 
The authors also thank Mr. Jake Herman for his numerical experiments with the Koopman operator.

\bibliographystyle{unsrt}
\bibliography{\smorebasedir paper_cited,\smorebasedir arxiv_missing_refs}

\appendix
\appendix
\newpage

\section{Multiplicative Filter Network}
\label{apdx:mfn}

A MFN replaces the compositional nonlinearities of a standard feedforward network with element-wise multiplication by sinusoidal filters of the input.    An $L$-layer MFN $D_{\phi}\colon\Omega\to\mathbb{R}$ with Fourier filters is defined by:
\begin{align}
    \label{eq:mfn-1}
    h^{(0)}(\mathbf{x}) &= \mathbf{x}, \\
    \label{eq:mfn-2}
    h^{(l)}(\mathbf{x}) &= \big(W^{(l-1)}h^{(l-1)}(\mathbf{x}) + b^{(l-1)}\big) \circ g(\mathbf{x}; w^{(l)}), \quad l=1,\ldots,L-1, \\
    \label{eq:mfn-3}
    D_{\phi}(\mathbf{x}) &= W^{(L-1)} h^{(L-1)}(\mathbf{x}) + b^{(L-1)}.
\end{align}
Here $L$ counts the $L-1$ filtered hidden layers and the output layer; in terms of the number of hidden linear layers $\ell_{\mathrm{dec}}$ reported in Appendix~\Cref{tab:app_hparam_candidates}, $L=\ell_{\mathrm{dec}}+2$. In \cref{eq:mfn-1,eq:mfn-2,eq:mfn-3}, $\mathbf{x}\in \Omega$ and 
\begin{equation}
    \label{eq:fourier-embedding}
    g(\mathbf{x};w^{(l)})=
\begin{bmatrix}
\sin(w^{(l)}\mathbf{x})\\
\cos(w^{(l)}\mathbf{x})
\end{bmatrix}
\end{equation}
is a Fourier embedding with frequency matrix $w^{(l)}=\frac{w_s}{\sqrt{L-1}}\,\tilde{w}^{(l)}$, where $\tilde{w}^{(l)}$ is trainable and the scaling factor $w_s$ is shared by all $L-1$ filters, and $\circ$ is the element-wise multiplication. The dimensions are $h^{(l)}(\mathbf{x})\in\mathbb{R}^{h_{\mathrm{dec}}}$ for $l=1,\ldots,L-1$, $\tilde{w}^{(l)}\in\mathbb{R}^{(h_{\mathrm{dec}}/2)\times d}$ so that $g(\mathbf{x};w^{(l)})\in\mathbb{R}^{h_{\mathrm{dec}}}$, $W^{(l)}\in\mathbb{R}^{h_{\mathrm{dec}}\times h_{\mathrm{dec}}}$ for $l=1,\ldots,L-2$, $b^{(l)}\in\mathbb{R}^{h_{\mathrm{dec}}}$ for $l=0,\ldots,L-2$, and $W^{(L-1)}\in\mathbb{R}^{1\times h_{\mathrm{dec}}}$, $b^{(L-1)}\in\mathbb{R}$ for the scalar output layer. The set of trainable MFN parameters is 
\begin{equation}
    \label{eq:mfn_def_phi}
    \phi \coloneqq \{\{\tilde{w}^{(l)}\}_{l=1}^{L-1}, \{W^{(l)}\}_{l=1}^{L-1}, \{b^{(l)}\}_{l=0}^{L-1} \}.
\end{equation}
For the first layer, we set $W^{(0)}=0$, so that $h^{(1)}(\mathbf{x})=b^{(0)}\circ g(\mathbf{x};w^{(1)})$ with a trainable bias $b^{(0)}$. The scaling factor $w_s$ is a hyperparameter that controls the initial frequency range~\cite{fathony2020multiplicative, sitzmann2020implicit}. In this study, we set $w_s=64$ for all datasets.

\section{Feature-wise Linear Modulation}
\label{apdx:film}

Here we describe how a component latent code $z \in \mathbb{R}^p$ modulates the decoder $D_{\phi}$ by replacing $\phi$ with the modulated parameters $\Phi_\theta(z)$, specifically through a shift modulation based on Feature-wise Linear Modulation (FiLM) \cite{perez2018film}. 
The main idea is to incorporate the effect of $z$ into \cref{eq:mfn-2} via the replacement in \cref{eq:film}.  
\begin{equation}
    \label{eq:film}
    h^{(l)}(\mathbf{x}) = \big(W^{(l-1)}h^{(l-1)}(\mathbf{x}) + b^{(l-1)} + M^{(l-1)}z\big) \circ g(\mathbf{x}; w^{(l)}), \quad l=1,\ldots,L-1,
\end{equation}
where $M^{(l)}\in\mathbb{R}^{h_{\mathrm{dec}}\times p}$, $l=0,\ldots,L-2$, are additional trainable matrices. This can be viewed as a restricted hypernetwork~\cite{pan2023neural} modulation $\Phi_{\theta}(z)$ where only the biases $b^{(0)},\ldots,b^{(L-2)}$ are influenced by $z$. The trainable decoder parameters are defined as $\theta \coloneqq \phi\cup\{M^{(l)}\}_{l=0}^{L-2}$. To summarize, the $z$-modulated MFN parameterized by $\theta$ can be written as $D_{\Phi_\theta(z)}\colon\Omega\to\mathbb{R}$ and the full field is assembled component-wise as in \cref{eq:fullfield}.

\section{Dissipativity and Conservation Defect Certificates for the Linear-Quadratic System}
\label{apdx:boundedness}

We give the proof of \Cref{thm:1} and the soft bound based on the conservation defect estimate used in \cref{eq:soft_defect_bound,eq:soft_lyapunov}.  Consider the linear-quadratic latent system
\begin{equation}
    \label{eq:app_lq_system}
    \dot{\alpha}
    =
    \mathbf{b}+\mathbf{W}\alpha
    +
    \begin{bmatrix}
        \alpha^\top \mathbf{Q}^{(1)} \alpha & \cdots & 
        \alpha^\top \mathbf{Q}^{(r)} \alpha
    \end{bmatrix}^{\top}.
\end{equation}

\subsection{Symmetrization of Quadratic Slices}
For each output component $i$, the scalar quadratic form satisfies
\begin{equation}
    \alpha^\top\mathbf{Q}^{(i)}\alpha
    =
    \alpha^\top
    \left(
    \frac{\mathbf{Q}^{(i)}+(\mathbf{Q}^{(i)})^\top}{2}
    \right)
    \alpha .
\end{equation}
Only the symmetric part of each slice is identifiable from the vector field.  We may replace $\mathbf{Q}^{(i)}$ by its symmetric part without changing the ODE.  Under this convention, the shifted state $\mathbf{y}=\alpha-\mathbf{m}$ satisfies
\begin{equation}
    \label{eq:app_shifted_lq}
    \dot{\mathbf{y}}
    =
    \mathbf{d}
    +
    \mathbf{A}\mathbf{y}
    +
    \begin{bmatrix}
        \mathbf{y}^\top \mathbf{Q}^{(1)}\mathbf{y} & \cdots & 
        \mathbf{y}^\top \mathbf{Q}^{(r)}\mathbf{y}
    \end{bmatrix}^{\top},
\end{equation}
where $\mathbf{d}$ and $\mathbf{A}$ are defined in \cref{eq:shifted_d,eq:shifted_A}.  The formula for $\mathbf{A}$ follows by expanding
\[
    (\mathbf{m}+\mathbf{y})^\top \mathbf{Q}^{(i)}(\mathbf{m}+\mathbf{y})
    =
    \mathbf{m}^\top \mathbf{Q}^{(i)}\mathbf{m}
    +
    2\mathbf{m}^\top \mathbf{Q}^{(i)}\mathbf{y}
    +
    \mathbf{y}^\top \mathbf{Q}^{(i)}\mathbf{y}.
\]
Hence the $i$-th row of $\mathbf{A}$ is the $i$-th row of $\mathbf{W}$ plus $2\mathbf{m}^\top\mathbf{Q}^{(i)}$, as in \cref{eq:shifted_A}.

\subsection{Cyclic Energy Cancellation}
Let $V(\mathbf{y})=\frac{1}{2}\Vert \mathbf{y}\Vert^2$.  The contribution of the quadratic term to $\dot V$ is
\begin{align}
    \sum_{i=1}^{r} y_i\,\mathbf{y}^{\top}\mathbf{Q}^{(i)}\mathbf{y}
    &=
    \sum_{i,j,k=1}^{r}
    \mathbf{Q}^{(i)}_{j,k}y_i y_j y_k  \\
    &=
    \frac{1}{3}\sum_{i,j,k=1}^{r}
    \left(
    \mathbf{Q}^{(i)}_{j,k}
    +\mathbf{Q}^{(k)}_{i,j}
    +\mathbf{Q}^{(j)}_{k,i}
    \right)y_i y_j y_k .
\end{align}
The second equality holds because $y_iy_jy_k$ is invariant under cyclic
permutations of $(i,j,k)$: relabeling the summation indices as
$(i,j,k)\to(k,i,j)$ and $(i,j,k)\to(j,k,i)$ turns the first sum into
$\sum\mathbf{Q}^{(k)}_{i,j}y_iy_jy_k$ and $\sum\mathbf{Q}^{(j)}_{k,i}y_iy_jy_k$,
respectively, and averaging the three equal expressions gives the result.
This step does not require the slices to be symmetric. The cyclic identity in \Cref{thm:1} makes the quadratic term energy-preserving:
\begin{equation}
    \label{eq:app_cyclic_zero}
    \sum_{i=1}^{r} y_i\,\mathbf{y}^{\top}\mathbf{Q}^{(i)}\mathbf{y}=0.
\end{equation}

\subsection{Proof of the Hard Certificate}
Using \cref{eq:app_shifted_lq,eq:app_cyclic_zero},
\begin{align}
    \dot V
    &=
    \mathbf{d}^{\top}\mathbf{y}
    +
    \mathbf{y}^{\top}\mathbf{A}\mathbf{y}  \\
    &=
    \mathbf{d}^{\top}\mathbf{y}
    +
    \mathbf{y}^{\top}\mathbf{A}_s\mathbf{y},
    \qquad
    \mathbf{A}_s=\frac{\mathbf{A}+\mathbf{A}^{\top}}{2}.
\end{align}
Since the cubic term vanishes, the only term that can counteract the
constant forcing $\mathbf{d}^\top\mathbf{y}$ is the linear term
$\mathbf{y}^\top\mathbf{A}_s\mathbf{y}$. Condition~1 of \Cref{thm:1}
requires this term to be dissipative: $\mathbf{A}_s\preceq-\gamma\mathbf{I}$
with $\gamma\geq0$ means
$\mathbf{y}^\top\mathbf{A}_s\mathbf{y}\leq-\gamma\Vert\mathbf{y}\Vert^2$
for all $\mathbf{y}$, or equivalently that every eigenvalue of the
symmetric matrix $\mathbf{A}_s$ is at most $-\gamma$. Bounding the first
term by the Cauchy--Schwarz inequality,
$\mathbf{d}^\top\mathbf{y}\leq\Vert\mathbf{d}\Vert\Vert\mathbf{y}\Vert$,
and the second by condition~1 gives
\begin{equation}
    \label{eq:app_hard_lyap}
    \dot V
    \leq
    \Vert \mathbf{d}\Vert\Vert\mathbf{y}\Vert
    -
    \gamma \Vert \mathbf{y}\Vert^2 .
\end{equation}
This is the basic energy estimate behind \Cref{thm:1}. For $\gamma>0$,
its right-hand side is negative whenever
$\Vert\mathbf{y}\Vert>\Vert\mathbf{d}\Vert/\gamma$, so the energy
decreases outside the ball of radius $\Vert\mathbf{d}\Vert/\gamma$
centered at $\mathbf{m}$; this is the absorbing ball of the system.
To turn \cref{eq:app_hard_lyap} into an explicit bound on
$\Vert\mathbf{y}(t)\Vert$, we avoid the non-differentiability of
$\Vert\mathbf{y}\Vert$ at $\mathbf{y}=\mathbf{0}$ as follows.
For $\epsilon>0$, let $w_\epsilon=\sqrt{2V+\epsilon^2}=\sqrt{\Vert\mathbf{y}\Vert^2+\epsilon^2}$.
Since $\Vert\mathbf{y}\Vert\le w_\epsilon$ and $\Vert\mathbf{y}\Vert^2=w_\epsilon^2-\epsilon^2$,
\cref{eq:app_hard_lyap} gives
\[
    \dot w_\epsilon
    =\frac{\dot V}{w_\epsilon}
    \le \Vert\mathbf{d}\Vert-\gamma w_\epsilon+\frac{\gamma\epsilon^2}{w_\epsilon}
    \le \Vert\mathbf{d}\Vert+\gamma\epsilon-\gamma w_\epsilon ,
\]
where the last step uses $w_\epsilon\geq\epsilon$. This is a linear
differential inequality $\dot w_\epsilon\leq a-\gamma w_\epsilon$ with
$a=\Vert\mathbf{d}\Vert+\gamma\epsilon$. By the comparison lemma,
$w_\epsilon$ does not exceed the solution of $\dot v=a-\gamma v$ with
$v(0)=w_\epsilon(0)$, that is,
\[
    w_\epsilon(t)
    \leq
    e^{-\gamma t}w_\epsilon(0)
    +a\,\frac{1-e^{-\gamma t}}{\gamma},
\]
where the factor $(1-e^{-\gamma t})/\gamma$ is interpreted as its limit
$t$ when $\gamma=0$. Since $\Vert\mathbf{y}(t)\Vert\leq w_\epsilon(t)$
and $w_\epsilon(0)=\sqrt{\Vert\mathbf{y}(0)\Vert^2+\epsilon^2}$, this gives
\[
    \Vert\mathbf{y}(t)\Vert
    \leq
    e^{-\gamma t}\sqrt{\Vert\mathbf{y}(0)\Vert^2+\epsilon^2}
    +(\Vert\mathbf{d}\Vert+\gamma\epsilon)\,\frac{1-e^{-\gamma t}}{\gamma}.
\]
The left-hand side is independent of $\epsilon$, so letting
$\epsilon\to0$ yields, on the interval of existence,
\begin{equation}
\label{eq:app_comparison}
    \Vert\mathbf{y}(t)\Vert
    \le
    e^{-\gamma t}\Vert\mathbf{y}(0)\Vert
    +\Vert\mathbf{d}\Vert\,\frac{1-e^{-\gamma t}}{\gamma}.
\end{equation}
The right-hand side is finite for every $t$, so no finite-time blow-up
occurs and the local solution extends to all $t\ge0$.
The three statements of \Cref{thm:1} follow directly from
\cref{eq:app_comparison}: letting $t\to\infty$ with $\gamma>0$ gives
\cref{eq:hard_absorbing_radius}; setting $\mathbf{d}=\mathbf{0}$ gives
\cref{eq:hard_exponential_decay}, which reduces to the non-expansive
bound when $\gamma=0$; and setting $\gamma=0$ with
$\mathbf{d}\neq\mathbf{0}$ gives \cref{eq:hard_linear_growth}.

\subsection{Soft Conservation Defect Certificate}
Define the cyclic-condition residual tensor
\begin{equation}
    T_{ijk}
    =
    \mathbf{Q}^{(i)}_{j,k}
    +\mathbf{Q}^{(k)}_{i,j}
    +\mathbf{Q}^{(j)}_{k,i}.
\end{equation}
By the identity above and Cauchy--Schwarz,
\begin{align}
    \left|
    \sum_{i=1}^{r} y_i\,\mathbf{y}^{\top}\mathbf{Q}^{(i)}\mathbf{y}
    \right|
    &=
    \left|
    \frac{1}{3}\sum_{i,j,k=1}^{r}T_{ijk}y_i y_j y_k
    \right|  \\
    &\leq
    \frac{1}{3}\Vert T\Vert_F
    \Vert \mathbf{y}\otimes\mathbf{y}\otimes\mathbf{y}\Vert_F \\
    &=
    \eta_Q \Vert \mathbf{y}\Vert^3,
    \qquad
    \eta_Q=\frac{1}{3}\Vert T\Vert_F.
\end{align}
Since $L_{\text{reg,cons}}=\Vert T\Vert_F^2$, we have $\eta_Q=\frac{1}{3}L_{\text{reg,cons}}^{1/2}$.  Combining this bound with $\mathbf{A}_s\preceq-\gamma\mathbf{I}$ yields
\[
    \dot V
    \leq
    \Vert \mathbf{d}\Vert\Vert\mathbf{y}\Vert
    -
    \gamma\Vert\mathbf{y}\Vert^2
    +
    \eta_Q\Vert\mathbf{y}\Vert^3.
\]
The following finite-radius soft certificate requires a strictly positive
margin and a nonzero defect, so assume $\gamma>0$ and $\eta_Q>0$.
Writing $R=\Vert\mathbf{y}\Vert$, the estimate above reads
$\dot V\leq R\,(\Vert\mathbf{d}\Vert-\gamma R+\eta_Q R^2)$, in which, for $\dot V<0$, the
dissipation $-\gamma R$ must overcome both the constant forcing
$\Vert\mathbf{d}\Vert$ and the defect term $\eta_Q R^2$. We allot half of
the dissipation to each. First, the defect term uses at most half of the
dissipation, $\eta_Q R^2\leq\frac{\gamma}{2}R$, provided
\[
    R\leq\frac{\gamma}{2\eta_Q}.
\]
Under this condition,
$\dot V\leq R\,(\Vert\mathbf{d}\Vert-\frac{\gamma}{2}R)$, and the
remaining half of the dissipation overcomes the constant forcing,
giving $\dot V<0$, provided
\[
    R>\frac{2\Vert\mathbf{d}\Vert}{\gamma}.
\]
Both conditions hold for every $R$ in the interval
\[
    \left(
    \frac{2\Vert\mathbf{d}\Vert}{\gamma},
    \frac{\gamma}{2\eta_Q}
    \right],
\]
which is nonempty if and only if $4\Vert\mathbf{d}\Vert\eta_Q<\gamma^2$.
Hence, for any $R_{\mathrm{cert}}$ in this interval, the boundary
$\Vert\mathbf{y}\Vert=R_{\mathrm{cert}}$ satisfies
\[
    \dot V
    \leq
    \Vert\mathbf{d}\Vert R_{\mathrm{cert}}
    -
    \frac{\gamma}{2}R_{\mathrm{cert}}^2
    <0.
\]
Hence the ball $\{\Vert\alpha-\mathbf{m}\Vert\leq R_{\mathrm{cert}}\}$ is forward invariant. Moreover, inside this ball, $R_{\mathrm{cert}}\le\gamma/(2\eta_Q)$ gives
$\eta_Q\Vert\mathbf{y}\Vert^3\le\frac{\gamma}{2}\Vert\mathbf{y}\Vert^2$, so
$\dot V\le\Vert\mathbf{d}\Vert\Vert\mathbf{y}\Vert-\frac{\gamma}{2}\Vert\mathbf{y}\Vert^2$.
This is \cref{eq:app_hard_lyap} with $\gamma$ replaced by $\gamma/2$, and the
argument of the hard case yields the practical estimate
\[
    \limsup_{t\to\infty}\Vert\alpha(t)-\mathbf{m}\Vert
    \leq
    \frac{2\Vert\mathbf{d}\Vert}{\gamma}
\]
for trajectories starting in the certified region.  When $\eta_Q=0$, the upper endpoint involving $\gamma/(2\eta_Q)$ is not used; the conservation defect vanishes and the hard certificate is recovered.

\section{Decoder Boundedness and Lipschitz Transfer}
\label{apdx:decoder_transfer}

We prove \Cref{thm:decoder_transfer} and the constants in the subsequent
remark. By \cref{eq:fullfield}, the decoder is applied separately to
each component code $z\in\mathbb{R}^p$.

\subsection{Affine Structure}
Up to a relabeling of layer indices, the FiLM-modulated MFN of
Appendices~\ref{apdx:mfn} and~\ref{apdx:film} has hidden states
\[
    h^{(0)}(x)=x,
    \qquad
    h^{(\ell+1)}(x,z)
    =
    \big(W_\ell h^{(\ell)}(x,z)+b_\ell+A^{(\ell)}z\big)\circ g_\ell(x),
\]
for $\ell=0,\ldots,L-1$, and output
$D_{\Phi_\theta(z)}(x)=W_Lh^{(L)}(x,z)+b_L$. Each layer applies only
affine operations to $(h^{(\ell)},z)$, followed by multiplication with
$g_\ell(x)$, which does not depend on $z$. Hence, if
$h^{(\ell)}(x,z)=p_\ell(x)+P_\ell(x)z$, then
\[
    h^{(\ell+1)}(x,z)
    =
    \operatorname{diag}(g_\ell(x))\big(W_\ell p_\ell(x)+b_\ell\big)
    +
    \operatorname{diag}(g_\ell(x))\big(W_\ell P_\ell(x)+A^{(\ell)}\big)z,
\]
which has the same form. Starting from $p_0(x)=x$ and $P_0=0$, induction
gives $D_{\Phi_\theta(z)}(x)=c_0(x)+J_0(x)z$ with $c_0=W_Lp_L+b_L$ and
$J_0=W_LP_L\in\mathbb{R}^{1\times p}$. Both are continuous in $x$, since
they are built from $x$ and the sinusoidal filters by sums and products,
and are therefore bounded on the compact set $\Omega$. Assembling the
$c$ components gives
\[
    D_{\Phi_\theta(\alpha)}(x)
    =
    c_0(x)\mathbf{1}_c+(\mathbf{I}_c\otimes J_0(x))\alpha .
\]

\subsection{Growth and Lipschitz Bounds}
By the triangle inequality,
$\Vert D_{\Phi_\theta(\alpha)}(x)\Vert\le\sqrt{c}\,|c_0(x)|+\Vert J_0(x)\Vert\Vert\alpha\Vert$,
which gives \cref{eq:decoder_growth_method} with
$C_D=\sqrt{c}\,\sup_{x\in\Omega}|c_0(x)|$ and
$D_D=\sup_{x\in\Omega}\Vert J_0(x)\Vert$. Write $\alpha-\beta$ as the
stack of $\Delta z^{(1)},\ldots,\Delta z^{(c)}$ and let
$G=\int_\Omega J_0(x)^\top J_0(x)\,dx$. Then
\[
    \Vert D_{\Phi_\theta(\alpha)}-D_{\Phi_\theta(\beta)}\Vert_{L^2(\Omega)}^2
    =
    \sum_{q=1}^c\Delta z^{(q)\top}G\,\Delta z^{(q)}
    \le
    \lambda_{\max}(G)\Vert\alpha-\beta\Vert^2,
\]
with equality when every $\Delta z^{(q)}$ is an eigenvector of $G$ for
$\lambda_{\max}(G)$. This gives \cref{eq:decoder_lipschitz_transfer_method}
with the attained constant $L_D=\lambda_{\max}(G)^{1/2}$. Since
$J_0^\top J_0\preceq\Vert J_0\Vert^2\mathbf{I}$ pointwise,
$G\preceq|\Omega|D_D^2\mathbf{I}$ and $L_D\le|\Omega|^{1/2}D_D$.

\subsection{Bound on the Decoded Field}
If $\Vert\alpha_t-\mathbf m\Vert\le R$ for all $t\in[0,T]$, then
$\Vert\alpha_t\Vert\le\Vert\mathbf m\Vert+R$, and
\cref{eq:decoder_growth_method} together with
$\Vert f\Vert_{L^2(\Omega)}\le|\Omega|^{1/2}\sup_{x\in\Omega}\Vert f(x)\Vert$
gives \cref{eq:decoded_field_bound_method}. The three values of $R$
stated after the theorem follow from \cref{eq:app_comparison}: if
$\mathbf d=\mathbf 0$, its right-hand side is at most
$\Vert\mathbf y(0)\Vert$; if $\gamma>0$, it is a convex combination of
$\Vert\mathbf y(0)\Vert$ and $\Vert\mathbf d\Vert/\gamma$, hence at most
their maximum; and if $\gamma=0$, it is at most
$\Vert\mathbf y(0)\Vert+\Vert\mathbf d\Vert T$.

\section{Sparse-Measurement Latent Recovery}
\label{apdx:sparse_recovery}

This section formalizes the test-time inference step described in \Cref{sec:method}, in which the initial latent state is inferred from a possibly sparse first-snapshot measurement by autodecoding (\cref{eq:inr-loss-1}). We derive a stability estimate for that inverse problem as the theoretical counterpart of the sparse-observation ablation in \Cref{sec:results_sparse_mask_ablation}.

Let $\mathcal{S}=\{x_1,\ldots,x_s\}\subset\Omega$ be the observed sensor locations of the initial snapshot, and define the sparse measurement operator and the latent-to-sensor map
\begin{equation}
    P_{\mathcal S}u=
    \begin{bmatrix}
    u(x_1)\\ \vdots\\ u(x_s)
    \end{bmatrix},
    \qquad
    G_{\mathcal S}(\alpha)=P_{\mathcal S}D_{\Phi_\theta(\alpha)} .
\end{equation}
Restricting the autodecoding objective in \cref{eq:inr-loss-1} to $\mathcal S$ seeks a latent code whose decoded sensor values match the measurement. Although the implementation minimizes the squared residual, the slack variable $\delta_{\mathrm{opt}}$ below is defined directly in terms of the residual norm for this estimate.

\begin{theorem}[Sparse-measurement latent recovery]
\label{thm:sparse_recovery}
Let the initial measurement be $y_{\mathrm{obs}}=P_{\mathcal S}u_0+\xi$ with noise $\Vert\xi\Vert\le\sigma$. Assume:
\begin{enumerate}[label=(\roman*)]
    \item \label{assump:observability}
    (Restricted observability) there exist a latent set $\mathcal B\subset\mathbb{R}^r$ and a constant $c_{\mathcal S}>0$ such that $c_{\mathcal S}\Vert\alpha-\beta\Vert\le\Vert G_{\mathcal S}(\alpha)-G_{\mathcal S}(\beta)\Vert$ for all $\alpha,\beta\in\mathcal B$;

    \item \label{assump:representable}
    (Representable initial state) there exists $\alpha_0\in\mathcal B$ with $\Vert G_{\mathcal S}(\alpha_0)-P_{\mathcal S}u_0\Vert\le\varepsilon_{\mathrm{sens}}$;

    \item \label{assump:nearoptimal}
    (Near-optimal fit) the inferred code $\widehat\alpha_0\in\mathcal B$ satisfies $\Vert G_{\mathcal S}(\widehat\alpha_0)-y_{\mathrm{obs}}\Vert\le\inf_{\alpha\in\mathcal B}\Vert G_{\mathcal S}(\alpha)-y_{\mathrm{obs}}\Vert+\delta_{\mathrm{opt}}$.
\end{enumerate}
Then
\begin{equation}
\label{eq:sparse_recovery_bound}
    \Vert\widehat\alpha_0-\alpha_0\Vert
    \le
    \frac{2(\varepsilon_{\mathrm{sens}}+\sigma)+\delta_{\mathrm{opt}}}{c_{\mathcal S}} .
\end{equation}
\end{theorem}

\begin{proof}
Since $\alpha_0\in\mathcal B$, the infimum in assumption~\ref{assump:nearoptimal}
is no larger than its value at $\alpha_0$.
By the triangle inequality and assumption~\ref{assump:representable},
\[
    \Vert G_{\mathcal S}(\alpha_0)-y_{\mathrm{obs}}\Vert
    \le \Vert G_{\mathcal S}(\alpha_0)-P_{\mathcal S}u_0\Vert+\Vert\xi\Vert
    \le \varepsilon_{\mathrm{sens}}+\sigma .
\]
Combining with assumption~\ref{assump:nearoptimal} gives $\Vert G_{\mathcal S}(\widehat\alpha_0)-y_{\mathrm{obs}}\Vert\le\varepsilon_{\mathrm{sens}}+\sigma+\delta_{\mathrm{opt}}$. A second triangle inequality yields
\[
    \Vert G_{\mathcal S}(\widehat\alpha_0)-G_{\mathcal S}(\alpha_0)\Vert
    \le \Vert G_{\mathcal S}(\widehat\alpha_0)-y_{\mathrm{obs}}\Vert
    +\Vert y_{\mathrm{obs}}-G_{\mathcal S}(\alpha_0)\Vert
    \le 2(\varepsilon_{\mathrm{sens}}+\sigma)+\delta_{\mathrm{opt}} .
\]
The restricted-observability assumption~\ref{assump:observability} gives $c_{\mathcal S}\Vert\widehat\alpha_0-\alpha_0\Vert\le\Vert G_{\mathcal S}(\widehat\alpha_0)-G_{\mathcal S}(\alpha_0)\Vert$; dividing by $c_{\mathcal S}>0$ proves \cref{eq:sparse_recovery_bound}.
\end{proof}

\begin{remark}
The constant $c_{\mathcal S}$ quantifies how well the sparse sensors identify the latent code; it is an assumption on the decoder--sensor pair rather than a consequence of \Cref{thm:1}. Since $G_{\mathcal S}(\alpha)-G_{\mathcal S}(\beta)=J_{\mathcal S}(\alpha-\beta)$, the map $G_{\mathcal S}$ is bi-Lipschitz with constants $\sigma_{\min}(J_{\mathcal S})$ and $\sigma_{\max}(J_{\mathcal S})$, so the recovered latent,
and hence the reconstructed field, degrades continuously as the noise
$\sigma$, representation error $\varepsilon_{\mathrm{sens}}$, and optimization
slack $\delta_{\mathrm{opt}}$ grow. Reducing the visible fraction of the initial snapshot can shrink $c_{\mathcal S}$, although the exact change depends on the sensor locations, the mask construction, and the trained decoder. The empirical transition observed around $2\%$--$4\%$ visibility in
\Cref{tab:mask_overall_test,tab:app_mask_overall_test_full} is consistent with
this interpretation, since observability remains adequate for moderately
sparse masks but eventually becomes insufficient when too few spatial samples are
available. \Cref{tab:app_mask_observability} confirms that $c_{\mathcal S}$ vanishes exactly at the visibility levels where $s<p$.
\end{remark}

\subsection{Proof of the Sparse-to-Rollout Full-Field Certificate}
\label{apdx:sparse_to_rollout_proof}

We prove \Cref{thm:sparse_to_rollout}. By \Cref{thm:sparse_recovery} and the definition of $e_{\mathrm{sparse}}$ in \cref{eq:e_sparse_method}, the inferred initial latent state satisfies
\begin{equation}
\label{eq:apdx_initial_sparse_error}
    E_0
    =
    \Vert\widehat\alpha_0-\alpha_0^\star\Vert
    \le
    e_{\mathrm{sparse}} .
\end{equation}
The rollout recurrence in \cref{eq:sparse_rollout_recursion_method} then implies, for every $0\le n\le N$ over the finite horizon under consideration,
\begin{equation}
\label{eq:apdx_discrete_rollout_bound}
    E_n
    \le
    a^n e_{\mathrm{sparse}}
    +
    b_{\mathrm{roll}}\sum_{\ell=0}^{n-1}a^\ell .
\end{equation}
Indeed, the claim is immediate for $n=0$. If it holds at $n$, then using $E_{n+1}\le aE_n+b_{\mathrm{roll}}$ and $a\ge0$ gives
\[
    E_{n+1}
    \le
    a^{n+1}e_{\mathrm{sparse}}
    +
    b_{\mathrm{roll}}\sum_{\ell=1}^{n}a^\ell
    +
    b_{\mathrm{roll}}
    =
    a^{n+1}e_{\mathrm{sparse}}
    +
    b_{\mathrm{roll}}\sum_{\ell=0}^{n}a^\ell ,
\]
which proves \cref{eq:apdx_discrete_rollout_bound} by induction.

For the decoded field, the triangle inequality gives
\begin{align}
    \left\Vert D_{\Phi_\theta(\widehat\alpha_n)}-u_n\right\Vert_{L^2(\Omega)}
    &\le
    \left\Vert
    D_{\Phi_\theta(\widehat\alpha_n)}
    -
    D_{\Phi_\theta(\alpha_n^\star)}
    \right\Vert_{L^2(\Omega)}
    +
    \left\Vert D_{\Phi_\theta(\alpha_n^\star)}-u_n\right\Vert_{L^2(\Omega)}  \\
    &\le
    L_D E_n+\varepsilon_{\mathrm{dec}} .
\end{align}
Substituting \cref{eq:apdx_discrete_rollout_bound} yields \cref{eq:sparse_to_rollout_bound_method} on the finite horizon. If the recurrence and decoder bounds hold uniformly for all $n\ge0$ and $0\le a<1$, then $a^n e_{\mathrm{sparse}}\to0$ and $\sum_{\ell=0}^{n-1}a^\ell\to(1-a)^{-1}$, which gives the asymptotic bound \cref{eq:sparse_to_rollout_limsup_method}.

\section{Details of Scheduled Sampling}
\label{apdx:scheduled_sampling}

Scheduled sampling~\cite{bengio2015scheduled} was introduced to mitigate exposure bias in long-sequence modeling, which arises from the accumulation of prediction errors when only the initial condition is given in real-world scenarios (see \cref{fig:apdx_scheduled_sampling}). 
Interestingly, exposure bias closely resembles the accumulated error problem in learning closure models for dynamical systems~\cite{pan2018_ddc}, where models trained using only one-step minimization (i.e., teacher forcing) often fail to generalize effectively.
The core idea of scheduled sampling is to gradually reduce reliance on teacher forcing during training.
In \Cref{alg:scheduled_sample_node}, $N_{\mathrm{train}}$ denotes the length of the
supplied time array. During training, this routine is called with
$N_{\mathrm{train}}=10$, so it returns the
initial state and nine subsequent states. Evaluation uses the initial
state and 19 subsequent prediction times, without teacher-forcing resets.
In our neural-ODE implementation, teacher forcing is applied by randomly selecting intermediate time points as restart anchors. The latent trajectory is split into integration segments, and each selected anchor starts the next segment from the corresponding target latent state rather than from the model prediction. 
We define the probability of selecting a time point as a teacher-forcing anchor as
$\varepsilon_t = \varepsilon_0 \rho_{\mathrm{TF}}^i$,
where $i$ represents the number of training epochs elapsed and $\rho_{\mathrm{TF}}$ is the
per-epoch decay factor. In our experiments, we set
$\varepsilon_0=0.99$ and $\rho_{\mathrm{TF}}=0.99$, following an approach similar to that used in autoregressive models. Finally, the forward propagation of the neural ODE with scheduled sampling is summarized in \Cref{alg:scheduled_sample_node}.

\begin{figure}[ht]
\centering
\includegraphics[width=12cm]{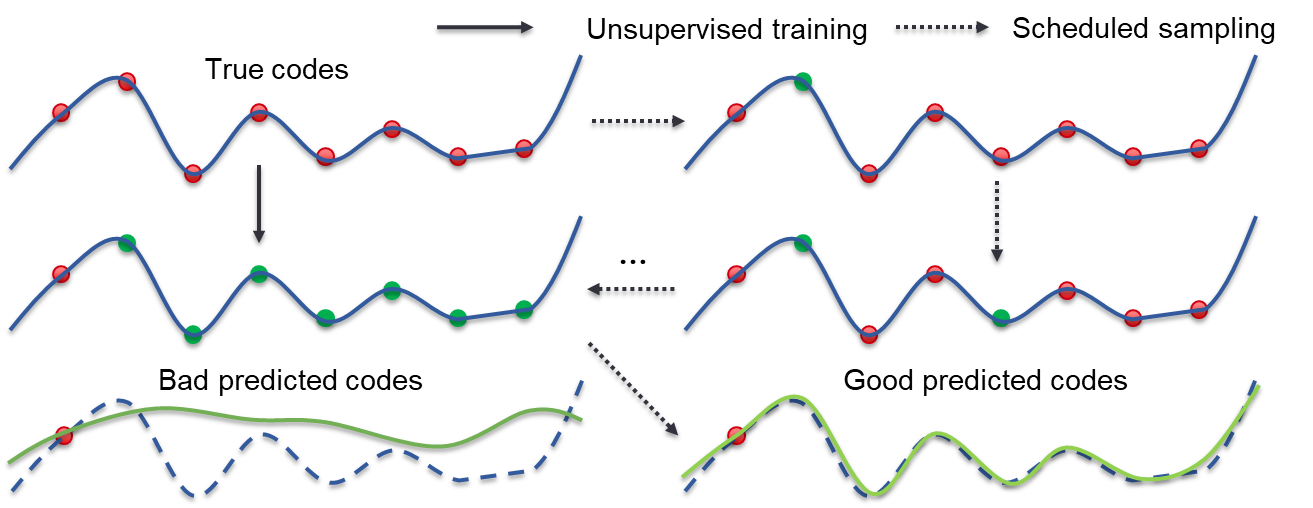}
\caption{Scheduled sampling for latent-dynamics training.
The upper-left panel shows target latent codes at supervised time
stamps. The intermediate panels illustrate the progressive replacement
of target codes by model-generated states during training.
Intermediate target states are used as teacher-forcing anchors with
probability $\varepsilon_t$; as this probability decays, training
increasingly relies on free-running latent rollouts.
The bottom panels schematically contrast a poorly matched rollout
without intermediate anchors (left) with an improved rollout obtained
using scheduled sampling (right).}
\label{fig:apdx_scheduled_sampling}
\end{figure}
\FloatBarrier

\begin{algorithm}[ht]
\caption{Forward propagation of the neural ODE with scheduled sampling}
\label{alg:scheduled_sample_node}
\begin{algorithmic}
\Function {scheduled\_sampling\_ode\_solve}{$f_\psi$, $\{\alpha_{t_l}\}_{l=1}^{N_{\mathrm{train}}}$, $\mathcal{T}_{N_{\mathrm{train}}}$, $\varepsilon_t$}

\State $N_{\mathrm{train}}\leftarrow |\mathcal{T}_{N_{\mathrm{train}}}|$
\If{$\varepsilon_t < 10^{-3}$}
\Comment{Classic full-length rollout is used for forward propagation}
\State \Return $\{\alpha(t_l)\}_{l=1}^{N_{\mathrm{train}}} \leftarrow \texttt{ODEInt}(f_\psi, \alpha_{t_1},t_1,\ldots,t_{N_{\mathrm{train}}})$

\Else
\State Generate Boolean variables $\{m_i\}_{i=2}^{N_{\mathrm{train}}-1}$, where each $m_i$ is True with probability $\varepsilon_t$. \(\varepsilon_t\) of being selected as a teacher-forcing anchor
\State Force the very last time point to \texttt{False} so the model must predict it
\State Discard the first entry of the mask
\State Set the start index $i_s \gets 1$ to mark the start index of each integration trajectory
\State $b \gets [\alpha_{t_1}]$
\Comment{store the initial anchor so the returned sequence has length $N_{\mathrm{train}}$}
\For{\(i=2\) to \(N_{\mathrm{train}}-1\)}
\If{\(m_i \text{ is } \texttt{True}\)}
\Comment{Restart the rollout from a teacher-forcing anchor.}
\State Determine a time sub-interval from \(t_{i_s}\) up to \(t_i\).
\State $\{\alpha(t_l)\}_{l=i_s}^{i} \leftarrow \texttt{ODEInt}(f_\psi, \alpha_{t_{i_s}},t_{i_s},\ldots,t_{i})$
\State Append $b$ with the integrated outputs of $\{\alpha(t_l)\}_{l=i_s+1}^{i}$
\State $i_s \gets i$
\EndIf
\EndFor
\State $\{\alpha(t_l)\}_{l=i_s}^{N_{\mathrm{train}}} \leftarrow \texttt{ODEInt}(f_\psi, \alpha_{t_{i_s}},t_{i_s},\ldots,t_{N_{\mathrm{train}}})$
\Comment{Numerical integration of last interval}
\State Append $b$ with the integrated outputs of $\{\alpha(t_l)\}_{l=i_s+1}^{N_{\mathrm{train}}}$
\State \Return $b$
\Comment{$b=\{\alpha(t_l)\}_{l=1}^{N_{\mathrm{train}}}$ includes the initial state}
\EndIf
\EndFunction
\end{algorithmic}
\end{algorithm}

\FloatBarrier

\newpage
\section{Experimental Details and Results}
\label{apdx:addition}

\subsection{Dataset Details}
\label{apdx:dataset_details}

\subsubsection{2D Wave}
The 2D Wave dataset follows the second-order wave equation used in
\cite{Yin2022Continuous}:
\begin{equation}
    \frac{\partial^2 u}{\partial t^2}
    =
    v^2
    \left(
    \frac{\partial^2 u}{\partial x^2}
    +
    \frac{\partial^2 u}{\partial y^2}
    \right),
\end{equation}
where $v$ is the wave speed. We set $v=2$ and define the spatial domain as
$\Omega=[-1,1]^2$. The initial displacement is sampled as a Gaussian pulse:
\begin{equation}
    u_0(x,y;\,a,\mathbf{b},\sigma)
    =
    a
    \exp
    \left(
    -\frac{(x-b_x)^2+(y-b_y)^2}{2\sigma^2}
    \right),
\end{equation}
where $a\sim U(2,4)$ is the peak displacement amplitude,
$\mathbf{b}=(b_x,b_y)$ is the peak location sampled uniformly from $\Omega$, and
$\sigma\sim U(0.25,0.30)$ is the standard deviation. The initial time derivative is
set to zero:
\begin{equation}
    \frac{\partial u}{\partial t}(0,x,y)=0.
\end{equation}
The benchmark prediction target is a two-channel wave field, denoted by $u_x$ and $u_y$ following the dataset convention. These channels are treated as the two predicted components of the surrogate model; the notation does not introduce additional spatial derivatives in our method.

We generate 256 trajectories for training and 16 trajectories for testing. Each
trajectory is partitioned into two non-overlapping sequences, yielding 512 training
sequences and 32 test sequences in total. Each sequence contains 20 snapshots sampled
every $\Delta t=0.25\,\mathrm{s}$, and each snapshot is represented on a uniform
$64^2$ spatial grid.

\subsubsection{2D Navier--Stokes}
The 2D Navier--Stokes dataset is generated from the pressure-free vorticity
formulation of the forced incompressible Navier--Stokes equations using the
streamfunction-velocity approach~\cite{li2020fourier}. The vorticity field evolves as
\begin{equation}
    \dot{w}(t,x,y)
    +
    v(t,x,y)\cdot\nabla w(t,x,y)
    =
    \nu\nabla^2 w(t,x,y)
    +
    f(x,y),
\end{equation}
where $\nu$ is the viscosity, $v(t,x,y)$ denotes the velocity field, and
$w(t,x,y)=\nabla\times v(t,x,y)$ is the vorticity. The vorticity field is used as the
prediction target. The forcing term is fixed as
\begin{equation}
    f(x,y)
    =
    0.1
    \left[
    \sin(2\pi(x+y))
    +
    \cos(2\pi(x+y))
    \right].
\end{equation}

The spatial domain is $\Omega=[-1,1]^2$. We set $\nu=10^{-3}$, impose periodic
boundary conditions, and use the same initial-condition sampling strategy as
\cite{li2020fourier}. We generate 256 trajectories for training and 16 trajectories
for testing. Each trajectory is partitioned into two non-overlapping sequences,
yielding 512 training sequences and 32 test sequences in total. Each sequence
contains 20 snapshots sampled every $\Delta t=1\,\mathrm{s}$, and each snapshot is
represented on a uniform $64^2$ spatial grid.

\subsubsection{3D Spherical Shallow Water}
\label{app:3dswdataset}
The 3D Spherical Shallow Water dataset follows the setup of
\cite{galewsky2004initial,Yin2022Continuous}. The governing equations are
\begin{equation}
    \begin{split}
        \frac{Du}{Dt}
        &=
        -f\,k\times u
        -
        g\nabla h
        +
        \nu\Delta u,\\
        \frac{Dh}{Dt}
        &=
        -h\nabla\cdot u
        +
        \nu\Delta h,
    \end{split}
\end{equation}
where $u$ is the velocity field tangent to the spherical surface, $h$ is the fluid
layer thickness, $k$ is the unit vector normal to the sphere, and $f$, $g$, and
$\nu$ denote the Coriolis parameter, gravitational acceleration, and viscosity,
respectively.

The initial zonal velocity is prescribed as
\begin{equation}
    u_0(\varphi,\theta)
    =
    \begin{cases}
    \left(
    \dfrac{u_{\max}}{e_n}
    \exp
    \left(
    \dfrac{1}{(\phi-\phi_0)(\phi-\phi_1)}
    \right),
    0
    \right),
    & \phi\in(\phi_0,\phi_1),\\[1ex]
    \left(
    \dfrac{u_{\max}}{e_n}
    \exp
    \left(
    \dfrac{1}{(\phi+\phi_0)(\phi+\phi_1)}
    \right),
    0
    \right),
    & \phi\in(-\phi_1,-\phi_0),\\[1ex]
    (0,0),
    & \text{otherwise},
    \end{cases}
\end{equation}
where
\begin{equation}
    e_n
    =
    \exp
    \left(
    -\frac{4}{(\phi_1-\phi_0)^2}
    \right).
\end{equation}
The initial water-height perturbation is defined as
\begin{equation}
    h_0'(\phi,\theta)
    =
    \hat{h}\cos(\phi)
    \exp
    \left(
    -\left(\frac{\theta}{\alpha_h}\right)^2
    \right)
    \left[
    \exp
    \left(
    -\left(\frac{\phi_2-\phi}{\beta}\right)^2
    \right)
    +
    \exp
    \left(
    -\left(\frac{\phi_2+\phi}{\beta}\right)^2
    \right)
    \right].
\end{equation}

Simulations are performed using Dedalus~\cite{burns2020dedalus} on a
latitude-longitude grid of size $128\times256$. The physical parameters follow
\cite{galewsky2004initial}. The initial-condition parameters are set to
\[
    \phi_0=\frac{\pi}{7},\quad
    \phi_1=\frac{\pi}{2}-\phi_0,\quad
    \phi_2=\frac{\pi}{4},\quad
    \hat{h}=120\,\mathrm{m},\quad \alpha_h=\frac{1}{3},\quad \beta=\frac{1}{15}.
\]
Trajectory-level variation is introduced by sampling
$u_{\max}\sim U(60,80)$. The model predicts two fields, the fluid layer thickness $h$ and the vorticity
field $w=\nabla\times u$.

The simulated $128\times256$ latitude--longitude grid is uniformly
subsampled by a factor of two in both spatial directions, yielding a
$64\times128$ grid. Each simulation records one snapshot per
hour until the final time of 320 hours. The first 160 h correspond to the initial geostrophic adjustment and jet spin-up, after which the flow settles into a quasi-periodic regime. To remove the transient spin-up phase, we
use the last 160 hours of each trajectory and divide them into eight consecutive
non-overlapping sequences of length 20. The height field $h$ is multiplied by
$3\times10^3$, and the vorticity field $w$ is scaled by 2. In total, eight long
trajectories are used for training, yielding 64 training sequences, and two long
trajectories are used for testing, yielding 16 test sequences.

\newpage
\subsection{Hyper-Parameter Selection and Computational Environment}
\label{app:Hyper_parameters}
The nominal training budget for the main model comparison is 5,000 epochs.
The stability study uses a separate 10,000-epoch budget, and
configurations that develop non-finite values are identified in its results.
During training, each run is evaluated every 100 epochs on test set. All experiments use three random seeds, selected once and then held fixed. The batch size is dataset-dependent, with 32 for 2D Wave, 32 for 2D
Navier--Stokes, and 4 for 3D Spherical Shallow Water.
The hyperparameters of baseline neural operators, listed in \Cref{tab:app_hparam_candidates}, are varied to match parameter budgets. The CNO implementation requires
square input grids. For Spherical Shallow Water, we therefore
downsample the $64\times128$ fields to $64\times64$ before
applying CNO. 

\begin{table}[!htbp]
  \centering
  \caption{Architecture hyper-parameters used in the final comparisons.
  The abbreviations \texttt{wa}, \texttt{ns}, and \texttt{sw} denote 2D Wave,
  2D Navier--Stokes, and Spherical Shallow Water, and ``--'' marks a parameter
  that does not apply.
  For the baseline operators, FNO is specified by $(M,W,L)$, where $M$ is the
  number of retained Fourier modes per dimension, $W$ the channel width, and
  $L$ the number of Fourier layers;
  CNO by $(L,R,RN,CM)$, where $L$ is the number of up/down-sampling layers,
  $R$ the number of residual blocks per level, $RN$ the number of residual
  blocks in the bottleneck, and $CM$ the channel multiplier;
  and Transolver by $(L,H,S)$, where $L$ is the number of layers, $H$ the
  hidden dimension, and $S$ the number of physics-attention slices.
  For the latent models, in which MLP denotes the DINo dynamics,
  $p$ is the per-component latent dimension and $r=cp$ the total latent
  dimension, $\mathrm{hdyn}$ is the hidden width of the MLP latent-dynamics
  network, $\mathrm{hdec}$ is the hidden width of the INR (MFN) decoder,
  $\ell_{\mathrm{dec}}$ is the number of hidden linear layers of the decoder, and $\mathrm{qr}$ is the
  quadratic rank $\kappa$ of the low-rank factorization used by SMORE-LRLQ.}
  \small
  \setlength{\tabcolsep}{5pt}
  \begin{tabular}{llrrrrrr}
    \toprule
    \multicolumn{8}{l}{\textit{Baseline models}} \\
    \midrule
    Dataset & Model & \multicolumn{6}{l}{Architecture} \\
    \midrule
    \texttt{wa} & FNO        & \multicolumn{6}{l}{$(M,W,L)=(12,10,4)$} \\
                & CNO        & \multicolumn{6}{l}{$(L,R,RN,CM)=(3,4,6,6)$} \\
                & Transolver & \multicolumn{6}{l}{$(L,H,S)=(5,33,64)$} \\
    \midrule
    \texttt{ns} & FNO        & \multicolumn{6}{l}{$(M,W,L)=(12,17,4)$} \\
                & CNO        & \multicolumn{6}{l}{$(L,R,RN,CM)=(3,1,8,10)$} \\
                & Transolver & \multicolumn{6}{l}{$(L,H,S)=(8,46,32)$} \\
    \midrule
    \texttt{sw} & FNO        & \multicolumn{6}{l}{$(M,W,L)=(12,25,4)$} \\
                & CNO        & \multicolumn{6}{l}{$(L,R,RN,CM)=(3,3,8,10)$} \\
                & Transolver & \multicolumn{6}{l}{$(L,H,S)=(8,46,64)$} \\
    \midrule
    \multicolumn{8}{l}{\textit{Latent models}} \\
    \midrule
    Dataset & Model & $p$ & $r$ & $\mathrm{hdyn}$ & $\mathrm{hdec}$ & $\ell_{\mathrm{dec}}$ & $\mathrm{qr}$ \\
    \midrule
    \texttt{wa} & MLP            & 100 & 200 & 146 & 46  & 3 & -- \\
                & SMORE-Koopman & 100 & 200 & --  & 116 & 3 & -- \\
                & SMORE-LRLQ           & 40  & 80  & --  & 112 & 3 & 5  \\
    \midrule
    \texttt{ns} & MLP            & 72  & 72  & 368 & 64  & 3 & -- \\
                & SMORE-Koopman & 128 & 128 & --  & 260 & 3 & -- \\
                & SMORE-LRLQ           & 72  & 72  & --  & 160 & 4 & 18 \\
    \midrule
    \texttt{sw} & MLP            & 100 & 200 & 290 & 100 & 3 & -- \\
                & SMORE-Koopman & 100 & 200 & --  & 262 & 3 & -- \\
                & SMORE-LRLQ           & 60  & 120 & --  & 96  & 3 & 10 \\
    \bottomrule
  \end{tabular}
  \label{tab:app_hparam_candidates}
\end{table}

\FloatBarrier

\subsection{Experimental Environment}
\label{app:Ex_device}
Experiments were conducted primarily on the Perlmutter supercomputer at the
National Energy Research Scientific Computing Center (NERSC), a U.S.
Department of Energy Office of Science user facility at Lawrence Berkeley
National Laboratory.
The NVIDIA GPU nodes used in these experiments were equipped with A100 GPUs
with either 40 or 80\,GB of device memory.
The stability experiments in
\Cref{sec:results_stability_robustness} were instead performed on the
Frontier supercomputer at the Oak Ridge Leadership Computing Facility (OLCF),
a U.S. Department of Energy Office of Science user facility at Oak Ridge
National Laboratory (ORNL), using AMD Instinct MI250X GPUs.
The implementation uses PyTorch with CUDA on the NVIDIA-based Perlmutter
system and ROCm on the AMD-based Frontier system.

\subsection{Detailed In-horizon and Extrapolation Statistics}
\label{app:rollout_decomposition}

The main comparison in \Cref{tab:combined_loss_comparison} reports mean
in-horizon and extrapolation errors for compactness.
\Cref{tab:app_rollout_decomposition} lists the corresponding
mean $\pm$ standard deviation values over three random seeds.
\Cref{fig:wave_per_frame} further shows the per-snapshot test MSE of
SMORE-LRLQ on Wave.

\begin{table}[!htbp]
  \centering
  \caption{Detailed in-horizon and extrapolation MSE on unseen test trajectories. Values are mean $\pm$ sample standard deviation over three random seeds. DINo, SMORE-Koopman, and SMORE-LRLQ use results after 1,000 initial latent optimization steps.}
  \small
  \setlength{\tabcolsep}{4pt}
  \begin{tabular}{llrcc}
    \toprule
    Dataset & Model & Parameters & In-horizon loss & Extrapolation loss \\
    \midrule
    Wave & FNO & 117,356 & $2.89e-04 \pm 1.09e-04$ & $7.37e-02 \pm 1.41e-02$ \\
    Wave & CNO & 111,996 & $1.99e-02 \pm 4.34e-03$ & $1.53e-01 \pm 1.47e-02$ \\
    Wave & Transolver & 128,328 & $7.82e-04 \pm 3.90e-04$ & $1.67e-02 \pm 4.29e-03$ \\
    Wave & DINo & 126,928 & $1.13e-03 \pm 9.98e-05$ & $1.38e-02 \pm 1.01e-03$ \\
    Wave & SMORE-Koopman & 128,045 & $6.74e-05 \pm 3.63e-05$ & $3.22e-04 \pm 1.54e-04$ \\
    Wave & SMORE-LRLQ & 127,265 & $1.25e-03 \pm 8.22e-07$ & $8.65e-04 \pm 3.09e-06$ \\
    \midrule
    Navier--Stokes & FNO & 336,653 & $1.61e-03 \pm 2.12e-03$ & $2.43e-03 \pm 3.03e-03$ \\
    Navier--Stokes & CNO & 326,184 & $1.52e-03 \pm 3.71e-04$ & $3.87e-03 \pm 1.15e-03$ \\
    Navier--Stokes & Transolver & 358,007 & $4.13e-01 \pm 1.04e-01$ & $6.22e-01 \pm 2.12e-01$ \\
    Navier--Stokes & DINo & 356,444 & $1.34e-03 \pm 1.33e-04$ & $2.32e-03 \pm 7.72e-05$ \\
    Navier--Stokes & SMORE-Koopman & 355,165 & $5.37e-02 \pm 2.02e-03$ & $9.78e-02 \pm 2.22e-03$ \\
    Navier--Stokes & SMORE-LRLQ & 353,961 & $1.01e-03 \pm 2.54e-05$ & $2.00e-03 \pm 1.99e-04$ \\
    \midrule
    Shallow Water & FNO & 366,336 & $1.79e-03 \pm 1.48e-05$ & $5.22e-03 \pm 9.19e-05$ \\
    Shallow Water & CNO & 346,727 & $8.03e-04 \pm 2.69e-05$ & $2.86e-03 \pm 4.95e-06$ \\
    Shallow Water & Transolver & 359,774 & $1.71e-03 \pm 5.88e-04$ & $6.05e-03 \pm 1.47e-03$ \\
    Shallow Water & DINo & 356,674 & $4.90e-04 \pm 1.99e-05$ & $6.49e-04 \pm 1.74e-05$ \\
    Shallow Water & SMORE-Koopman & 354,401 & $2.22e-04 \pm 3.39e-06$ & $5.34e-04 \pm 1.91e-05$ \\
    Shallow Water & SMORE-LRLQ & 354,553 & $5.86e-04 \pm 2.89e-05$ & $7.67e-04 \pm 2.91e-05$ \\
    \bottomrule
  \end{tabular}%
  \label{tab:app_rollout_decomposition}
\end{table}
\begin{figure}[!htbp]
  \centering
  \includegraphics[width=0.65\textwidth,keepaspectratio]{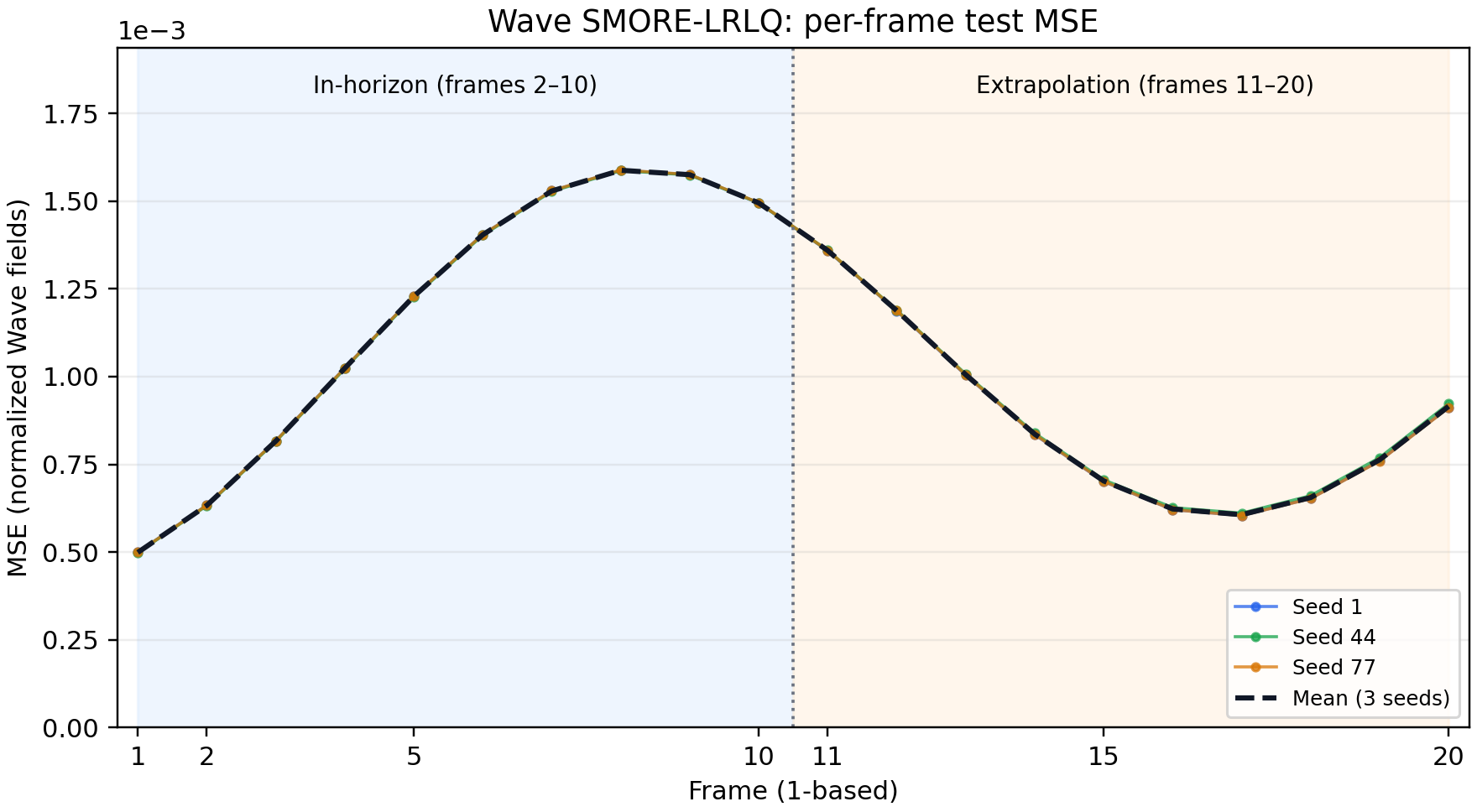}
  \caption{Per-snapshot test MSE of SMORE-LRLQ on the 2D Wave benchmark,
averaged over three random seeds (individual seeds shown as thin lines).
The dashed vertical line separates the in-horizon window (snapshots 2--10)
from the extrapolation window (snapshots 11--20). The error rises to a peak
inside the in-horizon window and decreases over much of the extrapolation
window, so the extrapolation-window average lies below the in-horizon
average even though no actual improvement occurs at longer horizons.}
  \label{fig:wave_per_frame}
\end{figure}

\subsection{Detailed Stability and Mechanism Ablations}
\label{sec:app_stability_mechanism}

This appendix provides detailed numerical results supporting
\Cref{sec:results_stability_robustness}.
All experiments in this diagnostic study use SMORE-LRLQ.
The reported latent size is the per-component latent dimension $p$.
The reported prediction metric is minimum overall test MSE and is shown in units of $10^{-3}$.

In the ablation labels, $L_{\mathrm{reg}}$ denotes the two regularizers enabled or removed jointly. For both Spherical Shallow Water and Navier--Stokes, we use $(\beta_{\mathrm{cons}},\beta_{\mathrm{eig}})=(10^{-8},10^{-7})$.

\begin{table}[!htbp]
  \centering
  \caption{
  First observed NaN epoch in the Spherical Shallow Water mechanism study.
  ``Not observed'' indicates that no NaN was recorded during the full
10,000-epoch training budget.
Remaining finite over this interval is an empirical observation and
does not establish formal stability.
  }
  \small
  \setlength{\tabcolsep}{5pt}
  \begin{tabular}{lccc}
    \toprule
    Configuration & $p=50$ & $p=70$ & $p=100$ \\
    \midrule
    OFF
      & 8111 & 200 & 100 \\
    $L_{\text{reg,cons}}$ only
      & Not observed & 147 & 142 \\
    $L_{\mathrm{reg,eig}}$ only
      & 8100 & 227 & 135 \\
    $L_{\mathrm{reg}}$ only
      & Not observed & 150 & 146 \\
    clip only
      & 7162 & 158 & 100 \\
    low-rank only
      & Not observed & Not observed & Not observed \\
    FULL
      & Not observed & Not observed & Not observed \\
    FULL without $L_{\mathrm{reg}}$
      & Not observed & Not observed & Not observed \\
    FULL without $\mathrm{clip}$
      & Not observed & Not observed & Not observed \\
    FULL without $\mathrm{low\mbox{-}rank}$
      & Not observed & 160 & 137 \\
    \bottomrule
  \end{tabular}
  \label{tab:app_sw_nan_epochs}
\end{table}

\begin{table}[!htbp]
  \centering
  \caption{SMORE-LRLQ stability ablation on 2D Navier--Stokes.
  Values are the total rollout MSE over the 19 predicted snapshots
  (snapshots 2--20) under the 1000-step initial latent fitting during evaluation.
  The configured training budget is 10,000 epochs.
  The low-rank quadratic rank is \(\lfloor p/6 \rfloor\).}
  \small
  \setlength{\tabcolsep}{5pt}
  \begin{tabular}{lccccc}
    \toprule
    Configuration & $p=10$ & $p=30$ & $p=50$ & $p=70$ & $p=100$ \\
    \midrule
    OFF
      & 3.139e-2 & 4.432e-3 & 1.966e-3 & 3.226e-3 & 7.215e-3 \\
    $L_{\text{reg,cons}}$ only
      & 3.137e-2 & 4.496e-3 & 2.921e-3 & 4.141e-3 & 6.708e-3 \\
    $L_{\mathrm{reg,eig}}$ only
      & 3.139e-2 & 4.506e-3 & 1.959e-3 & 3.351e-3 & 5.102e-3 \\
    $L_{\mathrm{reg}}$ only
      & 3.137e-2 & 4.536e-3 & 2.845e-3 & 4.097e-3 & 7.018e-3 \\
    clip only
      & 3.087e-2 & 4.441e-3 & 1.939e-3 & 2.295e-3 & 4.532e-3 \\
    low-rank only
      & 3.960e-2 & 5.335e-3 & 2.635e-3 & 2.887e-3 & 4.161e-3 \\
    FULL
      & 3.897e-2 & 4.970e-3 & 2.003e-3 & 1.923e-3 & 2.337e-3 \\
    FULL without $L_{\mathrm{reg}}$
      & 3.879e-2 & 5.104e-3 & 2.922e-3 & 2.509e-3 & 3.475e-3 \\
    FULL without $\mathrm{clip}$
      & 3.954e-2 & 5.049e-3 & 2.321e-3 & 1.902e-3 & 2.796e-3 \\
    FULL without $\mathrm{low\mbox{-}rank}$
      & 3.120e-2 & 4.438e-3 & 2.658e-3 & 3.271e-3 & 5.191e-3 \\
    \bottomrule
  \end{tabular}
  \label{tab:app_ns_mechanism_sweep}
\end{table}

\begin{table}[htbp]
  \centering
  \caption{
  SMORE-LRLQ stability-mechanism ablation on 3D Spherical Shallow Water.
  Values are the best finite overall MSE over the 19 predicted snapshots,
  reported in units of $10^{-3}$.
  A dagger ($\dagger$) indicates that the run subsequently produced NaN;
  the reported value is the best finite MSE obtained before divergence.
  The low-rank rank is 10.
  All training runs were continued for up to 10,000 epochs unless numerical
  divergence occurred.
  }
  \begin{tabular}{lccc}
    \toprule
    Configuration & $p=50$ & $p=70$ & $p=100$ \\
    \midrule
    OFF
      & 0.8614$^\dagger$ & 2.360$^\dagger$ & 4.069$^\dagger$ \\
    $\lambda_1$ only
      & 0.8574 & 2.757$^\dagger$ & 3.397$^\dagger$ \\
    $\lambda_2$ only
      & 0.8603$^\dagger$ & 2.699$^\dagger$ & 3.427$^\dagger$ \\
    $\lambda$ only
      & 0.8464 & 2.953$^\dagger$ & 3.240$^\dagger$ \\
    clip only
      & 0.8843$^\dagger$ & 2.669$^\dagger$ & 5.313$^\dagger$ \\
    low-rank only
      & 0.9120 & \textbf{0.7832} & 0.7305 \\
    FULL
      & 0.9301 & 0.7845 & 0.7274 \\
    FULL $-\lambda$
      & 0.9330 & 0.7833 & 0.7382 \\
    FULL $-\mathrm{clip}$
      & 0.9088 & \textbf{0.7793} & \textbf{0.7273} \\
    FULL $-\mathrm{low\mbox{-}rank}$
      & 0.8709 & 2.538$^\dagger$ & 5.480$^\dagger$ \\
    \bottomrule
  \end{tabular}
  \label{tab:app_sw_mechanism_sweep}
\end{table}

\FloatBarrier

\FloatBarrier

\subsection{Full Sparse-Initialization Results}
\label{app:mask_results_full}

\Cref{tab:app_mask_overall_test_full} reports the complete
sparse-initialization sweep corresponding to
\Cref{sec:results_sparse_mask_ablation}.
Each row reports overall test MSE as mean $\pm$ standard deviation over three random seeds.
The $100\%$ row is the full-observation reference obtained under the same
evaluation procedure, and the remaining rows evaluate sparse initial-snapshot
observations without retraining.
The same spatial mask-generation procedure is used for all three datasets.
For an evaluation grid of size $H\times W$ and visible ratio $v_{\mathrm{obs}}$,
$\lfloor v_{\mathrm{obs}}HW\rfloor$ spatial locations are sampled uniformly at random
without replacement.
The random realization is controlled by the mask seed.
Wave and Navier--Stokes use $64\times64$ grids, whereas Spherical Shallow
Water uses a $64\times128$ grid.
Only the initial latent state is optimized from the retained observations;
the trained decoder and latent dynamics remain unchanged.

\begin{table}[!htbp]
\centering
\caption{Complete sparse-initialization results using 1,000 initial latent optimization steps with frozen checkpoints. Overall test MSE excludes the fitted initial snapshot and is evaluated on the full spatial grid. Values are mean $\pm$ sample standard deviation across three random seeds. Wave and Navier--Stokes use $64\times64$ grids; Spherical Shallow Water uses $64\times128$. A fixed random permutation of grid locations is generated for each dataset, and nested subsets are taken without replacement using the same procedure. Masks are shared across training seeds and test trajectories. The visible count is $\lfloor v_{\mathrm{obs}}HW\rfloor$ for visible fraction $v_{\mathrm{obs}}$. Standard deviations reflect training-seed variation.}
\small
\setlength{\tabcolsep}{5pt}
\begin{tabular}{rccc}
\toprule
Visible ratio & \makecell{Wave\\(SMORE-Koopman)} & \makecell{Navier--Stokes\\(SMORE-LRLQ)} & \makecell{Shallow Water\\(SMORE-Koopman)} \\
\midrule
$100\%$ & 2.01e-04 $\pm$ 9.84e-05 & 1.53e-03 $\pm$ 1.01e-04 & 3.86e-04 $\pm$ 8.58e-06 \\
$50\%$ & 2.05e-04 $\pm$ 1.04e-04 & 1.53e-03 $\pm$ 1.09e-04 & 3.94e-04 $\pm$ 8.03e-06 \\
$25\%$ & 2.12e-04 $\pm$ 1.05e-04 & 1.56e-03 $\pm$ 1.06e-04 & 4.06e-04 $\pm$ 2.01e-06 \\
$20\%$ & 2.12e-04 $\pm$ 1.09e-04 & 1.57e-03 $\pm$ 9.51e-05 & 4.17e-04 $\pm$ 2.38e-06 \\
$15\%$ & 2.01e-04 $\pm$ 8.59e-05 & 1.59e-03 $\pm$ 1.03e-04 & 4.30e-04 $\pm$ 4.23e-06 \\
$10\%$ & 1.98e-04 $\pm$ 8.39e-05 & 1.60e-03 $\pm$ 1.03e-04 & 4.55e-04 $\pm$ 3.97e-06 \\
$8\%$ & 2.04e-04 $\pm$ 9.19e-05 & 1.64e-03 $\pm$ 1.26e-04 & 4.85e-04 $\pm$ 1.11e-05 \\
$6\%$ & 1.91e-04 $\pm$ 7.09e-05 & 1.65e-03 $\pm$ 1.31e-04 & 5.37e-04 $\pm$ 2.37e-05 \\
$4\%$ & 2.11e-04 $\pm$ 8.76e-05 & 1.78e-03 $\pm$ 2.38e-04 & 6.71e-04 $\pm$ 3.89e-05 \\
$2\%$ & 6.96e-04 $\pm$ 1.14e-04 & 3.77e-03 $\pm$ 1.32e-04 & 1.32e-03 $\pm$ 1.08e-04 \\
$1\%$ & 4.06e-02 $\pm$ 5.64e-03 & 2.48e-02 $\pm$ 3.71e-03 & 1.56e-03 $\pm$ 1.37e-04 \\
\bottomrule
\end{tabular}
\label{tab:app_mask_overall_test_full}
\end{table}

For Wave and Navier--Stokes, the $64\times64$ grid contains 4096 spatial
locations, so $10\%$, $4\%$, $2\%$, and $1\%$ visibility retain 409,
163, 81, and 40 locations, respectively.
For Spherical Shallow Water, the $64\times128$ grid contains 8192 locations,
giving 819, 327, 163, and 81 retained locations at the same visible ratios.
Since these absolute counts differ across datasets, the recovery transition is governed by the number of observed locations relative to the per-component latent dimension $p$ rather than by the visible ratio itself, as quantified in \Cref{tab:app_mask_observability}.

\begin{table}[!htbp]
\centering
\caption{Observability constant $c_{\mathcal S}=\sigma_{\min}(J_{\mathcal S})$
for the sparse-initialization masks of \Cref{tab:app_mask_overall_test_full}.
Since the decoder is affine in the latent code, $c_{\mathcal S}$ equals the
smallest singular value of the $s\times p$ matrix whose rows are the
single-component decoder Jacobians $J_0(\mathbf{x})$ at the $s$ observed
locations (\Cref{lem:sensor_jacobian}). Values are mean $\pm$ sample standard
deviation across the same three training seeds and use the same masks as
\Cref{tab:app_mask_overall_test_full}. Entries equal to zero are exact and
occur when $s<p$, in which case the initial latent state is not identifiable
from the observations. Values are not comparable across datasets, since the
scale of $J_0$ depends on the trained decoder and the field normalization.}
\label{tab:app_mask_observability}
\small
\begin{tabular}{lccc}
\toprule
Visible ratio
  & \makecell{Wave\\(SMORE-Koopman, $p=100$)}
  & \makecell{Navier--Stokes\\(SMORE-LRLQ, $p=72$)}
  & \makecell{Shallow Water\\(SMORE-Koopman, $p=100$)} \\
\midrule
$100\%$ & $0.170\pm0.005$    & $8.64\pm0.48$   & $8.20\pm0.35$ \\
$50\%$  & $0.114\pm0.005$    & $6.04\pm0.21$   & $5.58\pm0.22$ \\
$10\%$  & $0.0400\pm0.0026$  & $2.34\pm0.16$   & $1.92\pm0.10$ \\
$4\%$   & $0.0104\pm0.0006$  & $0.965\pm0.088$ & $0.696\pm0.025$ \\
$2\%$   & $0$                & $0.102\pm0.004$ & $0.177\pm0.016$ \\
$1\%$   & $0$                & $0$             & $0$ \\
\bottomrule
\end{tabular}
\end{table}

\FloatBarrier

\end{document}